\documentclass[preprint, 11pt]{elsarticle}

\usepackage[margin=2.5cm]{geometry}

\usepackage{amssymb}
\usepackage{amssymb}
\usepackage{adjustbox}
\usepackage{amsmath}
\usepackage[amssymb]{SIunits}
\usepackage{amssymb}
\usepackage{multirow}
\usepackage{wrapfig}
\usepackage{enumitem}
\usepackage{subcaption}
\usepackage{mathtools}
\usepackage{makecell}
\usepackage{threeparttable}
\usepackage{lipsum}
\usepackage{hyperref}
\usepackage{algorithm}
\usepackage{algpseudocode}
\usepackage{float}
\usepackage{color,soul}               
\usepackage{enumerate}
\usepackage{xurl}
\usepackage{booktabs}
\usepackage{longtable}   

\newtheorem{proposition}{Proposition}

\newtheorem{proof}{Proof}

\usepackage{lipsum}
\let\OLDthebibliography\thebibliography
\renewcommand\thebibliography[1]{
  \OLDthebibliography{#1}
  \setlength{\parskip}{0pt}
  \setlength{\itemsep}{0pt plus 0.3ex}
}

\journal{Elsevier}
\begin{document}

\begin{frontmatter}



\title{A Two-Echelon Covering Tour Vehicle Routing Problem with Drones for Post-Disaster Relief}


\author[a,b]{Dang Viet Anh Nguyen\corref{cor1}}
\ead{andng@dtu.dk}
\cortext[cor1]{Corresponding author.}

\affiliation[a]{organization={Department of Technology, Management and Economics, Technical University of Denmark},
             addressline={Akademivej Bygning 358},
           postcode={2800 Kongens Lyngby},
             country={Denmark}}
\affiliation[b]{organization={School of Mechanical and Aerospace Engineering, Nanyang Technological University},
             addressline={50 Nanyang Avenue},
           postcode={639798},
             country={Singapore}}


\author[b]{Rajesh Piplani}
\author[c]{Aldy Gunawan}

\affiliation[c]{organization={School of Computing and Information Systems, Singapore Management University},
            addressline={80 Stamford Road}, 
            postcode={178902}, 
            country={Singapore}}

\begin{abstract}
We introduce the two-echelon covering tour vehicle routing problem (2E-CTVRP) for the distribution of relief supplies after a disaster. In the first echelon, a fleet of trucks transports supplies and drones from a central depot to satellites at the periphery of the affected area. In the second echelon, drones launched in parallel from the satellites deliver the supplies to the centroids of victim clusters, which are obtained by clustering the victim locations, and each truck waits at a satellite until its drones have returned. The problem combines the assignment of satellites to trucks, the sequencing of the truck routes, and the assignment of clusters to satellites, and minimizes the sum of the arrival times of the trucks at the satellites and at the depot. We formulate the 2E-CTVRP as a mixed integer linear program and propose a hybrid metaheuristic, GRASP-ILS-PR, which combines greedy randomized construction, a single-trajectory search, and periodic path relinking on the assignment of clusters to satellites. On a new benchmark set of 110 instances, GRASP-ILS-PR finds all known optimal solutions within 30 seconds, matches or improves on the one-hour solutions of a commercial solver on most larger instances, and improves on them by up to 11\% for instances with 50 satellites. It clearly outperforms a conventional GRASP with evolutionary path relinking and yields more consistent results than a simulated annealing algorithm that uses the same search components. An analysis of two fleet configurations shows that few large trucks minimize the cumulative arrival time in most instances, whereas many small trucks deliver supplies to the victims earlier and more equitably in every instance.
\end{abstract}

\end{frontmatter}


\section{Introduction}
\label{Sec1:Intro}

In recent years, the world has witnessed a troubling trend of increasingly severe natural and man-made disasters with devastating impacts on communities and human lives. According to the International Federation of Red Cross and Red Crescent Societies (IFRC), a disaster is a serious disruption of the functioning of a community that exceeds its capacity to cope using its own resources \citep{Disaster41:online}. In such situations, delivering essential items such as food, drinking water, and medicines to affected areas is crucial in reducing human suffering and risk. Humanitarian logistics plays a pivotal role in disaster relief efforts by ensuring timely and effective responses to the needs of affected communities. To address these challenges, Operations Research (OR) models have been developed to support and expedite decision-making during post-disaster conditions. One area that receives considerable attention from researchers and practitioners is the Vehicle Routing Problem (VRP). While classical VRP problems aim to minimize costs or improve customer service, the application of VRP in humanitarian aid has the critical objective of reducing human suffering, with time being the most important factor \citep{toth2014vehicle}. Moreover, infrastructure and road networks after a disaster often have limitations due to damage or congestion caused by unplanned panic evacuations, leading to a low level of service \citep{anuar2021vehicle}. Consequently, ground transport vehicles like trucks and vans may face significant difficulties reaching remote locations.

The emergence of aerial mobility technology has led to unmanned aerial vehicles (UAVs), also known as drones, gaining popularity in supply chains and logistics as a potential solution to last-mile delivery challenges in humanitarian logistics. Drones can be operated without a driver on board and are not restricted by terrain, road networks, or traffic conditions \citep{rejeb2021humanitarian}. They are faster than trucks and have lower transport costs \citep{vu2022two}. However, drones have a small capacity, and GPS data can be inaccurate, affecting package drop-offs \citep{lee2016technological}. In addition, drones need to return to the station to recharge or swap batteries. To combine the speed of drones with the capacity and range of trucks, a substantial and fast-growing body of work has examined two-echelon truck-drone routing, in which trucks carry drones toward a service area before drones complete last-mile delivery \citep{kitjacharoenchai2020two}. Recent studies in this stream have pushed toward exact solution methods \citep{zhou2023exact}, robust formulations under demand and travel-time uncertainty \citep{yin2023robust, faiz2024robust, zhang2023robust}, and dedicated heuristic solvers for post-disaster settings \citep{faiz2024computational}. These models, however, generally require drones to reach every individual demand point directly, which becomes impractical when the number and precise locations of victims cannot be fully enumerated in the immediate aftermath of a disaster.

In post-disaster contexts, exhaustive individual-level delivery is often infeasible due to incomplete location data, infrastructure collapse, and limited operational capacity. Instead, adopting a cluster-based spatial coverage approach allows humanitarian planners to identify representative points (such as cluster centroids) that efficiently cover victims within a given radius, ensuring rapid reach without overextending scarce resources. This approach enhances fairness by ensuring that no region is left unattended while optimizing travel time and resource allocation \citep{seltzer2025clustering, vargas2015clustering}. Recent advances in remote sensing, GIS, and drone-based data acquisition have enabled the identification and clustering of victims' locations in near real time, significantly improving situational awareness and supporting responsive and equitable aid distribution \citep{hein2018remote, ortiz2020geographic}. Thus, embedding spatial coverage principles into routing decisions provides a robust mechanism to integrate efficiency, efficacy, and equity, the three central dimensions of humanitarian logistics performance \citep{huang2012models}. The covering tour problem (CTP) formalizes this idea by allowing a subset of nodes to be served directly while others are considered satisfied through proximity \citep{current1989covering}, and it has been applied to humanitarian relief through the location of satellite distribution centers \citep{naji2012covering} and, more recently, through metaheuristic \citep{allahyari2015hybrid} and inventory-integrated \citep{kazancc2026covering} extensions. These covering tour models, however, route a single echelon of ground vehicles directly to the covering nodes and do not incorporate aerial last-mile delivery.

Two active streams of literature therefore address complementary halves of the same operational problem: two-echelon truck-drone routing offers a multi-modal delivery structure but assumes exhaustive point-to-point service, while covering tour models offer spatial coverage logic but are confined to a single echelon of ground vehicles. This paper closes that gap. We introduce the two-echelon covering tour vehicle routing problem (2E-CTVRP), a novel variant of the two-echelon vehicle routing problem (2E-VRP) that integrates covering-based spatial design into a multi-modal delivery system for post-disaster relief. In this framework, a fleet of trucks in the first echelon transports essential supplies and drones from a central depot to satellites located at the periphery of the affected area. The drones then perform last-mile delivery, distributing supplies from the satellites to representative drop-off points identified through a clustering algorithm. This design enhances the equitable and feasible distribution of relief goods while avoiding exhaustive individual-level deliveries.

The 2E-CTVRP is formulated as a mixed integer linear program (MILP), which a commercial solver solves to optimality within one hour only for small instances, with up to nine satellites in our experiments. For larger instances, which must nevertheless be solved within the short decision times of the response phase, we develop a hybrid metaheuristic, the Greedy Randomized Adaptive Search Procedure with Iterated Local Search and Path Relinking (GRASP-ILS-PR). Instead of performing a complete construction, local search, and path relinking in every iteration, as in a conventional GRASP with path relinking, GRASP-ILS-PR carries a single solution across iterations and invokes the expensive components periodically, which allows many more iterations within a given time limit. Finally, we analyze how the choice between a fleet of few large trucks and a fleet of many small trucks affects the efficiency, efficacy, and equity of the relief operation.

The main contributions of this paper are as follows:
\begin{itemize}
\item We propose the first covering-tour-based two-echelon truck-drone framework for humanitarian logistics, integrating multi-truck, multi-drone delivery with covering-based cluster centroids to enable equitable and operationally feasible coverage of dispersed victim populations without exhaustive individual-level routing.

\item We develop GRASP-ILS-PR, a hybrid metaheuristic tailored to the two-echelon structure of the problem. It exploits a decomposition of the objective function to construct the assignment of centroids to satellites, schedules local search and path relinking periodically around a single search trajectory, and applies path relinking to the second-echelon assignment. GRASP-ILS-PR finds all known optimal solutions within 30\,s, matches or improves on the one-hour incumbents of Gurobi for most larger instances, clearly outperforms a conventional GRASP with evolutionary path relinking that uses the same search components, and yields significantly more consistent results across runs than a simulated annealing algorithm built from these components.

\item We introduce a benchmark set of 110 instances that reflects the spatial characteristics of victim locations after a disaster, including large-scale instances with up to 50 satellites. Using this set, we quantify the trade-off between truck fleet configurations: few large trucks minimize the cumulative arrival time in most instances, whereas many small trucks deliver supplies to the victims earlier and more equitably in every instance.
\end{itemize}

The remainder of this paper is organized as follows. Section~\ref{Sec2:Lit} reviews the relevant literature. Section~\ref{Sec3} describes the 2E-CTVRP and presents its mathematical formulation. Section~\ref{Sec4:GRA} presents the GRASP-ILS-PR algorithm. Section~\ref{Sec5} reports the computational results, and Section~\ref{Sec6} analyzes the impact of the truck fleet on the relief operation. Section~\ref{Sec7:Con} concludes the paper and outlines directions for future research.

\section{Literature review}
\label{Sec2:Lit}
The 2E-CTVRP is related to five streams of research: vehicle routing in humanitarian operations, two-echelon vehicle routing, truck--drone routing, covering tour problems, and routing under uncertainty in post-disaster settings. This section reviews these streams and positions the 2E-CTVRP relative to them.

\subsection{Vehicle routing in humanitarian operations}
\label{Sec2.1:HumVRP}
Since the first applications of vehicle routing to humanitarian relief \citep{knott1988vehicle}, research on routing models for disaster operations has expanded steadily. \cite{farahani2011logistics} distinguished several classes of models in humanitarian logistics, including facility location, transportation and distribution, inventory, and integrated models; transportation and distribution models address the delivery of relief supplies, evacuation, and rescue. This study focuses on the delivery of essential supplies, such as drinking water, food, and medicine, to victims in the affected area. Comprehensive reviews of routing in humanitarian relief are provided by \cite{toth2014vehicle} and \cite{anuar2021vehicle}. Among the many variants studied, \cite{afsar2014exact} proposed a generalized VRP with a flexible fleet size, solved by column generation and iterated local search, and \cite{rivera2015multistart, rivera2016mathematical} considered vehicles performing multiple trips in disaster logistics.

In contrast to commercial routing, humanitarian routing aims to meet urgent needs as quickly and as fairly as possible, particularly in the response phase. The most commonly used performance measures are the response time, the equity of service across affected regions, the satisfied demand, and the transportation cost \citep{golden2014chapter}, and these measures may conflict. \cite{huang2012models} analyzed how the route structure changes when the objective is based on efficiency, efficacy, or equity, and observed that a fleet of small vehicles can perform well with respect to all three, at the price of a greater coordination effort.

Response time is typically incorporated through min-max, min-avg, or min-sum objectives. The min-max objective minimizes the latest arrival time, the min-avg objective minimizes the average arrival time at the demand nodes, and the min-sum objective minimizes the sum of the arrival times. The definition of the arrival time differs between studies: \cite{campbell2008routing} and \cite{rekik2012two} consider the time at which the last demand node is visited, whereas \cite{vu2022two} and \cite{hentenryck2010strategic} consider the time at which the last vehicle returns to the depot. The min-sum objective corresponds to that of the cumulative capacitated vehicle routing problem \citep{ngueveu2010effective}, and it has been adopted in humanitarian settings, for example for the routing of unmanned aerial vehicles in search and rescue operations \citep{kyriakakis2022cumulative}. Other studies measure the impact on victims more directly, for instance by penalizing deprivation time in hybrid truck--drone delivery after disasters \citep{rahimi2024improved}.

\subsection{Two-echelon vehicle routing}
\label{Sec2.2:2EVRP}
Humanitarian supply chains typically connect ports, airports, primary warehouses, satellite depots, and demand points through several networks. Similar multi-echelon distribution systems have been studied extensively in city logistics, multi-modal freight transportation, and e-commerce, most commonly as two-echelon routing problems, in which different vehicles operate in each echelon to exploit economies of scale and to overcome limitations such as vehicle size or road conditions. Comprehensive reviews are provided by \cite{cuda2015survey} and, more recently, by \cite{sluijk2023two}. The 2E-CTVRP belongs to the class of two-echelon vehicle routing problems (2E-VRPs), in which the vehicles of the two echelons must be coordinated at the satellites; \cite{li2021two}, for instance, model the synchronization of vehicles at the satellites explicitly.

Applications of the 2E-VRP to humanitarian logistics have grown rapidly, spanning agile heuristics, exact algorithms, and robust formulations. \citet{do2021agile} studied a two-echelon vehicle routing problem with pickup and delivery in which a single fleet of homogeneous vehicles visits intermediate nodes before serving the customers, motivated by real-time applications such as rescue operations with drones in humanitarian logistics; a biased-randomized heuristic, executed in parallel, provides competitive solutions within milliseconds. \cite{faiz2024robust} proposed a two-echelon problem in which ground vehicles transport UAVs with medical and essential supplies to satellites outside the affected zone, from which the drones serve the final demand points; a two-stage robust model, solved by column-and-constraint generation and column generation, addresses uncertain demand, and a case study on Hurricane Maria in Puerto Rico demonstrates its effectiveness. A companion study develops heuristics for larger instances of the same problem \citep{faiz2024computational}. \cite{xiao2025two} studied a two-echelon drone--truck routing problem with time windows and stochastic demand, in which one truck and one drone per distribution center serve individual demand points, and solved it by particle swarm optimization. Beyond the humanitarian setting, \cite{zhou2023exact} developed a branch-and-price algorithm for a two-echelon vehicle routing problem with drones, in which drones perform multiple trips while their vehicle is stationed at a customer node. Exact approaches of this type provide optimal solutions but do not scale to the instance sizes and decision times of post-disaster planning.

\subsection{Truck--drone routing}
\label{Sec2.3:TruckDrone}
Drones are increasingly used for last-mile delivery, and companies such as Amazon, Google, and UPS are exploring the technology \citep{bamburry2015drones}; \cite{macrina2020drone} provide a comprehensive review of drone-aided routing. \cite{murray2015flying} introduced the flying sidekick traveling salesman problem, in which a truck and a single drone operate in tandem, and the parallel drone scheduling traveling salesman problem. Models with multiple trucks and drones for last-mile delivery followed \citep{wang2017vehicle, poikonen2017vehicle, kitjacharoenchai2019multiple}. \cite{kitjacharoenchai2020two} introduced the two-echelon vehicle routing problem with drones, with multiple trucks in the first echelon and multiple drones in the second, and solved it by a mixed integer program for small instances and by two heuristics for larger ones. \cite{vu2022two} proposed a two-echelon truck--drone model in which, unlike in the flying sidekick problem, the truck waits at a satellite for the return of multiple drones, but only a single truck operates in the first echelon.

Drones offer considerable potential in disaster relief, since they are not restricted by damaged road networks. \cite{rabta2018drone} proposed a drone fleet model for last-mile distribution in disaster relief that accounts for payload and battery limitations, and \cite{shi2024optimal} coordinated helicopters, trucks, and drones for the distribution of emergency supplies after an earthquake. Equity has recently been considered explicitly: \cite{lu2023humanitarian} maximized the fairness of supply allocation in a vehicle routing problem synchronized with drones under time-varying weather, and \cite{khameneh2025multi} minimized the gap between delivery times in a multi-truck, multi-drone system for flood relief. In all these models, however, the drones deliver to individual demand points.

\subsection{Covering tour problems}
\label{Sec2.4:CTP}
In covering tour problems, vehicles visit only a subset of the demand nodes, and the remaining nodes are covered if they lie within a given distance of a visited node, so that customers can collect their supplies nearby. \cite{current1989covering} introduced the concept as the covering salesman problem, and numerous variants have been studied since \citep{kammoun2017integration, pham2017solving}, including the multi-depot covering tour vehicle routing problem, for which \cite{allahyari2015hybrid} developed a hybrid metaheuristic based on variable neighborhood descent and variable neighborhood search. \cite{glock2022spatial} formalized the broader concept of spatial coverage, in which nodes that are not visited directly are considered served for the purpose of the objective or the constraints, and defined the vehicle routing problem with spatial coverage.

Covering tour models have been applied to humanitarian relief in several forms. \cite{naji2012covering} used a covering tour approach to locate satellite distribution centers for delivering aid to affected regions, and \cite{veenstra2018simultaneous} introduced a simultaneous facility location and routing problem with covering options, in which patients collect medical supplies from lockers within a coverage distance. In humanitarian applications, the objective often shifts from routing cost to minimizing walking distance and maximizing population coverage \citep{nolz2010water, nolz2011risk}. More recently, \cite{kazancc2026covering} proposed a two-echelon covering inventory routing problem for the distribution of medical kits under demand uncertainty. Across this literature, however, covering tours are performed by ground vehicles, and aerial last-mile delivery is not considered.

\subsection{Uncertainty in post-disaster routing}
\label{Sec2.4:Uncertainty}
Post-disaster environments are characterized by incomplete and unreliable information, and a growing stream of research models this uncertainty explicitly. \cite{yin2023robust} studied a robust vehicle routing problem with drones under uncertain demands and truck travel times in humanitarian logistics. \cite{zhang2023robust} addressed robust drone selective routing for the assessment of post-disaster transportation networks, whose conditions must be inferred through drone reconnaissance. \cite{ghelichi2022drone} located drone platforms for delivering relief to populations whose demand locations are initially unknown, using chance-constrained stochastic programming. As discussed above, \cite{faiz2024robust} and \cite{kazancc2026covering} protect against uncertain demand through robust and scenario-based stochastic formulations, respectively. These studies show that uncertainty is increasingly treated as a first-order modeling concern. The present study addresses the deterministic version of the problem, which provides the basis for such extensions; we return to uncertainty as a direction for future research in Section~\ref{Sec7:Con}.

\subsection{Positioning of the 2E-CTVRP}
\label{Sec2.5:Gap}
Table~\ref{tab:lit} compares the 2E-CTVRP with the most closely related studies. The literature reveals two largely separate streams. Covering tour models for humanitarian relief achieve equitable and operationally feasible coverage of dispersed victims through representative covering points, but they rely on a single echelon of ground vehicles. Truck--drone models, including their two-echelon humanitarian variants, provide a multi-modal delivery structure suited to damaged road networks, but the drones serve every individual demand point, and the two-echelon humanitarian variants are often restricted to a single fleet that serves both echelons, a single truck in the first echelon, or a single truck--drone pair per distribution center \citep{do2021agile, vu2022two, xiao2025two}.

The 2E-CTVRP combines the strengths of both streams. Building on the concept of spatial coverage \citep{glock2022spatial}, drones deliver supplies to the centroids of victim clusters rather than to individual locations, while a fleet of multiple trucks transports the supplies and drones from the depot to the satellites, and multiple drones operate in parallel at each satellite. The assignment of clusters to satellites is determined jointly with the truck routes, since it determines both the loads of the trucks and the waiting times at the satellites. The objective, the sum of the arrival times of the trucks, extends the cumulative routing objective \citep{ngueveu2010effective, kyriakakis2022cumulative} to a two-echelon setting, in which the drone missions at a satellite delay all subsequent visits of the truck. To the best of our knowledge, no existing study combines multi-truck, multi-drone two-echelon delivery with covering-based cluster centroids for humanitarian relief. To solve the problem, we propose the GRASP-ILS-PR metaheuristic, and we analyze how the choice between a fleet of few large trucks and a fleet of many small trucks affects the efficiency, efficacy, and equity of the relief operation.

\begin{table}[htbp]
\centering
\scriptsize
\setlength{\tabcolsep}{4pt}
\begin{threeparttable}
\caption{Comparison of the 2E-CTVRP with closely related studies. Trucks and drones: S = single, M = multiple, -- = not used. Delivery: Ind.\ = drones or vehicles serve individual demand points; Cov.\ = demand is served through covering nodes or cluster centroids.}
\label{tab:lit}
\begin{tabular}{lcccccc}
\toprule
Study & Two echelons & Trucks & Drones & Delivery & Humanitarian & Uncertainty \\
\midrule
\cite{murray2015flying}                & --           & S  & S  & Ind. & --           & -- \\
\cite{kitjacharoenchai2020two}         & $\checkmark$ & M  & M  & Ind. & --           & -- \\
\cite{zhou2023exact}                   & $\checkmark$ & M  & M  & Ind. & --           & -- \\
\cite{rabta2018drone}                  & --           & -- & M  & Ind. & $\checkmark$ & -- \\
\cite{do2021agile}\tnote{a}            & $\checkmark$ & -- & M  & Ind. & $\checkmark$ & -- \\
\cite{vu2022two}                       & $\checkmark$ & S  & M  & Ind. & $\checkmark$ & -- \\
\cite{yin2023robust}                   & --           & M  & M  & Ind. & $\checkmark$ & $\checkmark$ \\
\cite{lu2023humanitarian}              & --           & M  & M  & Ind. & $\checkmark$ & -- \\
\cite{faiz2024robust}                  & $\checkmark$ & M  & M  & Ind. & $\checkmark$ & $\checkmark$ \\
\cite{xiao2025two}                     & $\checkmark$ & S\tnote{b} & S\tnote{b} & Ind. & $\checkmark$ & $\checkmark$ \\
\cite{khameneh2025multi}               & --           & M  & M  & Ind. & $\checkmark$ & -- \\
\cite{naji2012covering}                & --           & M  & -- & Cov. & $\checkmark$ & -- \\
\cite{allahyari2015hybrid}             & --           & M  & -- & Cov. & --           & -- \\
\cite{kazancc2026covering}             & $\checkmark$ & M  & -- & Cov. & $\checkmark$ & $\checkmark$ \\
\midrule
This study                             & $\checkmark$ & M  & M  & Cov. & $\checkmark$ & -- \\
\bottomrule
\end{tabular}
\begin{tablenotes}
\item[a] A single homogeneous fleet serves both echelons; drone-based rescue operations are given as a motivating application.
\item[b] One truck and one drone per distribution center.
\end{tablenotes}
\end{threeparttable}
\end{table}

%
%

\section{Problem description and mathematical formulation}
\label{Sec3}

\subsection{Problem description}
\label{sec3:description}

The 2E-CTVRP addresses post-disaster humanitarian delivery through a coordinated two-echelon system. It is motivated by the 2025 flooding in Thai Nguyen Province, Vietnam, where severe rainfall submerged residential areas and destroyed critical road infrastructure. The problem reflects situations in which displaced residents gather at evacuation points while damaged bridges and roads prevent trucks from reaching the affected communities. Relief organizations must deliver essential supplies from a central depot to multiple isolated locations under severe infrastructure and resource constraints.

\begin{figure}[H]
\centering
    \includegraphics[width = 0.6\textwidth]{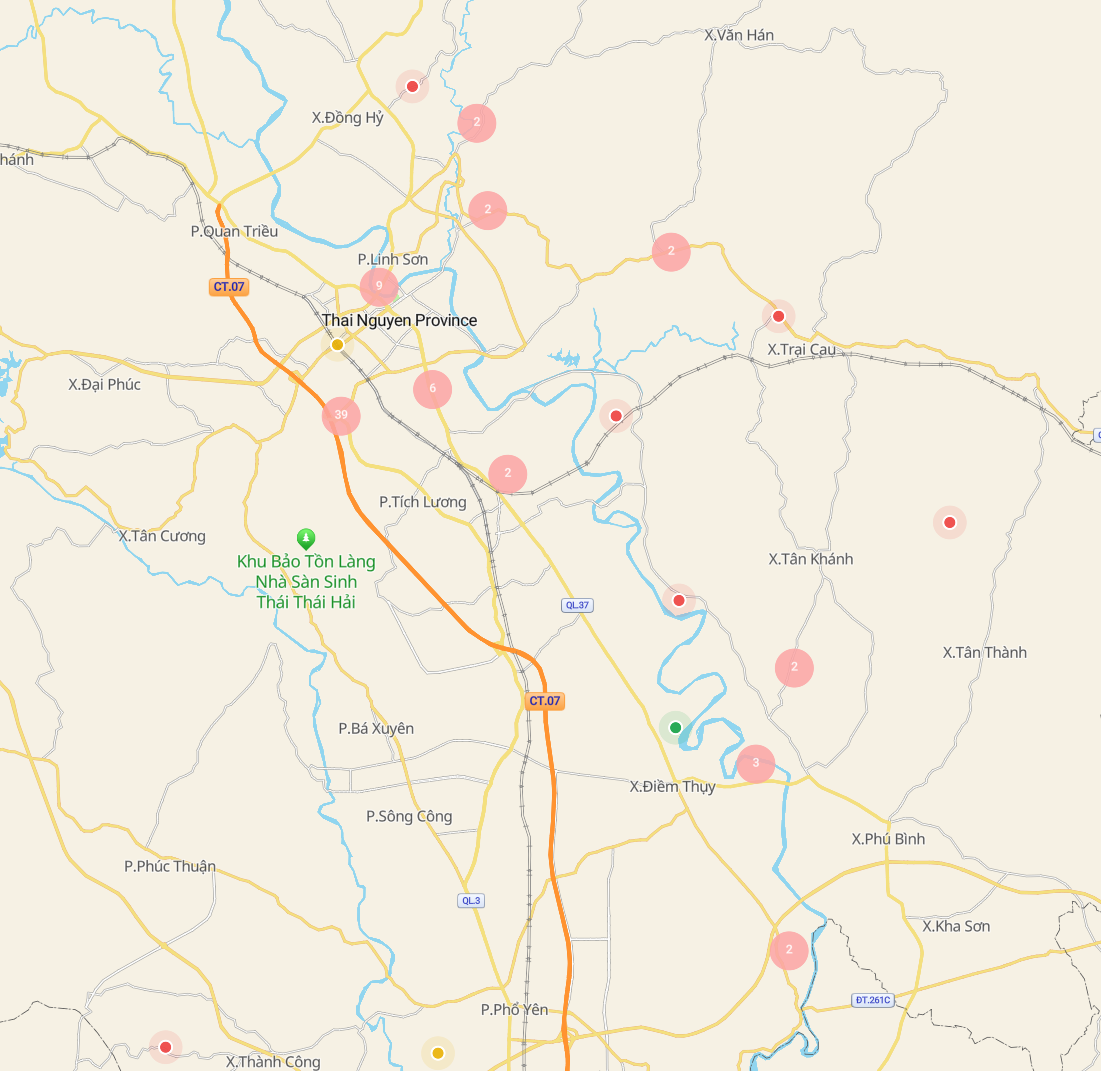}
    \caption{Spatial distribution of victim locations forming distinct clusters, with the number in each red circle indicating the number of victim locations in the area following the 2025 flood in Thai Nguyen Province, Vietnam (source: \url{https://thongtincuuho.org/})}
    \label{figA2}
\end{figure}

The problem is defined on a two-echelon network consisting of (i) a primary depot outside the affected area, where relief supplies and drones are stored; (ii) a set of satellites on the periphery of the affected area, whose locations are determined beforehand, for example by a preceding facility location decision, on the basis of their accessibility from the depot and their suitability for drone operations; and (iii) a set of demand clusters, each representing a group of victims requiring aid (Figure~\ref{figA2}). Delivering relief to every individual victim location is impractical in the immediate aftermath of a disaster, because victim locations are incompletely known and resources are scarce. The victim locations are therefore grouped into clusters in a preprocessing step, and the centroid of each cluster serves as the drone delivery point, from which victims collect supplies within walking distance. The number of clusters is set equal to the number of satellites. The assignment of clusters to satellites is not fixed in advance but is a decision of the problem, since it determines both the load carried by each truck and the waiting time at each satellite.

In the first echelon, a fleet of $K$ homogeneous trucks, each with capacity $Q$, transports relief supplies and drones from the depot to the satellites. Each truck performs one route that starts at the depot, visits a sequence of satellites, and returns to the depot. Each satellite is visited by exactly one truck, and each truck visits at least one satellite. In the second echelon, each satellite serves exactly one cluster centroid, and each centroid is served from exactly one satellite. When a truck arrives at a satellite, up to $U$ drones, each with payload capacity $P$, are launched in parallel to deliver the demand of the assigned centroid, and each drone performs a single round trip. The truck waits at the satellite until all drones have returned and then proceeds to the next satellite on its route.

This truck--drone synchronization assumption is more restrictive than the coordination schemes considered in commercial applications, but it reflects post-disaster conditions. First, communication infrastructure, such as cellular networks and positioning services, is often damaged or unavailable after a disaster, which makes real-time coordination between moving trucks and drones difficult; the Thai Nguyen flood disrupted networks across the affected districts. Second, drones face elevated operational risks, including adverse weather, poor visibility, debris, and unstable landing zones. If a truck departs before its drones return, recovery may fail, and drones may be lost at a time when equipment is scarce. Third, the limited battery endurance of drones requires trucks to remain within retrieval range to guarantee safe recovery. Finally, this conservative policy ensures operational reliability in the initial response phase, when communication is unreliable and spare equipment is unavailable. Since relief operations increasingly draw drones from several sources, such as government agencies, non-governmental organizations, and donors, the drone resources at a satellite are modeled through their aggregate payload capacity $UP$ rather than through individual drone schedules.

The 2E-CTVRP differs from classical covering tour problems, in which vehicles visit a subset of the nodes and the remaining nodes are covered if they lie within a given distance of a visited node. In the 2E-CTVRP, coverage is achieved at the clustering stage, and every cluster centroid must be served by drones. Three interdependent decisions remain: the allocation of satellites to trucks, the sequence in which each truck visits its satellites, and the assignment of centroids to satellites. The assignment determines the demand that each truck must carry, and hence the feasibility of the truck routes with respect to capacity. It also determines the waiting time at each satellite, which delays the arrival of the truck at all subsequent satellites on its route. The objective is to minimize the sum of the arrival times of the trucks at all satellites and of their return times to the depot. This objective is closely related to that of the cumulative capacitated vehicle routing problem \citep{ngueveu2010effective}. It favors the early service of all satellites rather than low travel cost, reflecting that in disaster response, time is the most critical factor in reducing human suffering.

\subsection{Mathematical formulation}
\label{sec3:formulation}

The 2E-CTVRP is defined on a directed graph $G=(V\cup M,E_1\cup E_2)$, where $V\cup M$ is the set of nodes and $E_1\cup E_2$ is the set of arcs. Let $V=V_0\cup V_s$, where $V_0=\{v_0,v_{n+1}\}$ contains the start and end depot nodes and $V_s=\{v_1,\ldots,v_n\}$ denotes the set of $n$ satellites. The nodes $v_0$ and $v_{n+1}$ represent the same physical depot; $v_{n+1}$ is introduced as a copy of $v_0$ to represent the end of each truck route.

Let $\mathcal{K}=\{1,\ldots,K\}$ denote the set of trucks, each with capacity $Q$. The first-echelon arc set is $E_1=\{(i,j)\mid i,j\in V,\ i\neq j\}$, and the travel time of truck $k\in\mathcal{K}$ from node $v_i$ to node $v_j$ is denoted by $t_{ij}^k$. Although the trucks within each fleet configuration are homogeneous, the vehicle-dependent notation is retained to accommodate the different truck fleet types considered in the computational experiments.

In the second echelon, drones deliver relief from the satellites to the cluster centroids. Let $M=\{\mu_1,\ldots,\mu_n\}$ denote the set of $n$ cluster centroids, and let $d_j$ denote the total demand of centroid $\mu_j$. The second-echelon arc set is $E_2=\{(i,j)\mid i\in V_s,\ j\in M\}$. At most $U$ drones, each with payload capacity $P$, can be launched in parallel from a visited satellite. The one-way flight time from satellite $v_i$ to centroid $\mu_j$ is denoted by $t'_{ij}$. Drones launched from a satellite must return to the same satellite before the truck can continue its route, and $w_i$ denotes the resulting waiting time of the truck at satellite $v_i$. We assume that all satellite--centroid pairs satisfy the flight-range requirements of the drones; therefore, no flight-range constraints are included in the formulation.

The planning horizon starts at time zero, when the trucks leave the depot carrying relief supplies and drones. Let $a_i^k$ denote the arrival time of truck $k$ at node $v_i$. Each centroid is assigned to exactly one satellite, and its demand is transported to that satellite by a single truck; multiple drones may be launched in parallel from the satellite to serve the demand of the centroid.

The binary variable $x_{ij}^k$ equals 1 if truck $k$ traverses arc $(i,j)\in E_1$, and 0 otherwise. The binary variable $y_i^k$ equals 1 if truck $k$ visits satellite $v_i\in V_s$, and 0 otherwise. The binary variable $z_{ij}$ equals 1 if centroid $\mu_j\in M$ is served from satellite $v_i\in V_s$, and 0 otherwise. The auxiliary binary variable $\delta_{ij}^k$ equals 1 if satellite $v_i$ is visited by truck $k$ and centroid $\mu_j$ is served from $v_i$, and 0 otherwise. The continuous variables $a_i^k$ and $w_i$ represent the arrival times and waiting times, respectively. Finally, $H$ denotes a sufficiently large constant. A valid value is
\begin{equation*}
\label{eq:bigH}
H=(n+1)\max_{(i,j)\in E_1,\,k\in\mathcal{K}}t_{ij}^{k}+\sum_{i\in V_s}2\max_{j\in M}t'_{ij},
\end{equation*}
since a truck route traverses at most $n+1$ arcs and waits at most once at each satellite, so that $H$ bounds the arrival time of any truck at any node. The notation is summarized in Table~\ref{tab:notation}, and Figure~\ref{fig1} illustrates a solution.

\begin{table}[H]
\centering
\caption{Notation used in the 2E-CTVRP formulation}
\begin{tabular}{p{0.25\textwidth} p{0.70\textwidth}}
\toprule
\textbf{Symbol} & \textbf{Description} \\
\midrule
\multicolumn{2}{l}{\textbf{Sets}} \\
\midrule
$G=(V\cup M,E_1\cup E_2)$ & Directed graph representing the two-echelon network \\
$V=V_0\cup V_s$ & Set of depot and satellite nodes \\
$V_0=\{v_0,v_{n+1}\}$ & Start depot $v_0$ and end-depot copy $v_{n+1}$ \\
$V_s=\{v_1,\ldots,v_n\}$ & Set of $n$ satellites \\
$M=\{\mu_1,\ldots,\mu_n\}$ & Set of $n$ cluster centroids \\
$\mathcal{K}=\{1,\ldots,K\}$ & Set of trucks \\
$E_1=\{(i,j)\mid i,j\in V,\;i\neq j\}$ & Set of first-echelon arcs \\
$E_2=\{(i,j)\mid i\in V_s,\;j\in M\}$ & Set of second-echelon arcs \\
\midrule
\multicolumn{2}{l}{\textbf{Parameters}} \\
\midrule
$K$ & Number of trucks \\
$U$ & Maximum number of drones launched in parallel from a satellite \\
$Q$ & Capacity of each truck (units of supplies) \\
$P$ & Payload capacity of each drone (units of supplies) \\
$d_j$ & Total demand of cluster centroid $\mu_j\in M$ \\
$t_{ij}^{k}$ & Travel time of truck $k\in\mathcal{K}$ on arc $(i,j)\in E_1$ \\
$t'_{ij}$ & One-way flight time from satellite $v_i\in V_s$ to centroid $\mu_j\in M$ \\
$H$ & Sufficiently large constant used in the time-propagation constraints \\
\midrule
\multicolumn{2}{l}{\textbf{Decision variables}} \\
\midrule
$x_{ij}^{k}$ & 1 if truck $k$ traverses arc $(i,j)\in E_1$; 0 otherwise \\
$y_i^{k}$ & 1 if truck $k$ visits satellite $v_i\in V_s$; 0 otherwise \\
$z_{ij}$ & 1 if centroid $\mu_j\in M$ is served from satellite $v_i\in V_s$; 0 otherwise \\
$\delta_{ij}^{k}$ & 1 if satellite $v_i$ is visited by truck $k$ and centroid $\mu_j$ is served from $v_i$; 0 otherwise \\
$a_i^{k}$ & Nonnegative continuous variable: arrival time of truck $k$ at node $v_i\in V$ \\
$w_i$ & Nonnegative continuous variable: waiting time of the truck at satellite $v_i\in V_s$ \\
\bottomrule
\end{tabular}
\label{tab:notation}
\end{table}

\begin{figure}[H]
\centering
    \includegraphics[width = 0.8\textwidth]{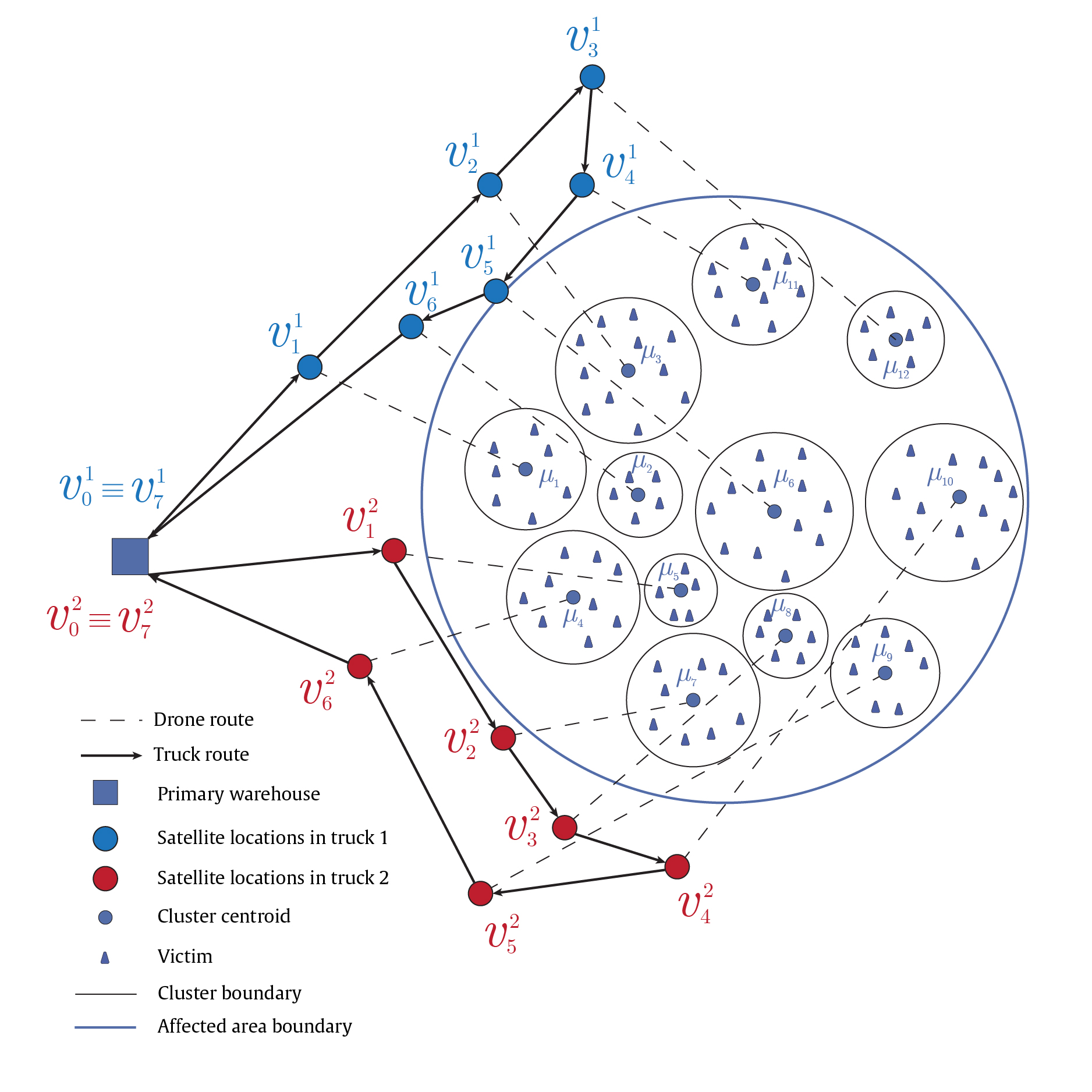}
    \caption{Illustration of a 2E-CTVRP solution}
    \label{fig1}
\end{figure}

The 2E-CTVRP is formulated as follows, where we set $w_0=0$ for compactness.

\begin{flalign}
\min \quad
& \sum_{k \in \mathcal{K}}
  \sum_{i \in V \setminus \{v_0\}} a_i^k
\label{eq1}&
\end{flalign}

\setlength{\belowdisplayskip}{0pt}
\setlength{\belowdisplayshortskip}{0pt}
\setlength{\abovedisplayskip}{0pt}
\setlength{\abovedisplayshortskip}{0pt}

\begin{subequations}
\begin{flalign}
&\text{s.t.}\quad
\sum_{i \in V_s} x_{i0}^k = 0
\quad \forall k \in \mathcal{K}
\label{eq2a}&
\\
&
\sum_{j \in V_s} x_{0j}^k = 1
\quad \forall k \in \mathcal{K}
\label{eq2b}&
\end{flalign}
\end{subequations}

\begin{subequations}
\begin{flalign}
&
\sum_{i \in V_s} x_{i,n+1}^k = 1
\quad \forall k \in \mathcal{K}
\label{eq3a}&
\\
&
\sum_{j \in V_s} x_{n+1,j}^k = 0
\quad \forall k \in \mathcal{K}
\label{eq3b}&
\end{flalign}
\end{subequations}

\begin{flalign}
&
\sum_{k \in \mathcal{K}}
\sum_{\substack{j \in V\\j\neq i}}
x_{ij}^k = 1
\quad \forall i \in V_s
\label{eq4}&
\end{flalign}

\begin{flalign}
&
\sum_{\substack{i \in V\\i\neq j}} x_{ij}^k
=
\sum_{\substack{i \in V\\i\neq j}} x_{ji}^k
\quad
\forall j \in V_s,\;
\forall k \in \mathcal{K}
\label{eq5}&
\end{flalign}

\begin{flalign}
&
\sum_{\substack{i \in V\\i\neq j}} x_{ij}^k
=
y_j^k
\quad
\forall j \in V_s,\;
\forall k \in \mathcal{K}
\label{eq6}&
\end{flalign}

\begin{flalign}
&
z_{ij}
\leq
\sum_{k \in \mathcal{K}} y_i^k
\quad
\forall i \in V_s,\;
\forall j \in M
\label{eq7}&
\end{flalign}

\begin{flalign}
&
\sum_{j \in M} z_{ij}
=
\sum_{k \in \mathcal{K}} y_i^k
\quad
\forall i \in V_s
\label{eq8}&
\end{flalign}

\begin{flalign}
&
\sum_{i \in V_s} z_{ij}
=1
\quad
\forall j \in M
\label{eq9}&
\end{flalign}

\begin{subequations}
\begin{flalign}
&
\delta_{ij}^k \leq y_i^k
\quad
\forall i \in V_s,\;
\forall j \in M,\;
\forall k \in \mathcal{K}
\label{eq10a}&
\\
&
\delta_{ij}^k \leq z_{ij}
\quad
\forall i \in V_s,\;
\forall j \in M,\;
\forall k \in \mathcal{K}
\label{eq10b}&
\\
&
\delta_{ij}^k
\geq
y_i^k + z_{ij} - 1
\quad
\forall i \in V_s,\;
\forall j \in M,\;
\forall k \in \mathcal{K}
\label{eq10c}&
\\
&
\sum_{i \in V_s}
\sum_{j \in M}
d_j \delta_{ij}^k
\leq Q
\quad
\forall k \in \mathcal{K}
\label{eq10d}&
\end{flalign}
\end{subequations}

\begin{flalign}
&
\sum_{j \in M} d_j z_{ij}
\leq UP
\quad
\forall i \in V_s
\label{eq11}&
\end{flalign}

\begin{flalign}
&
a_j^k
\geq
a_i^k + w_i + t_{ij}^k
-H(1-x_{ij}^k)
\nonumber\\[-1mm]
&
\hspace{2cm}
\forall i \in V\setminus\{v_{n+1}\},\;
\forall j \in V\setminus\{v_0,v_i\},\;
\forall k \in \mathcal{K}
\label{eq12}&
\end{flalign}

\begin{subequations}
\begin{flalign}
&
a_i^k \leq H y_i^k
\quad
\forall i \in V_s,\;
\forall k \in \mathcal{K}
\label{eq13a}&
\\
&
a_0^k = 0
\quad
\forall k \in \mathcal{K}
\label{eq13b}&
\end{flalign}
\end{subequations}

\begin{flalign}
&
w_i
\geq
2t'_{ij}z_{ij}
\quad
\forall i \in V_s,\;
\forall j \in M
\label{eq14}&
\end{flalign}

\begin{flalign}
&
x_{ij}^k \in \{0,1\}
\quad
\forall (i,j) \in E_1,\;
\forall k \in \mathcal{K}
\label{eq15}&
\end{flalign}

\begin{flalign}
&
y_i^k \in \{0,1\}
\quad
\forall i \in V_s,\;
\forall k \in \mathcal{K}
\label{eq16}&
\end{flalign}

\begin{flalign}
&
z_{ij} \in \{0,1\}
\quad
\forall i \in V_s,\;
\forall j \in M
\label{eq17}&
\end{flalign}

\begin{subequations}
\begin{flalign}
&
\delta_{ij}^k \in \{0,1\}
\quad
\forall i \in V_s,\;
\forall j \in M,\;
\forall k \in \mathcal{K}
\label{eq18a}&
\\
&
a_i^k \geq 0
\quad
\forall i \in V,\;
\forall k \in \mathcal{K}
\label{eq18b}&
\\
&
w_i \geq 0
\quad
\forall i \in V_s
\label{eq18c}&
\end{flalign}
\end{subequations}

The objective function~\eqref{eq1} minimizes the sum of the arrival times of the trucks at the satellites and of their return times to the end depot $v_{n+1}$. Since each satellite is visited by exactly one truck, Constraints~\eqref{eq13a} force the arrival-time variables of all other trucks at that satellite to zero, so that each satellite contributes exactly once to the objective. Constraints~\eqref{eq2a}--\eqref{eq2b} and~\eqref{eq3a}--\eqref{eq3b} ensure that each truck leaves the start depot $v_0$ towards a satellite and ends its route at the end-depot copy $v_{n+1}$. Constraints~\eqref{eq4} require each satellite to be visited by exactly one truck, and Constraints~\eqref{eq5} impose flow conservation at the satellites. Constraints~\eqref{eq6} link the routing variables $x_{ij}^k$ with the visiting variables $y_i^k$. Constraints~\eqref{eq7} ensure that a centroid can be served from satellite $v_i$ only if that satellite is visited by a truck. These constraints are implied by Constraints~\eqref{eq8}, since $z_{ij}\leq\sum_{j'\in M}z_{ij'}$, and are stated for clarity. Constraints~\eqref{eq8} require every visited satellite to serve exactly one centroid, and Constraints~\eqref{eq9} require every centroid to be served from exactly one satellite. Since Constraints~\eqref{eq4} and~\eqref{eq6} ensure that every satellite is visited, Constraints~\eqref{eq8} and~\eqref{eq9} together imply that the assignment of centroids to satellites is one-to-one. Constraints~\eqref{eq10a}--\eqref{eq10c} define $\delta_{ij}^k$ as the product of $y_i^k$ and $z_{ij}$, and Constraints~\eqref{eq10d} ensure that the total demand carried by each truck does not exceed its capacity $Q$. Constraints~\eqref{eq11} limit the demand served from each satellite to the aggregate payload capacity $UP$ of the drones launched in parallel. Constraints~\eqref{eq12} propagate the arrival times along the arcs traversed by each truck. Constraints~\eqref{eq13a} set the arrival time of a truck at a satellite it does not visit to zero, and Constraints~\eqref{eq13b} set the departure time from the depot to zero. Constraints~\eqref{eq14} set the waiting time at each satellite to at least the round-trip flight time to its assigned centroid. Finally, Constraints~\eqref{eq15}--\eqref{eq18c} define the domains of the decision variables.

Constraints~\eqref{eq12} also eliminate subtours, following the time-propagation principle described by \citet{ngueveu2010effective}. Consider a directed cycle $C$ consisting only of satellites and traversed by truck $k$. For every arc $(i,j)\in C$, $x_{ij}^k=1$, and Constraint~\eqref{eq12} reduces to $a_j^k \geq a_i^k + w_i + t_{ij}^k$. Summing these inequalities over all arcs of $C$ gives
\[
0 \geq \sum_{i\in C} w_i + \sum_{(i,j)\in C} t_{ij}^k.
\]
Since $w_i\geq0$ and $t_{ij}^k>0$, because the satellites are located at distinct points, the right-hand side is strictly positive, which is a contradiction. Hence, no cycle consisting only of satellites can exist, and every truck route is connected to the depot.

\section{Proposed Algorithm}
\label{Sec4:GRA}
In the response phase of a disaster, routing decisions must be made under tight time constraints and may have to be revised as information on damage, road accessibility, and demand becomes available. Solution methods for the 2E-CTVRP should therefore provide high-quality solutions within short computation times. The computational results in Section~\ref{Sec5} show that the MILP formulation of Section~\ref{Sec3} is solved to proven optimality within one hour only for instances with up to nine satellites. To address larger instances, we propose a hybrid metaheuristic, the Greedy Randomized Adaptive Search Procedure with Iterated Local Search and Path Relinking (GRASP-ILS-PR). The remainder of this section is organized as follows. Section~\ref{sec4:prelim} describes the algorithmic framework and the rationale for its design. Section~\ref{sec4:repr} introduces the solution representation, a decomposition of the objective function, and the feasibility conditions. Sections~\ref{sec4:constr} and~\ref{sec4:repair} present the construction heuristic and the repair operator, respectively. Section~\ref{sec4:ls} describes the neighborhood operators and the local search procedure. Section~\ref{sec4:pr} presents the path relinking procedure and the management of the elite pool. Finally, Section~\ref{sec4:hybrid} states the complete GRASP-ILS-PR algorithm.

\subsection{Algorithmic framework}
\label{sec4:prelim}
The Greedy Randomized Adaptive Search Procedure \citep{feo1995greedy,resende2016optimization} is a multi-start metaheuristic. Each iteration constructs a solution by a greedy randomized heuristic and improves it by local search, and the best solution over all iterations is returned. Since the iterations are independent, basic GRASP does not use information from the search history. Path relinking \citep{glover1997tabu,glover2000fundamentals} addresses this limitation by maintaining a pool of elite solutions and exploring trajectories that connect new solutions to members of the pool. Along such a trajectory, the attributes of the initial solution are progressively replaced by those of the guiding solution, so that good attributes of both solutions can be combined. The hybridization of GRASP with path relinking \citep{laguna1999grasp} has been applied successfully to routing problems with intermediate facilities, including the truck and trailer routing problem with satellite depots \citep{villegas2011grasp}.

In a standard GRASP with path relinking, every iteration performs a complete randomized construction, a complete local search, and path relinking against the elite pool. The cost of the construction and of the local search grows with the number of satellites and trucks. Under a fixed time budget, the number of iterations that can be performed on large instances is therefore limited. GRASP-ILS-PR retains the components of GRASP with path relinking but organizes them around a single search trajectory, in the spirit of Iterated Local Search \citep{lourencco2003iterated}. A working solution is carried over from one iteration to the next. In each iteration, one move is sampled from a randomly selected neighborhood, and the resulting solution replaces the working solution if it is feasible and its objective value is not worse. The acceptance of moves of equal value allows the search to traverse plateaus of the objective function. The more expensive components are invoked periodically rather than in every iteration: a complete local search is applied at a fixed iteration interval, and path relinking between the working solution and a member of the elite pool is applied at another fixed interval. The greedy randomized construction generates the initial solution and restarts the search whenever the best solution found has not improved for a given number of iterations. This organization combines the high iteration rate of a trajectory-based method with the intensification provided by path relinking and the diversification provided by randomized restarts.

Path relinking is applied to the second-echelon decisions, that is, to the assignment of cluster centroids to satellites, with the truck routes kept fixed. These assignments form a set of bijections between satellites and centroids, on which a natural distance between solutions can be defined (Section~\ref{sec4:pr}). We also evaluated an extension in which the elite pool is periodically evolved by relinking pairs of its members, following the evolutionary path relinking scheme of \citet{resende2010grasp}. Since this extension did not improve solution quality in the experiments reported in Section~\ref{Sec5}, it is not part of the final algorithm.

\subsection{Solution representation and evaluation}
\label{sec4:repr}
The notation of Section~\ref{Sec3} is used throughout. Since the trucks within a fleet configuration are homogeneous, we omit the truck index of the travel times and write $t_{ij}$ for $t^{k}_{ij}$.

A solution of the 2E-CTVRP is represented by a pair $S=(T,\sigma)$. The first-echelon component $T=(r_1,\dots,r_K)$ contains one route per truck,
\begin{equation}
\label{eq:route}
r_k=\left(v_{\pi_k(0)},v_{\pi_k(1)},\dots,v_{\pi_k(m_k)},v_{\pi_k(m_k+1)}\right),\qquad \pi_k(0)=0,\quad \pi_k(m_k+1)=n+1,\qquad k\in\mathcal{K},
\end{equation}
where $v_{\pi_k(1)},\dots,v_{\pi_k(m_k)}$ are the $m_k$ satellites visited by truck $k$, in the order of visit. By Constraints~\eqref{eq4}, every satellite belongs to exactly one route, so that $\sum_{k\in\mathcal{K}}m_k=n$. By Constraints~\eqref{eq2b}, every truck leaves the depot towards a satellite, so that $m_k\geq 1$ for all $k\in\mathcal{K}$; in particular, $n\geq K$ is a necessary condition for feasibility. The second-echelon component is a mapping $\sigma:\{1,\dots,n\}\rightarrow\{1,\dots,n\}$, where $\sigma(i)$ is the index of the centroid $\mu_{\sigma(i)}$ served from satellite $v_i$. Constraints~\eqref{eq4} and~\eqref{eq6} imply $\sum_{k\in\mathcal{K}}y_i^k=1$ for every satellite, so that Constraints~\eqref{eq8} assign exactly one centroid to each satellite, and Constraints~\eqref{eq9} assign each centroid to exactly one satellite. Hence, $\sigma$ is a permutation of $\{1,\dots,n\}$, and $z_{ij}=1$ if and only if $j=\sigma(i)$.

Given a solution $S$, the waiting time at satellite $v_i$ is the round-trip flight time to its assigned centroid,
\begin{equation}
\label{eq:wait}
w_i=2\,t'_{i\sigma(i)},\qquad i=1,\dots,n,
\end{equation}
and $w_0=0$ as in Section~\ref{Sec3}. The arrival times along route $r_k$ are obtained by the recursion
\begin{equation}
\label{eq:arr}
a^{k}_{\pi_k(0)}=0,\qquad a^{k}_{\pi_k(p)}=a^{k}_{\pi_k(p-1)}+w_{\pi_k(p-1)}+t_{\pi_k(p-1)\pi_k(p)},\qquad p=1,\dots,m_k+1,
\end{equation}
and the objective value of $S$ is
\begin{equation}
\label{eq:fS}
f(S)=\sum_{k\in\mathcal{K}}\sum_{p=1}^{m_k+1}a^{k}_{\pi_k(p)}.
\end{equation}
Let $L_k(S)=\sum_{p=1}^{m_k}d_{\sigma(\pi_k(p))}$ denote the load of truck $k$. Solution $S$ is feasible if
\begin{align}
& L_k(S)\leq Q, && k\in\mathcal{K}, \label{eq:F1}\\
& d_{\sigma(i)}\leq UP, && i=1,\dots,n. \label{eq:F2}
\end{align}
Conditions~\eqref{eq:F1} and~\eqref{eq:F2} are Constraints~\eqref{eq10d} and~\eqref{eq11} expressed in terms of $(T,\sigma)$, since $\sum_{i\in V_s}\sum_{j\in M}d_j\delta^{k}_{ij}=L_k(S)$ and $\sum_{j\in M}d_jz_{ij}=d_{\sigma(i)}$. Proposition~\ref{prop:equiv} shows that the representation loses no generality with respect to the MILP formulation.

\begin{proposition}
\label{prop:equiv}
(i) Every feasible solution $S=(T,\sigma)$ defines a feasible solution of formulation~\eqref{eq1}--\eqref{eq18c} whose objective value equals $f(S)$. (ii) Every feasible solution of formulation~\eqref{eq1}--\eqref{eq18c} defines a feasible solution $S=(T,\sigma)$ whose objective value $f(S)$ does not exceed the objective value~\eqref{eq1} of the MILP solution. Consequently, the minimum of $f(S)$ over all feasible solutions $S$ equals the optimal value of formulation~\eqref{eq1}--\eqref{eq18c}.
\end{proposition}

\begin{proof}
(i) Set $x^{k}_{ij}=1$ if and only if $v_j$ immediately follows $v_i$ in $r_k$; $y^{k}_{i}=1$ if and only if $v_i$ belongs to $r_k$; $z_{ij}=1$ if and only if $j=\sigma(i)$; and $\delta^{k}_{ij}=y^{k}_{i}z_{ij}$. Let $w_i$ be given by~\eqref{eq:wait}, let $a^{k}_{i}$ be given by~\eqref{eq:arr} for the nodes of $r_k$, and let $a^{k}_{i}=0$ for the satellites not in $r_k$. Constraints~\eqref{eq2a}--\eqref{eq6} hold because each $r_k$ is a path from $v_0$ to $v_{n+1}$ that visits at least one satellite, and the routes partition $V_s$. Constraints~\eqref{eq7}--\eqref{eq9} hold because $\sigma$ is a permutation, Constraints~\eqref{eq10a}--\eqref{eq10c} hold by the definition of $\delta^{k}_{ij}$, and Constraints~\eqref{eq10d} and~\eqref{eq11} hold by Conditions~\eqref{eq:F1} and~\eqref{eq:F2}. Constraints~\eqref{eq14} hold with equality for $j=\sigma(i)$ and trivially otherwise. Constraints~\eqref{eq12} hold with equality for the arcs of $r_k$. For any other arc, $x^{k}_{ij}=0$, and the constraint holds because $a^{k}_{j}\geq 0$ and $a^{k}_{i}+w_i+t_{ij}\leq H$. The latter inequality holds because $a^{k}_{i}+w_i$ is either $w_i$, if $v_i$ is not in $r_k$, or the departure time of truck $k$ from $v_i$, which is bounded by $n$ arc traversals and one waiting time at each satellite. Together with the traversal of arc $(i,j)$, this gives at most $n+1$ travel times, so the bound follows from the definition of $H$ in Section~\ref{Sec3}. The same argument gives $a^{k}_{i}\leq H$, so Constraints~\eqref{eq13a} hold, and Constraints~\eqref{eq13b} and~\eqref{eq15}--\eqref{eq18c} hold by construction. Since $a^{k}_{i}=0$ for the satellites not visited by truck $k$, objective function~\eqref{eq1} equals $f(S)$.

(ii) Consider a feasible solution of the MILP. By Constraints~\eqref{eq4} and~\eqref{eq6}, every satellite is visited by exactly one truck. By Constraints~\eqref{eq2a}--\eqref{eq5}, the arcs traversed by truck $k$ form a path from $v_0$ to $v_{n+1}$ through the satellites with $y^{k}_{i}=1$, possibly together with cycles among these satellites. Such cycles are excluded by Constraints~\eqref{eq12}, as shown in Section~\ref{Sec3}. This path defines $r_k$, and $\sigma(i)$ is defined as the unique index $j$ with $z_{ij}=1$, which exists by Constraints~\eqref{eq8} and~\eqref{eq9}. Conditions~\eqref{eq:F1} and~\eqref{eq:F2} follow from Constraints~\eqref{eq10d} and~\eqref{eq11}. By Constraints~\eqref{eq14}, $w_i\geq 2t'_{i\sigma(i)}$. By Constraints~\eqref{eq12} applied to the arcs of $r_k$, and by induction on $p$, the value of $a^{k}_{\pi_k(p)}$ in the MILP solution is not smaller than the value given by recursion~\eqref{eq:arr}. Since all arrival-time variables are nonnegative by Constraints~\eqref{eq18b}, objective function~\eqref{eq1} is not smaller than $f(S)$.
\end{proof}

Unrolling recursion~\eqref{eq:arr}, the travel time on the $p$-th arc of $r_k$ delays the arrivals at $v_{\pi_k(p)}$ and at the $m_k+1-p$ subsequent nodes. The waiting time at $v_{\pi_k(p)}$ delays only the arrivals at the $m_k+1-p$ subsequent nodes. Objective~\eqref{eq:fS} can therefore be written as
\begin{equation}
\label{eq:decomp}
f(S)=\sum_{k\in\mathcal{K}}\left[\,\sum_{p=1}^{m_k+1}\left(m_k+2-p\right)t_{\pi_k(p-1)\pi_k(p)}+\sum_{p=1}^{m_k}\left(m_k+1-p\right)w_{\pi_k(p)}\right].
\end{equation}
Decomposition~\eqref{eq:decomp} has three implications for the design of the algorithm. First, $f(S)$ can be evaluated in $O(n)$ time. Second, the weight of a waiting time decreases along a route: the waiting time at the first satellite of $r_k$ is counted $m_k$ times, whereas the waiting time at the last satellite is counted once. Long drone missions at satellites visited early are therefore particularly costly, a property exploited by the construction heuristic in Section~\ref{sec4:constr}. Third, the two echelons interact only through the waiting times, which depend on $\sigma$, and their positional weights, which depend on $T$. If $\sigma$ is modified while $T$ is kept fixed, the objective changes by $\sum_{i}(m_{k(i)}+1-p(i))(w'_i-w_i)$, where the sum runs over the satellites whose centroids are modified, $k(i)$ and $p(i)$ denote the truck and the position of satellite $v_i$ in $T$, and $w'_i$ denotes the new waiting time. Such a modification can therefore be evaluated without recomputing the arrival times.

Regarding feasibility, the number of drones $U$ and their payload capacity $P$ are identical at all satellites, and $\sigma$ is a permutation. Condition~\eqref{eq:F2} therefore holds for every assignment $\sigma$ if and only if $\max_{j\in M}d_j\leq UP$. This condition is a property of the instance and is verified once before the search starts; if it is violated, the instance is infeasible. During the search, only Condition~\eqref{eq:F1} has to be verified. Since the load $L_k(S)$ depends on both components of the solution, a modification of $\sigma$ can violate Condition~\eqref{eq:F1} even if $T$ is unchanged. No flight-range condition is required, since all satellite--centroid pairs are assumed to satisfy the flight-range requirements of the drones (Section~\ref{Sec3}).

\subsection{Construction heuristic}
\label{sec4:constr}
The construction heuristic builds a solution in three stages: the truck routes are seeded with one satellite each, the remaining satellites are inserted into the routes by a randomized cheapest-insertion rule, and centroids are assigned to satellites by a randomized nearest-centroid rule. The two randomized stages, and two of the neighborhood operators in Section~\ref{sec4:ls}, use the following biased selection rule. Given a finite set $\mathcal{C}$ of candidates and a nonnegative score $s(c)$ for each candidate $c\in\mathcal{C}$, candidate $c$ is selected with probability
\begin{equation}
\label{eq:bias}
\psi(c)=\frac{\left(s(c)+\varepsilon\right)^{-1}}{\sum_{c'\in\mathcal{C}}\left(s(c')+\varepsilon\right)^{-1}},\qquad c\in\mathcal{C},
\end{equation}
where $\varepsilon>0$ is a small constant that keeps the probabilities well defined when a score equals zero. Candidates with small scores are the most likely to be selected, but every candidate is selected with positive probability, so that repeated constructions produce different solutions.

\paragraph{Seeding of the truck routes}
Each of the $K$ routes is initialized with one seed satellite, which leaves $Y=n-K\geq 0$ satellites to be inserted. Since each insertion adds a satellite to one route, at most $Y$ routes receive insertions, and at least $C=\max\{K-Y,0\}$ routes remain single-satellite routes in the constructed solution. By Decomposition~\eqref{eq:decomp} with $m_k=1$, and since $t_{i,n+1}=t_{0i}$ because $v_0$ and $v_{n+1}$ represent the same depot, a single-satellite route serving $v_i$ contributes $3t_{0i}+w_i$ to the objective. Single-satellite routes should therefore serve satellites close to the depot. The seeds are selected by one of two procedures, depending on the value of $C$.
\begin{itemize}
\item \textit{Procedure H1} ($C>0$, i.e., $Y<K$). The $C$ satellites with the shortest travel times from the depot seed $C$ routes, which anticipates the single-satellite routes of the constructed solution. The remaining $K-C=Y$ routes are seeded with the $Y$ satellites with the longest travel times from the depot, so that these routes extend in distinct directions during the insertion stage.
\item \textit{Procedure H2} ($C=0$, i.e., $Y\geq K$). The first seed is the satellite with the longest travel time from the depot. Each subsequent seed is the unselected satellite with the longest travel time from the previously selected seed, until $K$ seeds are selected. The routes thus start from satellites that are far from the depot and far from each other.
\end{itemize}

\paragraph{Randomized insertion}
The remaining $Y$ satellites are inserted one at a time. For a satellite $v_i$, every pair $(v_a,v_b)$ of consecutive nodes in a route, including the pairs that contain $v_0$ or $v_{n+1}$, defines a candidate insertion position, with insertion cost
\begin{equation}
\label{eq:inscost}
c_i(a,b)=t_{ai}+t_{ib}-t_{ab}.
\end{equation}
Since truck travel times are proportional to Manhattan distances (Section~\ref{Sec5}), they are symmetric and satisfy the triangle inequality, so that $c_i(a,b)\geq 0$. Under the Manhattan metric, $c_i(a,b)=0$ whenever $v_i$ lies in the axis-parallel rectangle spanned by $v_a$ and $v_b$; zero insertion costs are therefore frequent, which motivates the constant $\varepsilon$ in rule~\eqref{eq:bias}. The insertion position of $v_i$ is selected by rule~\eqref{eq:bias}, with the set of all candidate positions of all routes as $\mathcal{C}$ and $c_i(a,b)$ as score. Truck capacity is not considered at this stage, because the loads depend on the assignment $\sigma$, which is constructed afterwards; capacity violations are resolved by the repair operator of Section~\ref{sec4:repair}.

\paragraph{Assignment of centroids to satellites}
The satellites are arranged in a list $A$ in increasing order of their positions in their routes, with ties broken by truck index:
\begin{equation}
\label{eq:listA}
A=\left(v_{\pi_1(1)},v_{\pi_2(1)},\dots,v_{\pi_K(1)},\,v_{\pi_1(2)},v_{\pi_2(2)},\dots\right),
\end{equation}
where the routes with fewer than $p$ satellites are skipped at position $p$. Figure~\ref{fig2} illustrates the construction of $A$ for the two routes of Figure~\ref{fig1}. The satellites are processed in the order of $A$, and $\Omega$ denotes the set of indices of the centroids assigned so far, initially empty. For satellite $v_i$, a centroid index $j\in\{1,\dots,n\}\setminus\Omega$ is selected by rule~\eqref{eq:bias} with score $t'_{ij}$, and we set $\sigma(i)=j$ and $\Omega\leftarrow\Omega\cup\{j\}$. After the last satellite of $A$ has been processed, $\sigma$ is a permutation of $\{1,\dots,n\}$. By Decomposition~\eqref{eq:decomp}, the waiting time at the $p$-th satellite of route $r_k$ carries the weight $m_k+1-p$, which is largest at the first positions of the routes. Processing the satellites in the order of $A$ gives the satellites visited early the first choice among the nearby centroids, so that short drone missions tend to be placed where waiting times are most costly. This ordering approximates the ordering by weight, since the weight also depends on the route length $m_k$.

\begin{figure}[H]
\centering
    \includegraphics[width = 0.9\textwidth]{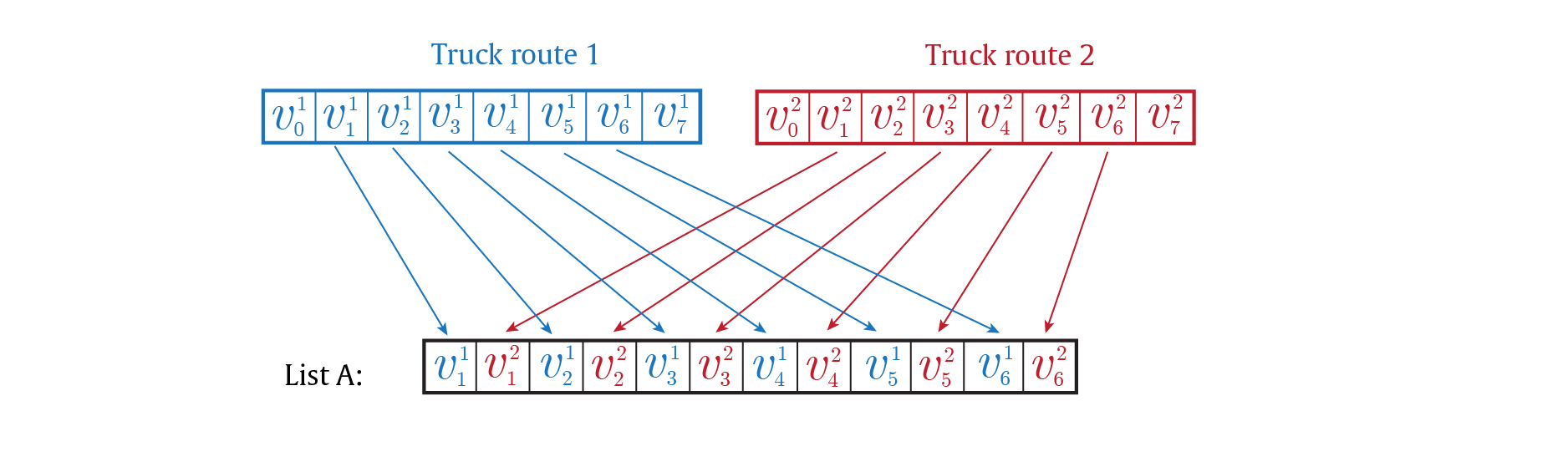}
    \caption{Construction of list $A$ for two truck routes}
    \label{fig2}
\end{figure}

Algorithm~\ref{alg:construct} summarizes the construction heuristic. Each insertion evaluates $O(n)$ candidate positions, and each assignment evaluates at most $n$ centroids, so that a construction requires $O(n^{2})$ time. The constructed solution satisfies Condition~\eqref{eq:F2} whenever the instance is feasible (Section~\ref{sec4:repr}), but it may violate Condition~\eqref{eq:F1}. In that case, the repair operator of Section~\ref{sec4:repair} is applied.

\begin{algorithm}[htbp]
\scriptsize
\caption{Greedy randomized construction heuristic}
\label{alg:construct}
\begin{algorithmic}[1]
\Require satellites $V_s$, centroids $M$, number of trucks $K$
\State $Y\gets n-K$; \quad $C\gets\max\{K-Y,0\}$ \label{alg:construct:seed1}
\If{$C>0$}
    \State $T\gets\Call{H1}{V_s,C}$ \Comment{$K$ single-satellite routes}
\Else
    \State $T\gets\Call{H2}{V_s}$
\EndIf \label{alg:construct:seed2}
\For{each satellite $v_i$ not used as a seed} \label{alg:construct:ins1}
    \State compute $c_i(a,b)$ by~\eqref{eq:inscost} for every pair $(v_a,v_b)$ of consecutive nodes of every route of $T$
    \State insert $v_i$ between the pair $(v_a,v_b)$ selected by rule~\eqref{eq:bias} with score $c_i(a,b)$
\EndFor \label{alg:construct:ins2}
\State build list $A$ from $T$ by~\eqref{eq:listA}; \quad $\Omega\gets\emptyset$ \label{alg:construct:ass1}
\For{each satellite $v_i$ in the order of $A$}
    \State select $j\in\{1,\dots,n\}\setminus\Omega$ by rule~\eqref{eq:bias} with score $t'_{ij}$
    \State $\sigma(i)\gets j$; \quad $\Omega\gets\Omega\cup\{j\}$
\EndFor \label{alg:construct:ass2}
\State \Return $S=(T,\sigma)$
\end{algorithmic}
\end{algorithm}

\subsection{Repair operator}
\label{sec4:repair}
The construction heuristic builds the truck routes before the assignment $\sigma$ is known, so the constructed solution may violate Condition~\eqref{eq:F1}. The repair operator restores feasibility by modifying the routes $T$ while keeping $\sigma$ fixed. The demand to be transported to satellite $v_i$ is then fixed and equal to $d_{\sigma(i)}$, and finding feasible routes amounts to partitioning the satellites into $K$ nonempty groups whose loads do not exceed $Q$. The sequence of the satellites within a route does not affect Condition~\eqref{eq:F1}.

Keeping $\sigma$ fixed entails no loss of generality. Since $\sigma$ is a permutation, the multiset of demands $\{d_{\sigma(i)}:i=1,\dots,n\}$ coincides with the multiset $\{d_j:j=1,\dots,n\}$ of centroid demands for every assignment $\sigma$. Hence, a feasible solution with assignment $\sigma$ exists if and only if the centroid demands can be packed into $K$ bins of capacity $Q$, and this condition is independent of $\sigma$. The requirement that every route contains at least one satellite imposes no additional restriction when $n\geq K$. If a packing leaves a bin empty, another bin contains at least two items, and moving one of them to the empty bin preserves feasibility, since each item fits in a bin on its own. Consequently, if the instance is feasible, a feasible solution exists for every assignment $\sigma$ produced by the construction heuristic. Deciding whether such a packing exists is the decision version of the bin packing problem, which is NP-complete \citep{garey1979computers}. The repair operator is therefore a heuristic, which may fail even if a feasible packing exists. It consists of three stages, each of which is applied only if the previous stages have not restored feasibility.

\paragraph{Stage 1: reinsertion}
Satellites are removed from every route whose load exceeds $Q$ until the load of the route satisfies Condition~\eqref{eq:F1}, keeping at least one satellite in each route. The removed satellites are reinserted one at a time by the randomized insertion of Section~\ref{sec4:constr}. A candidate position is excluded, i.e., its selection probability in rule~\eqref{eq:bias} is set to zero, if the residual capacity of its route is smaller than the demand of the satellite.

\paragraph{Stage 2: merging}
If some removed satellites cannot be reinserted, the operator searches for two routes whose combined load does not exceed $Q$ and merges them into a single route, which frees one truck. The satellites that could not be reinserted form the route of the freed truck, so that the number of routes remains $K$. This stage succeeds if such a pair of routes exists and the total demand of the remaining satellites does not exceed $Q$.

\paragraph{Stage 3: re-packing}
If the solution is still infeasible, all satellites are reassigned to the trucks by the packing heuristic of Algorithm~\ref{alg:repack}. The trucks are filled one at a time. Each truck first receives the unassigned satellite with the largest demand and then, as long as the capacity permits, the unassigned satellites with the smallest demands. Starting each truck with the largest remaining item ensures that large demands, which are the most difficult to place, are assigned while the trucks are still empty, and completing each truck with small items keeps the number of items per truck large. The heuristic fails if some satellite remains unassigned after $K$ trucks have been filled. Otherwise, the satellites assigned to each truck are sequenced into a route by the randomized insertion of Section~\ref{sec4:constr}.

\begin{algorithm}[htbp]
\scriptsize
\caption{Re-packing heuristic (Stage 3 of the repair operator)}
\label{alg:repack}
\begin{algorithmic}[1]
\Require assignment $\sigma$, demands $d_j$, number of trucks $K$, capacity $Q$
\State $R\gets\{1,\dots,n\}$ \Comment{indices of the unassigned satellites}
\For{$k=1$ \textbf{to} $K$}
    \If{$R=\emptyset$} \Return \textsc{Fail} \EndIf
    \State $i\gets\operatorname*{arg\,max}_{i'\in R}d_{\sigma(i')}$
    \State $B_k\gets\{i\}$; \quad $R\gets R\setminus\{i\}$
    \While{$R\neq\emptyset$ \textbf{and} $\sum_{i'\in B_k}d_{\sigma(i')}+\min_{i'\in R}d_{\sigma(i')}\leq Q$}
        \State $i\gets\operatorname*{arg\,min}_{i'\in R}d_{\sigma(i')}$
        \State $B_k\gets B_k\cup\{i\}$; \quad $R\gets R\setminus\{i\}$
    \EndWhile
\EndFor
\If{$R\neq\emptyset$} \Return \textsc{Fail} \EndIf
\State sequence the satellites $\{v_i:i\in B_k\}$ into route $r_k$ by randomized insertion, for all $k\in\mathcal{K}$
\State \Return $T=(r_1,\dots,r_K)$
\end{algorithmic}
\end{algorithm}

The check in Line 3 of Algorithm~\ref{alg:repack} cannot fail when $n\geq K$ and the demands are positive. Each item fits in a bin on its own (otherwise, the instance is infeasible), so the first item assigned to a truck never violates the capacity, and each truck receives at least one item. The heuristic therefore fails only if items remain after the last truck has been filled.

If the repair operator fails, a new solution is constructed with new random choices. Algorithm~\ref{alg:feasible} repeats the construction and the repair up to $R_{\max}$ times. If no feasible solution is obtained, the algorithm terminates and reports that no feasible solution was found. Since the success of Stages 1 and 2 depends on the routes and on the assignment produced by the construction heuristic, repeated attempts can succeed where an earlier attempt failed.

\begin{algorithm}[htbp]
\scriptsize
\caption{Construction of a feasible solution}
\label{alg:feasible}
\begin{algorithmic}[1]
\Require maximum number of attempts $R_{\max}$
\For{$r=1$ \textbf{to} $R_{\max}$}
    \State $S=(T,\sigma)\gets$ \Call{Construct}{\,} \Comment{Algorithm~\ref{alg:construct}}
    \If{$S$ satisfies Condition~\eqref{eq:F1}} \Return $S$ \EndIf
    \State $T\gets$ \Call{Repair}{$T,\sigma$} \Comment{Stages 1--3}
    \If{$T\neq\textsc{Fail}$} \Return $(T,\sigma)$ \EndIf
\EndFor
\State \Return \textsc{Fail}
\end{algorithmic}
\end{algorithm}

\subsection{Neighborhood operators and local search}
\label{sec4:ls}
The search uses seven neighborhood operators. Operators N1--N6 modify the truck routes $T$ and keep the assignment $\sigma$ fixed, whereas operator N7 modifies $\sigma$ and keeps $T$ fixed. Each call of an operator samples a single move at random and returns the resulting candidate solution; if the sampled move is not applicable to the current solution or the candidate violates Condition~\eqref{eq:F1}, no candidate is returned. Every candidate is required to keep at least one satellite in each route, as required by Constraints~\eqref{eq2b}. Each neighborhood contains $O(n^{2})$ solutions, whereas sampling a move and evaluating the candidate by Equation~\eqref{eq:decomp} requires $O(n)$ time. Sampling therefore keeps the cost of an operator call independent of the neighborhood size. This property is used both in the perturbation step of the main algorithm (Section~\ref{sec4:hybrid}) and, through repeated sampling, in the local search described below. Table~\ref{tab:operators} summarizes the operators, and Figure~\ref{fig:ops} illustrates operators N1--N6.

\begin{table}[htbp]
\centering
\scriptsize
\caption{Neighborhood operators}
\label{tab:operators}
\begin{tabular}{lllll}
\toprule
Operator & Modified component & Scope & Move & Capacity check \\
\midrule
N1 & $T$ & intra-route & relocate a satellite & not required \\
N2 & $T$ & intra-route & biased swap of two satellites & not required \\
N3 & $T$ & intra-route & 2-opt & not required \\
N4 & $T$ & inter-route & swap of two satellites & both routes \\
N5 & $T$ & inter-route & 2-opt* & both routes \\
N6 & $T$ & inter-route & relocate a satellite & receiving route \\
N7 & $\sigma$ & -- & biased swap of two centroids & if the trucks differ \\
\bottomrule
\end{tabular}
\end{table}

\paragraph{Intra-route operators}
Operators N1--N3 change the order of the satellites within a route $r_k$, which is selected uniformly at random among the routes with at least two satellites. Since the set of satellites of the route is unchanged, the load $L_k(S)$ is unchanged, and the candidate satisfies Condition~\eqref{eq:F1} whenever the current solution does.
\begin{itemize}
\item \textit{N1 -- Relocation.} A satellite of $r_k$, selected uniformly at random, is removed from its position and reinserted at a different position of $r_k$, also selected uniformly at random.
\item \textit{N2 -- Biased swap.} A satellite $v_i$ of $r_k$ is selected uniformly at random. A second satellite $v_j$ of $r_k$, $j\neq i$, is selected by rule~\eqref{eq:bias} with score $t_{ij}$, and the positions of $v_i$ and $v_j$ are exchanged. The bias favors swaps between nearby satellites, which change the travel times of the route only moderately.
\item \textit{N3 -- 2-opt.} Two positions $p$ and $q$ with $0\leq p<q-1\leq m_k-1$ are selected at random. The arcs $(v_{\pi_k(p)},v_{\pi_k(p+1)})$ and $(v_{\pi_k(q)},v_{\pi_k(q+1)})$ are replaced by the arcs $(v_{\pi_k(p)},v_{\pi_k(q)})$ and $(v_{\pi_k(p+1)},v_{\pi_k(q+1)})$, which reverses the segment $(v_{\pi_k(p+1)},\dots,v_{\pi_k(q)})$. Since the truck travel times are symmetric, the travel times of the arcs within the reversed segment are unchanged. The objective nevertheless changes through the two new arcs and through the positional weights in Decomposition~\eqref{eq:decomp} of the arcs and waiting times in the segment, whose order is reversed.
\end{itemize}

\paragraph{Inter-route operators}
Operators N4--N6 move satellites between two distinct routes $r_{k_1}$ and $r_{k_2}$, selected uniformly at random. Since they change the sets of satellites served by the two trucks, a candidate is returned only if both modified routes satisfy Condition~\eqref{eq:F1} and contain at least one satellite. The other routes are not affected.
\begin{itemize}
\item \textit{N4 -- Swap.} A satellite of $r_{k_1}$ and a satellite of $r_{k_2}$, each selected uniformly at random, exchange their positions.
\item \textit{N5 -- 2-opt*.} A cut position is selected at random in each route, which splits the routes into $r_{k_1}=(v_0,A_1,B_1,v_{n+1})$ and $r_{k_2}=(v_0,A_2,B_2,v_{n+1})$, where $A_1$, $B_1$, $A_2$, and $B_2$ are sequences of satellites. Two recombinations are generated. The first exchanges the tails of the routes, giving $(v_0,A_1,B_2,v_{n+1})$ and $(v_0,A_2,B_1,v_{n+1})$. The second connects the heads and the tails of the two routes, giving $(v_0,A_1,\overleftarrow{A_2},v_{n+1})$ and $(v_0,\overleftarrow{B_1},B_2,v_{n+1})$, where $\overleftarrow{X}$ denotes sequence $X$ in reverse order. Among the recombinations that satisfy Condition~\eqref{eq:F1} and leave both routes nonempty, the one with the smaller objective value is returned.
\item \textit{N6 -- Relocation between routes.} A satellite of $r_{k_2}$, which must contain at least two satellites, is selected uniformly at random, removed from $r_{k_2}$, and inserted into $r_{k_1}$ at a position selected uniformly at random.
\end{itemize}

\begin{figure}[htbp]
\centering
\begin{subfigure}{0.32\textwidth}
    \centering
    \includegraphics[width=\linewidth]{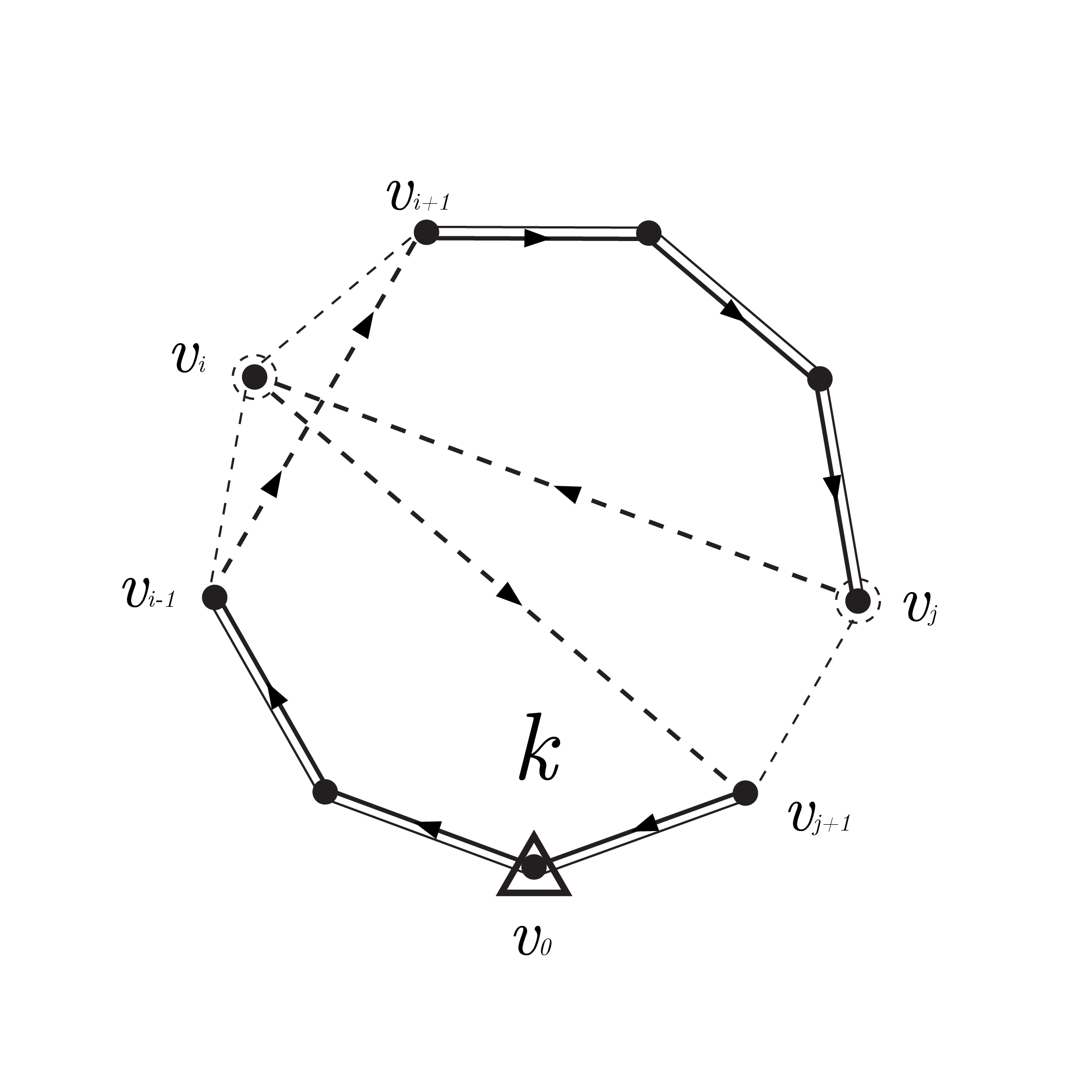}
    \caption{N1 -- Relocation (intra-route)}
    \label{fig:n1}
\end{subfigure}\hfill
\begin{subfigure}{0.32\textwidth}
    \centering
    \includegraphics[width=\linewidth]{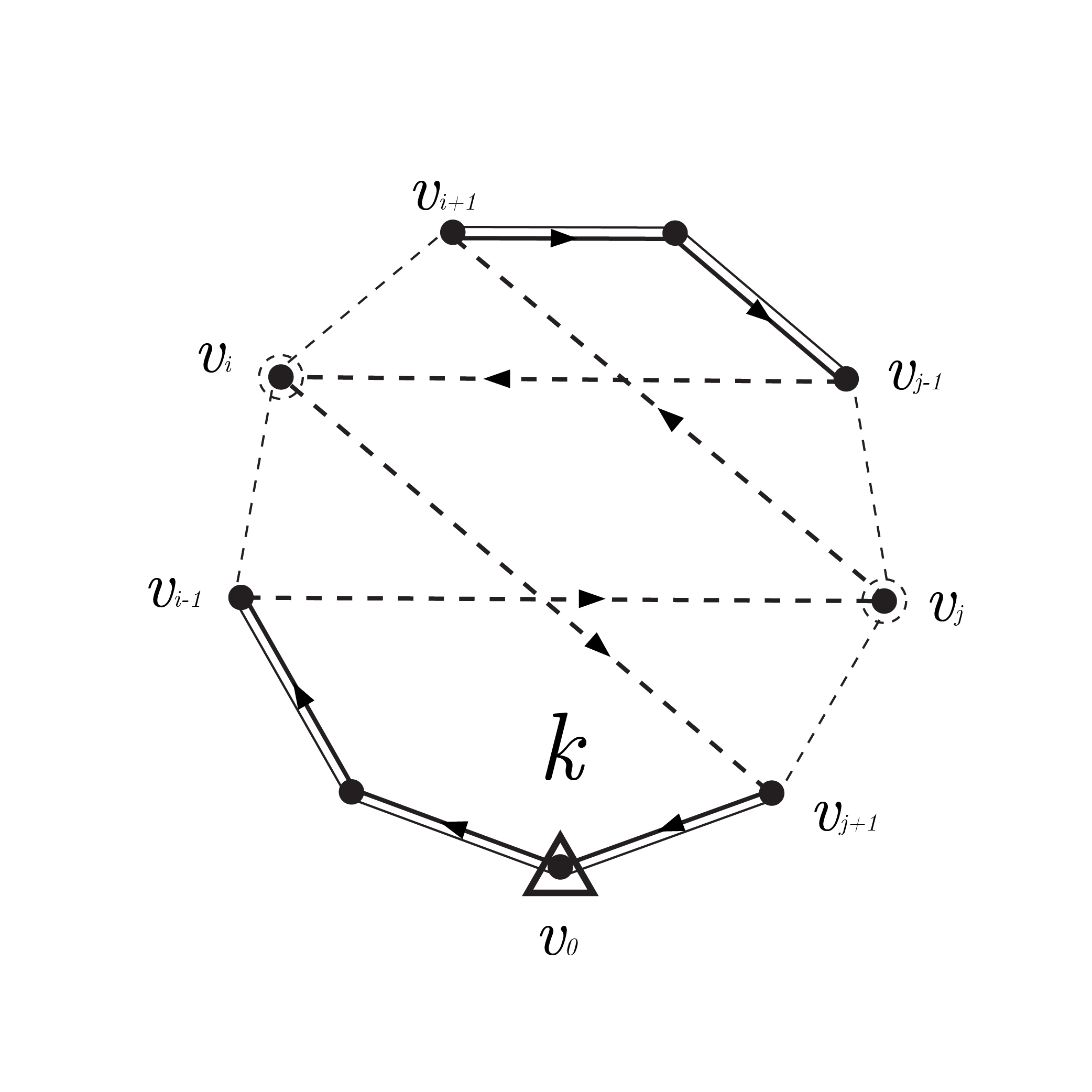}
    \caption{N2 -- Biased swap (intra-route)}
    \label{fig:n2}
\end{subfigure}\hfill
\begin{subfigure}{0.32\textwidth}
    \centering
    \includegraphics[width=\linewidth]{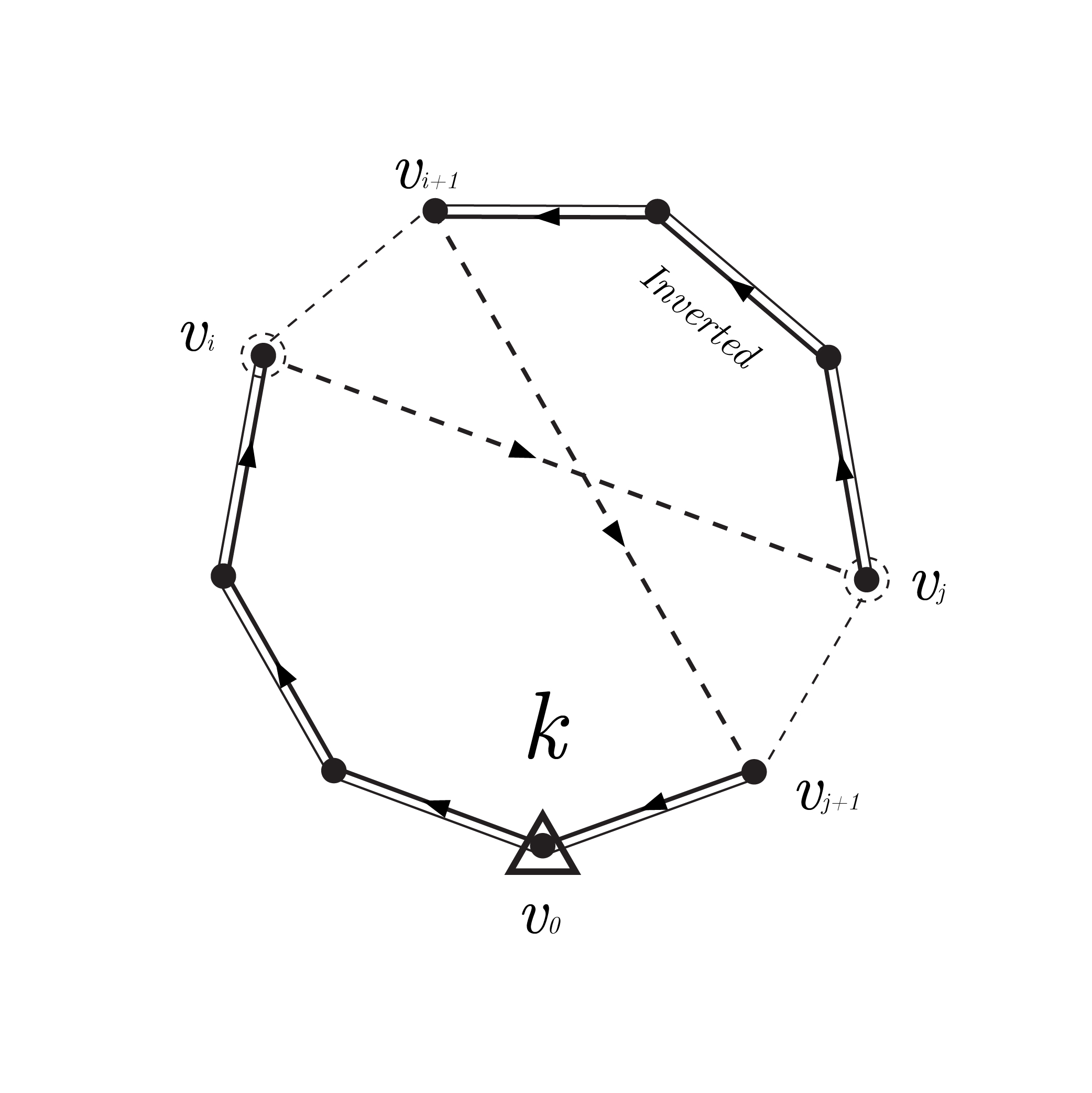}
    \caption{N3 -- 2-opt (intra-route)}
    \label{fig:n3}
\end{subfigure}

\vspace{2mm}
\begin{subfigure}{0.32\textwidth}
    \centering
    \includegraphics[width=\linewidth]{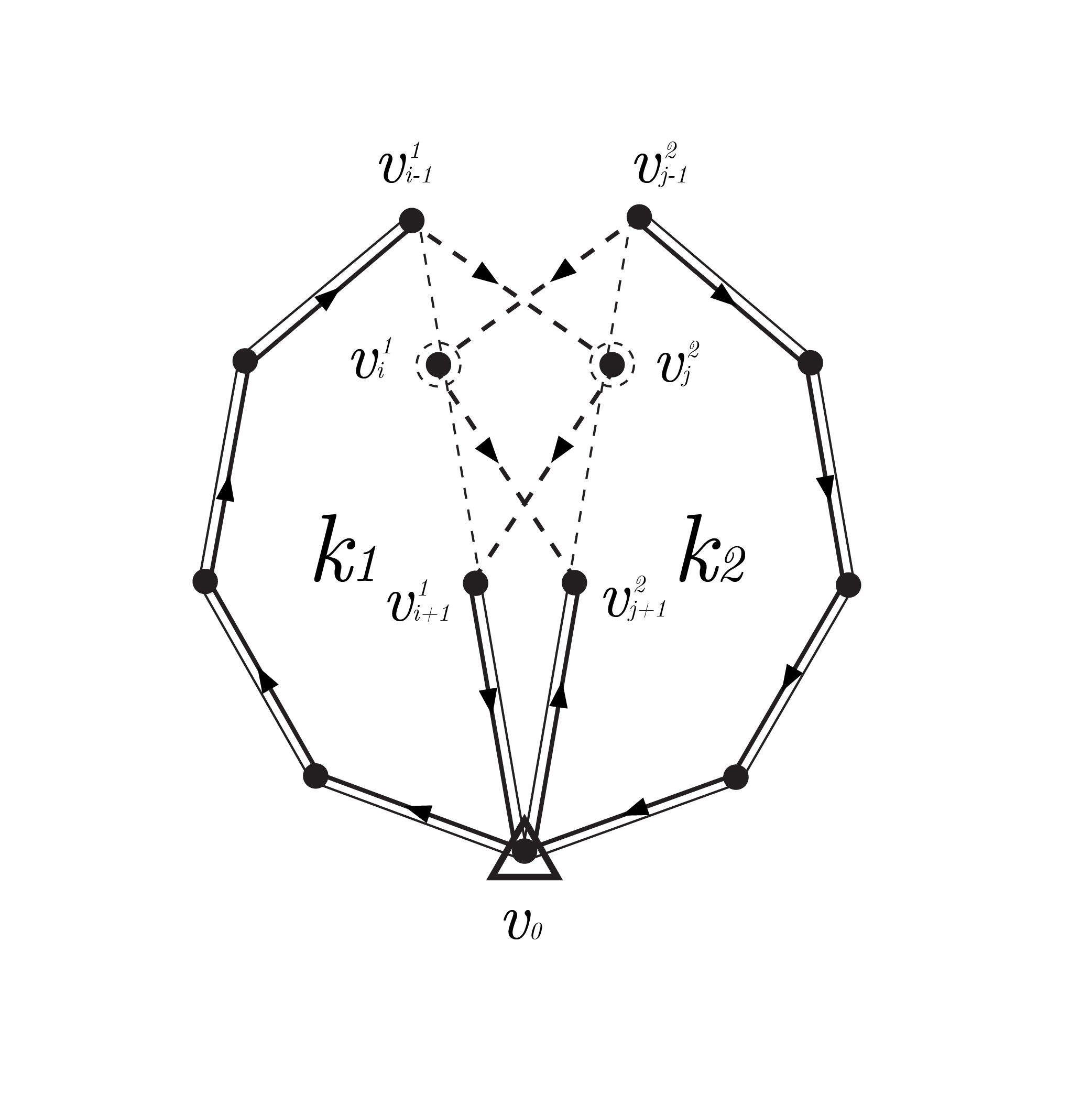}
    \caption{N4 -- Swap (inter-route)}
    \label{fig:n4}
\end{subfigure}\hfill
\begin{subfigure}{0.32\textwidth}
    \centering
    \includegraphics[width=\linewidth]{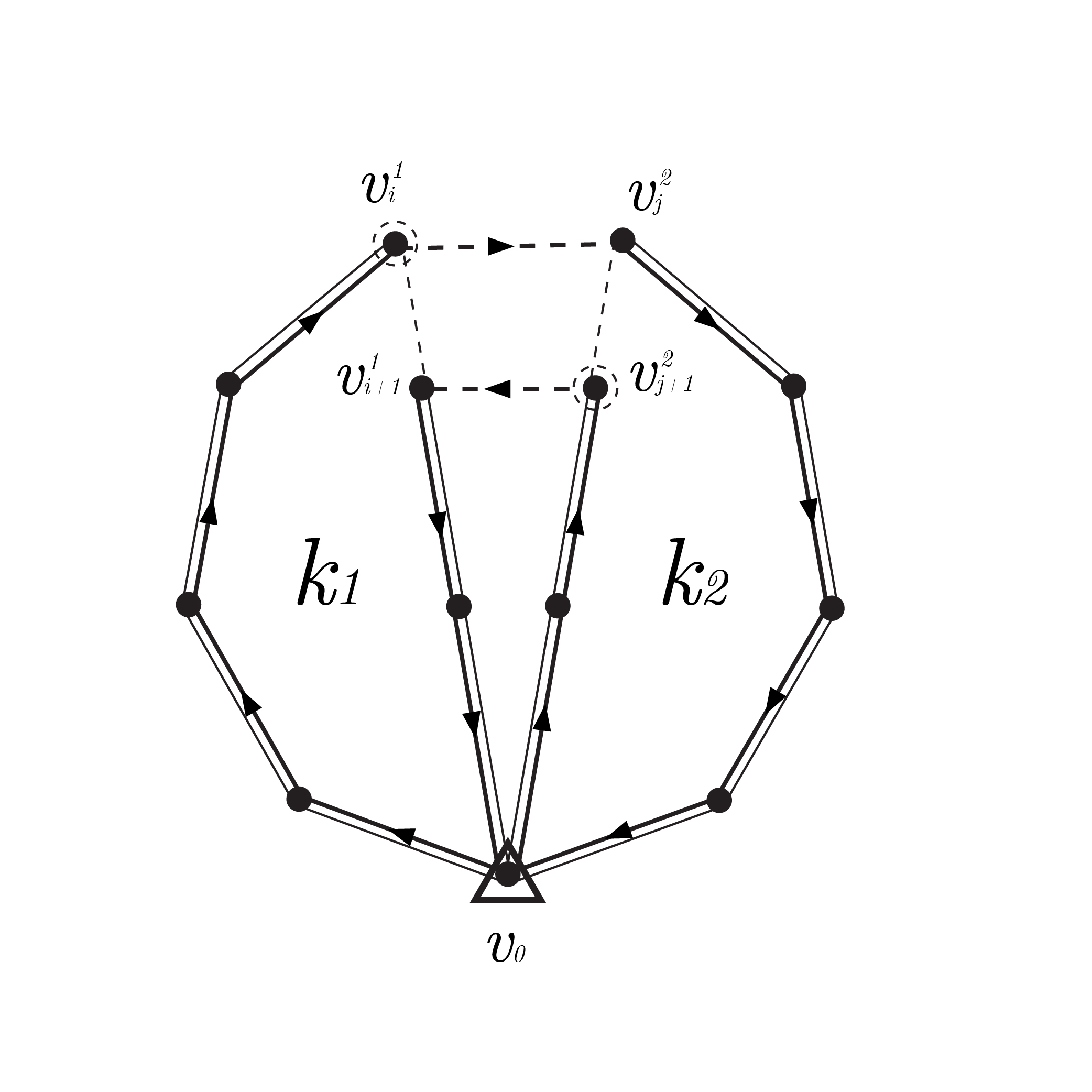}
    \caption{N5 -- 2-opt* (inter-route)}
    \label{fig:n5}
\end{subfigure}\hfill
\begin{subfigure}{0.32\textwidth}
    \centering
    \includegraphics[width=\linewidth]{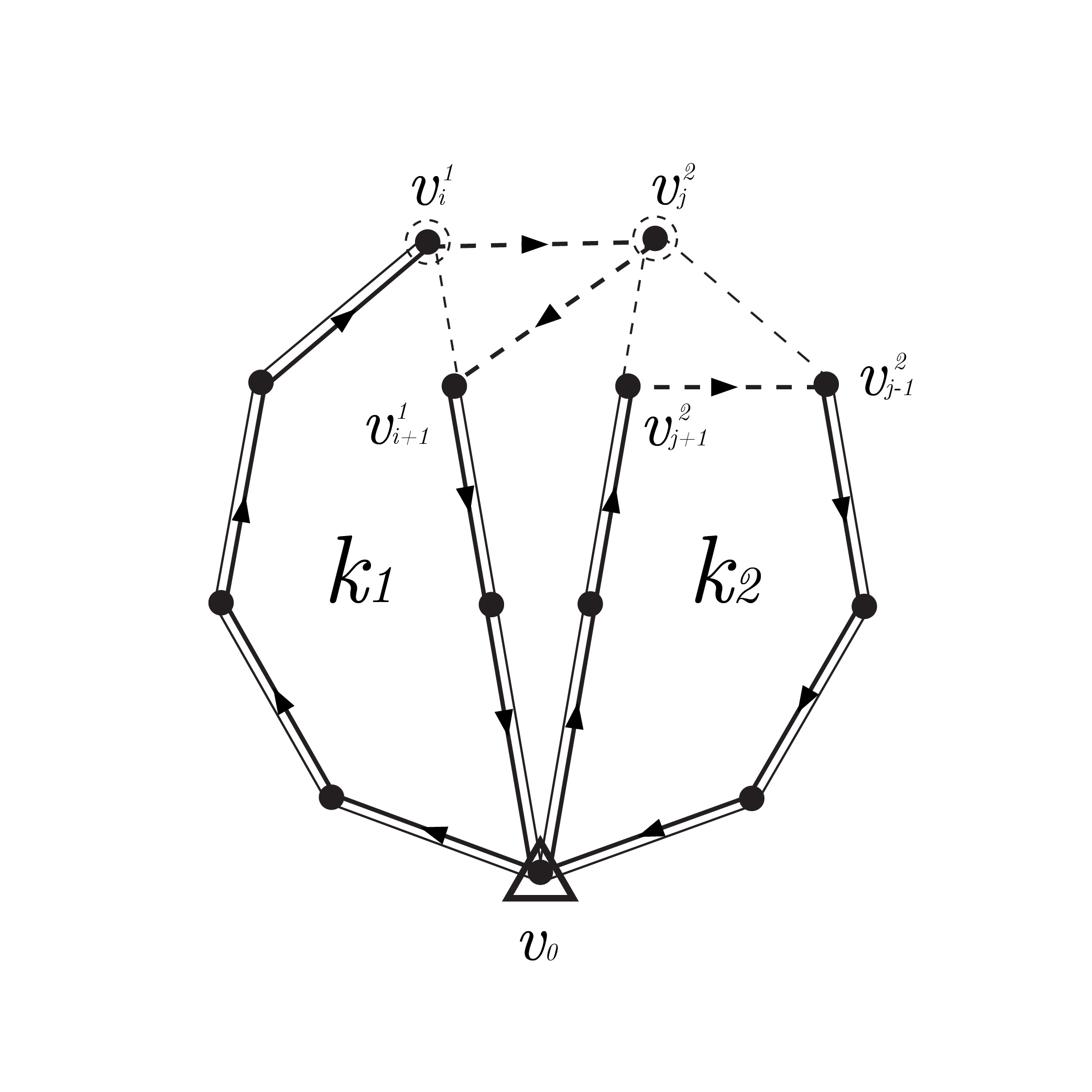}
    \caption{N6 -- Relocation (inter-route)}
    \label{fig:n6}
\end{subfigure}
\caption{Neighborhood operators for the truck routes}
\label{fig:ops}
\end{figure}

\paragraph{Assignment operator}
\textit{N7 -- Biased centroid swap.} A satellite $v_i$ is selected uniformly at random. A second satellite $v_j$, $j\neq i$, is selected by rule~\eqref{eq:bias}, with the Euclidean distance between the centroids $\mu_{\sigma(i)}$ and $\mu_{\sigma(j)}$ as score. The centroids of the two satellites are then exchanged: $\sigma'(i)=\sigma(j)$, $\sigma'(j)=\sigma(i)$, and $\sigma'(l)=\sigma(l)$ for $l\notin\{i,j\}$. Since $\sigma'$ is again a permutation, the candidate satisfies Condition~\eqref{eq:F2}. The bias favors exchanges between nearby centroids, for which the flight times of both satellites change only moderately. By Decomposition~\eqref{eq:decomp}, the move changes only the waiting times $w_i$ and $w_j$, so its effect on the objective can be computed from the positions of $v_i$ and $v_j$ in their routes. If $v_i$ and $v_j$ belong to the same route, the load of that route is unchanged. Otherwise, the loads of both routes change by $\pm(d_{\sigma(j)}-d_{\sigma(i)})$, and the candidate is returned only if both routes satisfy Condition~\eqref{eq:F1}. N7 is the only operator that modifies $\sigma$, and hence the only operator that can change the loads of the trucks without modifying $T$.

\paragraph{Local search}
The local search, outlined in Algorithm~\ref{alg:ls}, is a first-improvement descent over sampled neighborhoods. A sweep processes the operators in the fixed order N1--N7. For each operator, $\kappa$ candidates are sampled, and the best feasible candidate replaces the current solution if its objective value is strictly smaller. Sweeps are repeated until a complete sweep yields no improvement. Since every accepted candidate strictly decreases the objective value and the number of solutions is finite, the procedure terminates. Each sweep samples $7\kappa$ candidates, which requires $O(\kappa n)$ time. Because the neighborhoods are sampled rather than enumerated, the returned solution is not necessarily a local optimum with respect to the complete neighborhoods N1--N7. It is a solution for which a complete sweep of $\kappa$ samples per operator has found no improving candidate. The sampling parameter $\kappa$ controls the trade-off between the thoroughness of the descent and its cost.

\begin{algorithm}[htbp]
\scriptsize
\caption{Sampled first-improvement local search}
\label{alg:ls}
\begin{algorithmic}[1]
\Require feasible solution $S$, number of samples per operator $\kappa$
\Repeat
    \State \textit{improved} $\gets$ \textsc{false}
    \For{each operator $N\in(\mathrm{N1},\mathrm{N2},\dots,\mathrm{N7})$}
        \State sample $\kappa$ candidates from $N(S)$ and let $S'$ be the best feasible candidate
        \If{$S'$ exists \textbf{and} $f(S')<f(S)$}
            \State $S\gets S'$; \quad \textit{improved} $\gets$ \textsc{true}
        \EndIf
    \EndFor
\Until{\textit{improved} $=$ \textsc{false}}
\State \Return $S$
\end{algorithmic}
\end{algorithm}

\subsection{Path relinking and elite pool}
\label{sec4:pr}
Path relinking is applied to the second-echelon assignment $\sigma$, with the truck routes $T$ kept fixed. Two properties motivate this choice. First, the assignments form the set of permutations of $\{1,\dots,n\}$, on which the attributes in which two solutions differ, and hence a relinking path, are naturally defined. Second, by Decomposition~\eqref{eq:decomp}, an assignment can be evaluated with any fixed set of routes, so assignments found in different parts of the search can be combined with the current routes. For two assignments $\sigma$ and $\sigma'$, let
\begin{equation}
\label{eq:delta}
\Delta(\sigma,\sigma')=\{i\in\{1,\dots,n\}:\sigma(i)\neq\sigma'(i)\}
\end{equation}
denote the set of satellites whose centroids differ.

\paragraph{Relinking procedure}
Given the current routes $T$, an initial assignment $\sigma^{s}$, and a guiding assignment $\sigma^{g}$, path relinking (Algorithm~\ref{alg:pr}) transforms $\sigma^{s}$ into $\sigma^{g}$ by a sequence of swaps. Let $\sigma$ denote the current assignment on the path, initially $\sigma^{s}$. At each step, the satellite $v_i$ with the smallest index $i\in\Delta(\sigma,\sigma^{g})$ is selected. The satellite $v_u$ that currently serves the centroid required by the guide, i.e., $u=\sigma^{-1}(\sigma^{g}(i))$, is identified, and the centroids of $v_i$ and $v_u$ are exchanged. After the exchange, $\sigma(i)=\sigma^{g}(i)$, so each step is a move of the type performed by operator N7, directed towards the guiding solution. Every intermediate assignment is evaluated with the fixed routes $T$, and the best assignment on the path that satisfies Condition~\eqref{eq:F1} is returned, including the two endpoints. Since the loads $L_k$ depend on $\sigma$, intermediate assignments may violate Condition~\eqref{eq:F1} even if both endpoints satisfy it; such assignments are traversed but not returned. Figure~\ref{fig:pr} illustrates the procedure for an instance with seven satellites.

The path generated by the procedure has the following property. Consider the permutation $\rho=\sigma^{-1}\circ\sigma^{g}$, which maps each satellite index $i$ to the index of the satellite that serves centroid $\sigma^{g}(i)$ in the current assignment. The fixed points of $\rho$ are the satellites not in $\Delta(\sigma,\sigma^{g})$, and the satellites in $\Delta(\sigma,\sigma^{g})$ form $c$ cycles of $\rho$, each of length at least two. A step with selected satellite $v_i$ removes $i$ from its cycle, which reduces the length of that cycle by one and resolves a cycle of length two completely; satellite $v_i$ is not involved in any later step. A cycle of length $\ell$ is therefore resolved in $\ell-1$ steps, and the path from $\sigma^{s}$ to $\sigma^{g}$ consists of exactly
\begin{equation}
\label{eq:prlength}
\left|\Delta(\sigma^{s},\sigma^{g})\right|-c\;\leq\;n-1
\end{equation}
steps, where $c$ is the number of cycles of $(\sigma^{s})^{-1}\circ\sigma^{g}$ of length at least two. This number equals the minimum number of exchanges required to transform one permutation into the other \citep{deza1998metrics}. The procedure therefore follows a shortest path in the exchange distance, and each step reduces the distance to the guiding solution by exactly one. In the example of Figure~\ref{fig:pr}, the two assignments differ in seven satellites that form two cycles, and the path consists of five exchanges. Each intermediate assignment is evaluated in $O(n)$ time by Equation~\eqref{eq:decomp}, so a relinking path requires $O(n^{2})$ time. Since each exchange modifies only two waiting times, the evaluation can also be performed incrementally.

In the main algorithm, relinking is performed in both directions between the current assignment and an elite assignment: forward, from the current assignment towards the elite assignment, and backward, from the elite assignment towards the current assignment. The two paths connect the same endpoints but, since the satellites are selected in index order from different starting assignments, they generally traverse different intermediate assignments. The better of the two results is retained.

\begin{algorithm}[htbp]
\scriptsize
\caption{Path relinking on the second-echelon assignment}
\label{alg:pr}
\begin{algorithmic}[1]
\Require truck routes $T$, initial assignment $\sigma^{s}$, guiding assignment $\sigma^{g}$
\State $\sigma\gets\sigma^{s}$
\State $\sigma^{*}\gets$ the better of $\sigma^{s}$ and $\sigma^{g}$ among those for which $(T,\cdot)$ satisfies Condition~\eqref{eq:F1}
\While{$\Delta(\sigma,\sigma^{g})\neq\emptyset$}
    \State $i\gets\min\Delta(\sigma,\sigma^{g})$
    \State $u\gets\sigma^{-1}\!\left(\sigma^{g}(i)\right)$
    \State exchange $\sigma(i)$ and $\sigma(u)$ \Comment{now $\sigma(i)=\sigma^{g}(i)$}
    \If{$(T,\sigma)$ satisfies Condition~\eqref{eq:F1} \textbf{and} $f(T,\sigma)<f(T,\sigma^{*})$}
        \State $\sigma^{*}\gets$ a copy of $\sigma$
    \EndIf
\EndWhile
\State \Return $\sigma^{*}$
\end{algorithmic}
\end{algorithm}

\begin{figure}[H]
\centering
    \includegraphics[width = \textwidth]{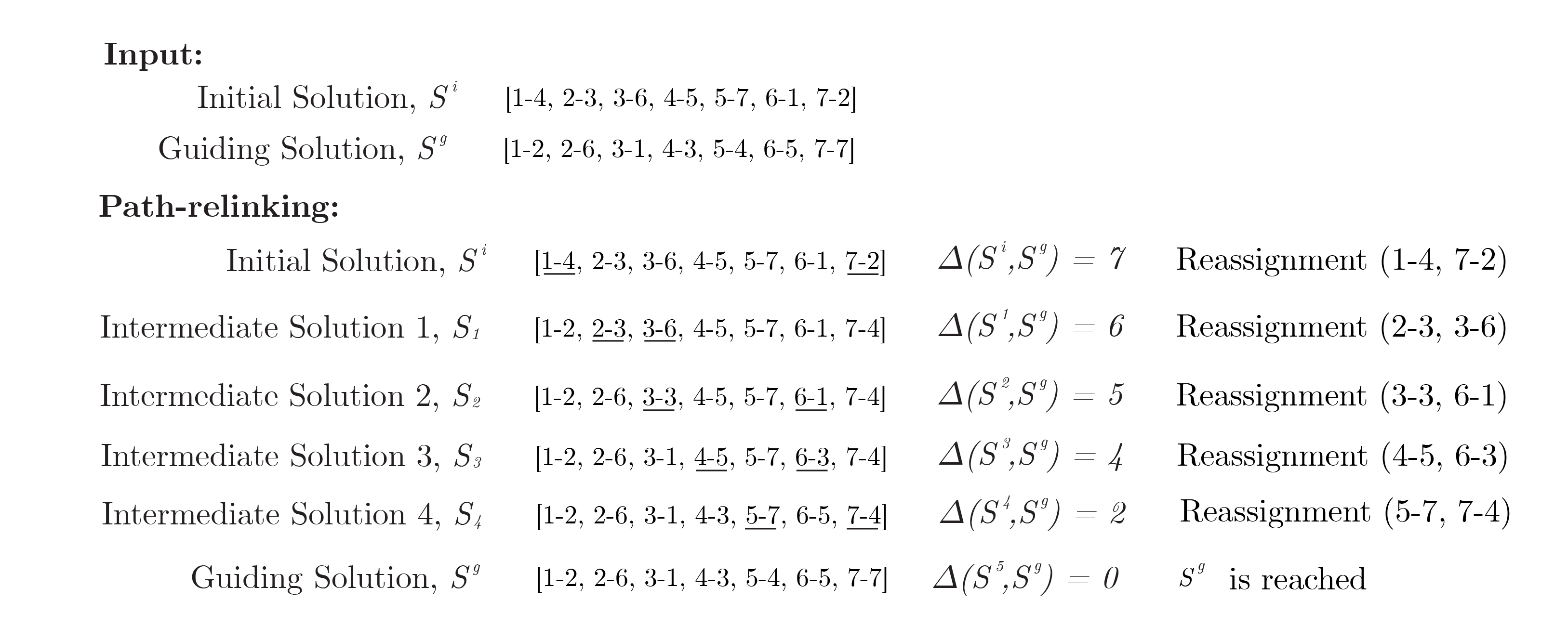}
    \caption{Path relinking between two assignments of seven satellites. A pair $i$--$j$ denotes $\sigma(i)=j$; underlined pairs are exchanged in the next step.}
    \label{fig:pr}
\end{figure}

\paragraph{Elite pool}
The elite pool $\mathcal{E}$ stores at most $\omega$ distinct assignments; the truck routes are not stored. When an elite assignment is used, it is evaluated with the current routes $T$. This allows good assignments found with earlier routes to be recombined with the current routes, at the price that an elite assignment may violate Condition~\eqref{eq:F1} for the current routes; Algorithm~\ref{alg:pr} accounts for this by returning only assignments that satisfy the condition. The pool is initialized with the assignment of the first solution obtained after the initial local search. Let $S^{*}$ denote the best solution found so far. The current assignment $\sigma$ enters the pool if $\sigma\notin\mathcal{E}$ and
\begin{equation}
\label{eq:entry}
f(T,\sigma)<(1+\beta)\,f(S^{*}),
\end{equation}
where $\beta\geq 0$ determines how much worse than the best solution an elite solution may be, and thereby balances the quality and the diversity of the pool. The result of each relinking call enters the pool if it is not already contained in it. If the pool is full, the new assignment replaces the member that has been in the pool longest.

\subsection{The GRASP-ILS-PR algorithm}
\label{sec4:hybrid}
Algorithm~\ref{alg:graspils} states the complete GRASP-ILS-PR algorithm, and Table~\ref{tab:params} lists its parameters, whose values are reported in Section~\ref{Sec5}. Before the search starts, the algorithm verifies the instance conditions $n\geq K$ and $\max_{j\in M}d_j\leq UP$ (Section~\ref{sec4:repr}); if either condition is violated, the instance is infeasible (Lines~\ref{ln:chk1}--\ref{ln:chk2}). An initial feasible solution is obtained by Algorithm~\ref{alg:feasible} and improved by the local search of Algorithm~\ref{alg:ls}. This solution becomes the current solution $S$ and the best solution $S^{*}$, and its assignment initializes the elite pool (Lines~\ref{ln:init1}--\ref{ln:init2}).

Each iteration starts with a perturbation step (Lines~\ref{ln:pert1}--\ref{ln:pert2}). An operator is selected uniformly at random from N1--N7, and one candidate $S'$ is sampled from its neighborhood. The candidate replaces the current solution if it is feasible and $f(S')\leq f(S)$. Accepting candidates of equal objective value allows the search to move across plateaus of the objective function. Every $\varphi_{\mathrm{LS}}$ iterations, the current solution is improved by the local search (Line~\ref{ln:ls}). Every $\varphi_{\mathrm{PR}}$ iterations, an elite assignment $\sigma^{e}$ is selected uniformly at random from $\mathcal{E}$, and forward and backward path relinking are performed between $\sigma$ and $\sigma^{e}$ with the current routes $T$ (Lines~\ref{ln:pr1}--\ref{ln:pr2}). The better result $\sigma^{p}$ is offered to the elite pool, and it replaces the current assignment if $(T,\sigma^{p})$ is feasible and strictly better than $S$. The current assignment then enters the elite pool if it satisfies the entry rule~\eqref{eq:entry} (Lines~\ref{ln:pool1}--\ref{ln:pool2}).

At the end of each iteration, the best solution is updated (Lines~\ref{ln:inc1}--\ref{ln:inc2}). If $f(S)<f(S^{*})$, the current solution becomes the best solution and the stagnation counter $s$ is reset to zero; otherwise, $s$ is increased by one. If the best solution has not improved for $\rho$ consecutive iterations, the search is restarted (Lines~\ref{ln:rs1}--\ref{ln:rs2}). A new feasible solution is constructed by Algorithm~\ref{alg:feasible}, improved by Algorithm~\ref{alg:ls}, and adopted as the current solution. The elite pool and the best solution are retained across restarts, so that path relinking can combine assignments found before and after a restart. The search terminates when the time limit $t_{\max}$ is reached or $n_{\max}$ iterations have been performed, and returns $S^{*}$.

\begin{algorithm}[htbp]
\scriptsize
\caption{GRASP-ILS-PR}
\label{alg:graspils}
\begin{algorithmic}[1]
\Require parameters $t_{\max},n_{\max},\kappa,\varphi_{\mathrm{LS}},\varphi_{\mathrm{PR}},\rho,\omega,\beta,R_{\max}$
\If{$n<K$ \textbf{or} $\max_{j\in M}d_j>UP$} \label{ln:chk1}
    \State \Return \textsc{Infeasible}
\EndIf \label{ln:chk2}
\State $S\gets$ \Call{ConstructFeasible}{$R_{\max}$} \Comment{Algorithm~\ref{alg:feasible}} \label{ln:init1}
\State $S\gets$ \Call{LocalSearch}{$S,\kappa$} \Comment{Algorithm~\ref{alg:ls}}
\State $S^{*}\gets S$; \quad $\mathcal{E}\gets\{\sigma\}$; \quad $s\gets 0$; \quad $i\gets 0$ \label{ln:init2}
\While{elapsed time $<t_{\max}$ \textbf{and} $i<n_{\max}$}
    \State $i\gets i+1$
    \State select $N\in\{\mathrm{N1},\dots,\mathrm{N7}\}$ uniformly at random and sample one candidate $S'$ from $N(S)$ \label{ln:pert1}
    \If{$S'$ exists \textbf{and} $S'$ is feasible \textbf{and} $f(S')\leq f(S)$}
        \State $S\gets S'$
    \EndIf \label{ln:pert2}
    \If{$i \bmod \varphi_{\mathrm{LS}}=0$}
        \State $S\gets$ \Call{LocalSearch}{$S,\kappa$} \label{ln:ls}
    \EndIf
    \If{$i \bmod \varphi_{\mathrm{PR}}=0$} \label{ln:pr1}
        \State select $\sigma^{e}\in\mathcal{E}$ uniformly at random
        \State $\sigma^{f}\gets$ \Call{PathRelinking}{$T,\sigma,\sigma^{e}$} \Comment{Algorithm~\ref{alg:pr}}
        \State $\sigma^{b}\gets$ \Call{PathRelinking}{$T,\sigma^{e},\sigma$}
        \State $\sigma^{p}\gets\operatorname*{arg\,min}_{\sigma'\in\{\sigma^{f},\sigma^{b}\}}f(T,\sigma')$
        \If{$\sigma^{p}\notin\mathcal{E}$}
            \State \Call{UpdatePool}{$\mathcal{E},\sigma^{p}$}
        \EndIf
        \If{$(T,\sigma^{p})$ satisfies Condition~\eqref{eq:F1} \textbf{and} $f(T,\sigma^{p})<f(S)$}
            \State $\sigma\gets\sigma^{p}$
        \EndIf \label{ln:pr2}
    \EndIf
    \If{$\sigma\notin\mathcal{E}$ \textbf{and} $f(S)<(1+\beta)f(S^{*})$} \label{ln:pool1}
        \State \Call{UpdatePool}{$\mathcal{E},\sigma$}
    \EndIf \label{ln:pool2}
    \If{$f(S)<f(S^{*})$} \label{ln:inc1}
        \State $S^{*}\gets S$; \quad $s\gets 0$
    \Else
        \State $s\gets s+1$
    \EndIf \label{ln:inc2}
    \If{$s\geq\rho$} \label{ln:rs1}
        \State $S\gets$ \Call{ConstructFeasible}{$R_{\max}$}
        \State $S\gets$ \Call{LocalSearch}{$S,\kappa$}; \quad $s\gets 0$
    \EndIf \label{ln:rs2}
\EndWhile
\State \Return $S^{*}$
\end{algorithmic}
\end{algorithm}

\begin{table}[htbp]
\centering
\scriptsize
\caption{Parameters of GRASP-ILS-PR}
\label{tab:params}
\begin{tabular}{ll}
\toprule
Symbol & Meaning \\
\midrule
$t_{\max}$ & Time limit (primary stopping criterion) \\
$n_{\max}$ & Maximum number of iterations \\
$\kappa$ & Number of candidates sampled per operator in a local search sweep \\
$\varphi_{\mathrm{LS}}$ & Number of iterations between two calls of the local search \\
$\varphi_{\mathrm{PR}}$ & Number of iterations between two calls of path relinking \\
$\rho$ & Number of iterations without improvement of $S^{*}$ before a restart \\
$\omega$ & Maximum size of the elite pool \\
$\beta$ & Elite pool entry threshold, Equation~\eqref{eq:entry} \\
$R_{\max}$ & Maximum number of construction attempts, Algorithm~\ref{alg:feasible} \\
\bottomrule
\end{tabular}
\end{table}

\paragraph{Feasibility}
The current solution and the best solution remain feasible throughout the search. The initial solution and every restart solution are feasible by construction of Algorithm~\ref{alg:feasible}. The local search and the perturbation step accept only feasible candidates, and a relinking result replaces the current assignment only if it satisfies Condition~\eqref{eq:F1}. Condition~\eqref{eq:F2} holds for every assignment once the instance conditions have been verified. Since the best solution is only replaced by the current solution, it is feasible as well.

\paragraph{Search dynamics and computational effort}
Between two restarts, the objective value of the current solution never increases: the perturbation step accepts only non-worsening candidates, and the local search and path relinking accept only improving ones. Diversification is provided by three mechanisms: the acceptance of equal-valued candidates, which lets the search traverse plateaus; path relinking, which combines the current routes with assignments found earlier in the search; and restarts from new randomized constructions. The computational effort is dominated by the periodic components. A perturbation step requires $O(n)$ time, a local search sweep $O(\kappa n)$ time, and a path relinking call $O(n^{2})$ time (Sections~\ref{sec4:ls} and~\ref{sec4:pr}). Because the local search and path relinking are called only every $\varphi_{\mathrm{LS}}$ and $\varphi_{\mathrm{PR}}$ iterations, their cost is spread over many inexpensive perturbation steps. By contrast, a conventional GRASP with path relinking performs a complete construction, a complete local search, and two relinking paths in every iteration. The resulting difference in the number of iterations completed within a given time limit is analyzed in Section~\ref{Sec5}.

\section{Computational experiments}
\label{Sec5}

\subsection{Benchmark instances}
\label{sec5:instances}
Since no benchmark instances exist for the 2E-CTVRP, we generate a new set of instances that reflects the spatial characteristics of victim locations after a disaster. The generator follows the approaches used for two-echelon vehicle routing problems \citep{crainic2010two, christofides1969algorithm, vu2022two} and for truck--drone routing \citep{murray2015flying}. Victims tend to move to safer locations and form clusters within the affected area; these clusters are generated following \citet{perboli2011two}.

The affected area is a circle with a radius of 1000 distance units, corresponding to 10\,km. The centers of $n_c$ victim clusters are placed uniformly at random within a concentric circle with a radius of 850 units. The radius of each cluster is selected uniformly from $\{110,120,130,140,150\}$, and a cluster with one of these radii contains 80, 90, 100, 120, or 130 victims, respectively. The victims are located uniformly at random within their cluster, and the demand of each victim is drawn uniformly from $[0,25]$. The depot is located to the right of the affected area. Its $x$-coordinate lies between 2000 and 3000 units from the boundary of the affected area, and its $y$-coordinate deviates by at most 500 units from that of the center of the affected area. The $n$ satellites are divided into two groups of $n_1$ and $n_2$ satellites, located at random on the arcs of two circles concentric with the affected area, with radii of 1000 and 1200 units, on the side facing the depot. Figure~\ref{fig:instance} illustrates the resulting spatial distribution.

\begin{figure}[htbp]
\centering
    \includegraphics[width = 0.7\textwidth]{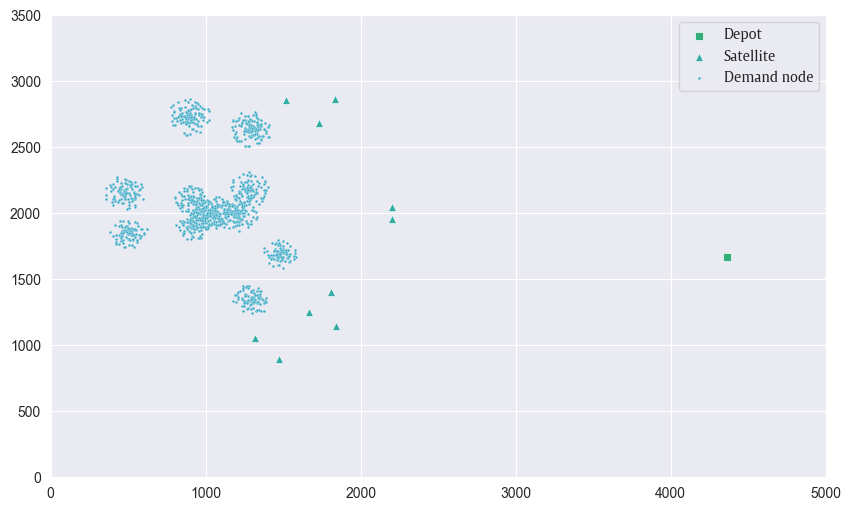}
    \caption{Spatial distribution of the depot, the satellites, and the victims in a benchmark instance}
    \label{fig:instance}
\end{figure}

The victims are aggregated into $n$ clusters, one per satellite, by the $k$-means algorithm with a fixed random seed. The centroid $\mu_j$ of each cluster is the drone delivery point, and its demand $d_j$ is the total demand of the victims in the cluster. When the number $n_c$ of generated clusters differs from the number $n$ of satellites, the clusters used in the optimization differ from the generated ones. The resulting centroids and demands are part of the published instances, so that the results can be reproduced without repeating the clustering.

Truck travel times are computed from Manhattan distances, reflecting movement along the road network, with a truck speed of 40\,km/h. Drone flight times are computed from Euclidean distances, and five drone speeds $v_d\in\{40,50,60,70,80\}$\,km/h are considered. All travel times, and hence all objective values, are expressed in hours. Each drone has a payload capacity of $P=1000$ units. The number of drones $U$ at a satellite is set according to the average cluster demand $D/n$, where $D=\sum_{j\in M}d_j$ is the total demand. Two truck fleets are considered for each instance: a fleet of large trucks with capacity $Q=12\,500$ and a fleet of small trucks with capacity $Q=3500$. The number of trucks $K$ is set according to the total demand. Across all instances, the ratio $(D/n)/(UP)$ between the average cluster demand and the aggregate drone payload ranges from 0.27 to 0.45, and the ratio $D/(KQ)$ between the total demand and the capacity of the fleet ranges from 0.38 to 0.82. In all instances, the largest cluster demand does not exceed $UP$, so that Condition~\eqref{eq:F2} is satisfied, and the cluster demands can be packed into the $K$ trucks, so that all instances are feasible.

Each combination of $n$, fleet type, and drone speed defines one instance. An instance is named M-$n_c$-$n$-$v_d$-$f$, where $f=1$ denotes the large-truck fleet and $f=2$ the small-truck fleet. For example, instance M-15-13-60-2 contains 15 generated victim clusters, 13 satellites, drones flying at 60\,km/h, and small trucks. The instances of a set M-$n_c$-$n$-$f$ differ only in the drone speed. Two instance sets are generated.
\begin{itemize}
\item The \textit{main set} contains 90 instances with nine values of $n$ between 3 and 15. It is used for the comparison with the MILP (Section~\ref{sec5:comparison}), since it includes instances that Gurobi solves to optimality as well as instances for which it reaches the time limit. Table~\ref{table1} summarizes its characteristics.
\item The \textit{large-scale set} contains 20 instances with $n=20$ and $n=50$ satellites, in which the number of generated clusters equals the number of satellites and the satellites are divided equally between the two groups. It is used to assess the scalability of the metaheuristics beyond the size that the MILP can handle (Section~\ref{sec5:large}). Table~\ref{table_large} summarizes its characteristics.
\end{itemize}
Since the ratios $(D/n)/(UP)$ and $D/(KQ)$ lie within the same ranges in both sets, the large-scale instances differ from the main set essentially in their size. All instances are available at \url{https://github.com/andngdtudk/2E_CTVRP}.

\begin{table}[htbp]
\centering
\scriptsize
\caption{Characteristics of the main instance sets: number of generated victim clusters ($n_c$), number of satellites ($n$) in the two groups ($n_1$, $n_2$), number of victims ($n_v$), total demand ($D$), truck capacity ($Q$), number of trucks ($K$), and number of drones per satellite ($U$). The drone payload capacity is $P=1000$ in all instances.}
\label{table1}
\begin{tabular}{lrrrrrrrrr}
\toprule
Set & $n_c$ & $n$ & $n_1$ & $n_2$ & $n_v$ & $D$ & $Q$ & $K$ & $U$ \\
\midrule
M-05-03-1 & 5  & 3  & 1 & 2 & 460  & 5650    & 12\,500 & 1 & 6 \\
M-05-03-2 & 5  & 3  & 1 & 2 & 460  & 5650    & 3500    & 2 & 6 \\
M-05-04-1 & 5  & 4  & 2 & 2 & 520  & 6403    & 12\,500 & 1 & 4 \\
M-05-04-2 & 5  & 4  & 2 & 2 & 520  & 6403    & 3500    & 3 & 4 \\
M-08-05-1 & 8  & 5  & 2 & 3 & 730  & 8972    & 12\,500 & 1 & 4 \\
M-08-05-2 & 8  & 5  & 2 & 3 & 730  & 8972    & 3500    & 4 & 4 \\
M-10-07-1 & 10 & 7  & 4 & 3 & 990  & 12\,146 & 12\,500 & 2 & 6 \\
M-10-07-2 & 10 & 7  & 4 & 3 & 990  & 12\,146 & 3500    & 5 & 6 \\
M-10-08-1 & 10 & 8  & 4 & 4 & 940  & 11\,576 & 12\,500 & 2 & 5 \\
M-10-08-2 & 10 & 8  & 4 & 4 & 940  & 11\,576 & 3500    & 5 & 4 \\
M-09-09-1 & 9  & 9  & 4 & 5 & 780  & 9587    & 12\,500 & 2 & 3 \\
M-09-09-2 & 9  & 9  & 4 & 5 & 780  & 9587    & 3500    & 4 & 3 \\
M-12-11-1 & 12 & 11 & 6 & 5 & 1240 & 15\,238 & 12\,500 & 2 & 5 \\
M-12-11-2 & 12 & 11 & 6 & 5 & 1240 & 15\,238 & 3500    & 7 & 5 \\
M-15-13-1 & 15 & 13 & 6 & 7 & 1420 & 17\,273 & 12\,500 & 3 & 5 \\
M-15-13-2 & 15 & 13 & 6 & 7 & 1420 & 17\,273 & 3500    & 7 & 5 \\
M-17-15-1 & 17 & 15 & 7 & 8 & 1650 & 20\,000 & 12\,500 & 3 & 4 \\
M-17-15-2 & 17 & 15 & 7 & 8 & 1650 & 20\,000 & 3500    & 7 & 4 \\
\bottomrule
\end{tabular}
\end{table}

\begin{table}[htbp]
\centering
\scriptsize
\caption{Characteristics of the large-scale instance sets: number of generated victim clusters ($n_c$), number of satellites ($n$) in the two groups ($n_1$, $n_2$), number of victims ($n_v$), total demand ($D$), truck capacity ($Q$), number of trucks ($K$), and number of drones per satellite ($U$). The drone payload capacity is $P=1000$ in all instances.}
\label{table_large}
\begin{tabular}{lrrrrrrrrr}
\toprule
Set & $n_c$ & $n$ & $n_1$ & $n_2$ & $n_v$ & $D$ & $Q$ & $K$ & $U$ \\
\midrule
M-20-20-1 & 20 & 20 & 10 & 10 & 2050 & 24\,065 & 12\,500 & 3  & 4 \\
M-20-20-2 & 20 & 20 & 10 & 10 & 2050 & 24\,065 & 3500    & 11 & 4 \\
M-50-50-1 & 50 & 50 & 25 & 25 & 4830 & 58\,298 & 12\,500 & 6  & 3 \\
M-50-50-2 & 50 & 50 & 25 & 25 & 4830 & 58\,298 & 3500    & 26 & 3 \\
\bottomrule
\end{tabular}
\end{table}

\subsection{Experimental setup and compared methods}
\label{sec5:setup}
All algorithms were implemented in Python~3.10.12, and the MILP formulation of Section~\ref{Sec3} was solved with Gurobi~13.0.2. The experiments were performed on an AMD Ryzen Threadripper 3960X (24 cores, 48 threads) with 125\,GB RAM, running Ubuntu 22.04.3 LTS. The metaheuristic runs were executed in parallel, 14 at a time, with one thread per run. On the main instance set, Gurobi was run with the default number of threads (up to 48), solving one instance at a time; the settings for the large-scale instances are given in Section~\ref{sec5:large}.

\paragraph{Compared methods}
Four methods are compared on the 90 instances of the main set described in Section~\ref{sec5:instances}.

\begin{itemize}
    \item \textit{MILP.} Formulation~\eqref{eq1}--\eqref{eq18c} is solved by Gurobi with a time limit of 3600\,s and a relative MIP gap tolerance of zero. For each instance, we report the objective value $z^{\mathrm{G}}$ of the best solution found (the incumbent), the best lower bound $z^{\mathrm{LB}}$, the MIP gap $(z^{\mathrm{G}}-z^{\mathrm{LB}})/z^{\mathrm{G}}$ at termination, and the runtime.

    \item \textit{GRASP-EvPR.} A conventional GRASP with evolutionary path relinking \citep{resende2010grasp}. It uses the construction heuristic, repair operator, local search, and path relinking procedure of Sections~\ref{sec4:constr}--\ref{sec4:pr}, but combines them in the standard way: every iteration constructs a new solution, improves it by local search, and relinks it in both directions with an elite solution selected at random from the pool, and every $i_{\mathrm{ev}}=100$ iterations the elite pool is evolved by relinking all pairs of its members. The pool size and the entry threshold, $\omega=20$ and $\beta=0.008$, were set in preliminary experiments. Since GRASP-EvPR differs from GRASP-ILS-PR only in the organization of the search, the comparison of the two methods measures the effect of carrying a single solution across iterations and invoking the expensive components periodically (Section~\ref{sec4:prelim}).

    \item \textit{Simulated annealing (SA).} As an alternative single-trajectory strategy, we implemented a simulated annealing algorithm \citep{kirkpatrick1983optimization} built from the same components as GRASP-ILS-PR: the construction heuristic and repair operator of Sections~\ref{sec4:constr} and~\ref{sec4:repair}, the seven neighborhood operators of Section~\ref{sec4:ls}, and the evaluation of Section~\ref{sec4:repr}. In each iteration, an operator is selected uniformly at random and one candidate is sampled from its neighborhood. A candidate that worsens the objective value by $\delta>0$ is accepted with probability $\exp(-\delta/\tau)$, where $\tau$ is the current temperature, and the best solution found is returned at the end of the run. The temperature decreases geometrically with the elapsed time $t$,
    \begin{equation*}
    \tau(t)=\tau_0\left(\frac{\tau_{\min}}{\tau_0}\right)^{t/t_{\max}},
    \end{equation*}
    from $\tau_0=0.10\,z_0$ at the start of the run to $\tau_{\min}=10^{-4}z_0$ at the time limit, where $z_0$ is the objective value of the initial solution. Expressing $\tau_0$ and $\tau_{\min}$ relative to $z_0$ makes the acceptance behavior comparable across instances with different objective scales. Linking the schedule to the elapsed time rather than to the number of iterations ensures that the whole temperature range is traversed within $t_{\max}$ on every instance, although the number of iterations completed by SA varies by a factor of 16 across instances (Table~\ref{tab:iterations}). SA thus differs from GRASP-ILS-PR in its acceptance criterion and in the absence of periodic local search, path relinking, and restarts.

   \item \textit{GRASP-ILS-PR.} The algorithm of Section~\ref{Sec4:GRA}, with the parameter values of Table~\ref{tab:paramvalues}.
    
\end{itemize}

\paragraph{Experimental protocol}
Each metaheuristic is run ten times on each instance with different random seeds and a time limit of $t_{\max}=30$\,s per run. The iteration limits are set high enough never to be reached, so that the time limit is the binding stopping criterion for all three metaheuristics. A time limit of 30\,s is compatible with the repeated re-planning required in the response phase of a disaster and allows the metaheuristics to be compared on equal terms. All methods solve identical instances, including the same clustering of the victims. The parameter values of GRASP-ILS-PR, reported in Table~\ref{tab:paramvalues}, were set in preliminary experiments; their influence is analyzed in Section~\ref{sec5:sensitivity}.

\begin{table}[htbp]
\centering
\scriptsize
\caption{Parameter values of GRASP-ILS-PR.}
\label{tab:paramvalues}
\begin{tabular}{llr}
\toprule
Symbol & Meaning & Value \\
\midrule
$t_{\max}$ & time limit & 30\,s \\
$n_{\max}$ & maximum number of iterations & $10^{8}$ \\
$\kappa$ & candidates sampled per operator in a local search sweep & 5 \\
$\varphi_{\mathrm{LS}}$ & iterations between two calls of the local search & 30 \\
$\varphi_{\mathrm{PR}}$ & iterations between two calls of path relinking & 10 \\
$\rho$ & iterations without improvement before a restart & 2000 \\
$\omega$ & maximum size of the elite pool & 20 \\
$\beta$ & elite pool entry threshold & 0.007 \\
$R_{\max}$ & maximum number of construction attempts & 20 \\
\bottomrule
\end{tabular}
\end{table}

\paragraph{Performance measures}
For a metaheuristic, let $z^{\mathrm{best}}$ and $z^{\mathrm{avg}}$ denote the best and the average objective values over the ten runs on an instance. Solution quality is measured by the relative deviations from the Gurobi incumbent,
\begin{equation}
\label{eq:dev}
\Delta_{\mathrm{best}}=100\,\frac{z^{\mathrm{best}}-z^{\mathrm{G}}}{z^{\mathrm{G}}},\qquad
\Delta_{\mathrm{avg}}=100\,\frac{z^{\mathrm{avg}}-z^{\mathrm{G}}}{z^{\mathrm{G}}}.
\end{equation}
If Gurobi proves optimality, $\Delta_{\mathrm{best}}$ and $\Delta_{\mathrm{avg}}$ are optimality gaps; otherwise, they measure the deviation from the best solution found by Gurobi within one hour, and negative values indicate that the metaheuristic found a better solution. The best-known solution (BKS) of an instance is the best solution found by any of the four methods. Two objective values are considered equal if their relative difference does not exceed $10^{-5}$; this tolerance absorbs the deviations of the Gurobi incumbents caused by the feasibility and integrality tolerances of the solver, which reach about $2\times10^{-7}$ in our experiments, and treats solutions whose objective values differ by less than 0.001\% as equivalent. The differences between GRASP-ILS-PR and each of the other two metaheuristics are tested with the Wilcoxon signed-rank test \citep{wilcoxon1945individual} on the paired per-instance values of $z^{\mathrm{best}}$ and of $z^{\mathrm{avg}}$, discarding pairs with zero difference, at a significance level of 0.05.

To exclude errors in the solution evaluation of Section~\ref{sec4:repr}, the best solutions of GRASP-ILS-PR and SA on all 110 instances were verified independently. For each solution, the truck routes and the assignment of centroids to satellites were fixed in the MILP of Section~\ref{Sec3}, and the resulting model was solved by Gurobi. In all cases, the objective value obtained by Gurobi agreed with the value computed by the heuristic to within a relative difference of $10^{-12}$.

\subsection{Comparison with the MILP and alternative metaheuristics}
\label{sec5:comparison}

Table~\ref{tab:summary} summarizes the results by instance set, and Tables~\ref{tab:appA1} and~\ref{tab:appA2} in ~\ref{app:results} report the results for each instance with the large-truck and the small-truck fleet, respectively. We first discuss the performance of the MILP, then the performance of the two alternative metaheuristics, GRASP-EvPR and SA, and finally the performance of the proposed GRASP-ILS-PR relative to all three methods.

\begin{table}[htbp]
\centering
\scriptsize
\setlength{\tabcolsep}{3.5pt}
\caption{Summary comparison of Gurobi, GRASP-EvPR, SA, and the proposed GRASP-ILS-PR by instance set; the proposed method is shown with a bold header. Gurobi: number of instances solved to proven optimality (\#Opt), average runtime, average MIP gap at termination, and number of best-known solutions found (\#BKS). Metaheuristics (ten runs of 30\,s per instance): average relative deviation of the best ($\Delta_{\mathrm{best}}$) and average ($\Delta_{\mathrm{avg}}$) objective values from the Gurobi incumbent, in \%, and number of best-known solutions found (\#BKS). Negative deviations indicate improvements over the Gurobi incumbent. Each set contains five instances.}
\label{tab:summary}
\begin{tabular}{lrrrrrrrrrrrrr}
\toprule
 & \multicolumn{4}{c}{Gurobi} & \multicolumn{3}{c}{GRASP-EvPR} & \multicolumn{3}{c}{SA} & \multicolumn{3}{c}{\textbf{GRASP-ILS-PR}} \\
\cmidrule(lr){2-5}\cmidrule(lr){6-8}\cmidrule(lr){9-11}\cmidrule(lr){12-14}
Set & \#Opt & CPU (s) & Gap (\%) & \#BKS & $\Delta_{\mathrm{best}}$ & $\Delta_{\mathrm{avg}}$ & \#BKS & $\Delta_{\mathrm{best}}$ & $\Delta_{\mathrm{avg}}$ & \#BKS & $\Delta_{\mathrm{best}}$ & $\Delta_{\mathrm{avg}}$ & \#BKS \\
\midrule
M-05-03-1 & 5 & 0.01 & 0.00 & 5 & 0.000 & 0.000 & 5 & 0.000 & 0.000 & 5 & 0.000 & 0.000 & 5 \\
M-05-03-2 & 5 & 0.02 & 0.00 & 5 & 0.000 & 0.000 & 5 & 0.000 & 0.479 & 5 & 0.000 & 0.000 & 5 \\
M-05-04-1 & 5 & 0.02 & 0.00 & 5 & 0.000 & 0.000 & 5 & 0.000 & 0.000 & 5 & 0.000 & 0.000 & 5 \\
M-05-04-2 & 5 & 0.23 & 0.00 & 5 & 0.000 & 0.000 & 5 & 0.000 & 0.000 & 5 & 0.000 & 0.000 & 5 \\
M-08-05-1 & 5 & 0.08 & 0.00 & 5 & 0.000 & 0.000 & 5 & 0.000 & 0.000 & 5 & 0.000 & 0.000 & 5 \\
M-08-05-2 & 5 & 0.48 & 0.00 & 5 & 0.000 & 0.000 & 5 & 0.000 & 0.000 & 5 & 0.000 & 0.000 & 5 \\
M-10-07-1 & 5 & 4.67 & 0.00 & 5 & 0.000 & 0.000 & 5 & 0.000 & 0.000 & 5 & 0.000 & 0.000 & 5 \\
M-10-07-2 & 5 & 7.53 & 0.00 & 5 & 0.000 & 0.000 & 5 & 0.000 & 0.000 & 5 & 0.000 & 0.000 & 5 \\
M-10-08-1 & 5 & 36.27 & 0.00 & 5 & 0.000 & 0.000 & 5 & 0.000 & 0.000 & 5 & 0.000 & 0.002 & 5 \\
M-10-08-2 & 5 & 279.10 & 0.00 & 5 & 0.000 & 0.000 & 5 & 0.000 & 0.006 & 5 & 0.000 & 0.000 & 5 \\
M-09-09-1 & 5 & 161.38 & 0.00 & 5 & 0.000 & 0.038 & 5 & 0.000 & 0.000 & 5 & 0.000 & 0.000 & 5 \\
M-09-09-2 & 5 & 1527.18 & 0.00 & 5 & 0.000 & 0.080 & 5 & 0.000 & 0.019 & 5 & 0.000 & 0.041 & 5 \\
M-12-11-1 & 0 & 3600.02 & 40.70 & 5 & 0.094 & 0.514 & 2 & 0.013 & 0.184 & 4 & 0.000 & 0.010 & 5 \\
M-12-11-2 & 2 & 2630.81 & 11.77 & 5 & 0.060 & 0.294 & 0 & 0.051 & 0.256 & 4 & 0.000 & 0.058 & 5 \\
M-15-13-1 & 0 & 3600.01 & 55.85 & 4 & 0.811 & 1.568 & 0 & -0.190 & 0.020 & 5 & -0.190 & -0.159 & 5 \\
M-15-13-2 & 0 & 3600.08 & 29.93 & 2 & 0.810 & 1.178 & 0 & -0.062 & 0.191 & 3 & 0.017 & 0.211 & 1 \\
M-17-15-1 & 0 & 3600.02 & 59.89 & 4 & 1.963 & 2.662 & 0 & 0.305 & 0.680 & 0 & -0.011 & 0.277 & 2 \\
M-17-15-2 & 0 & 3600.02 & 34.41 & 2 & 1.106 & 1.570 & 0 & 0.017 & 0.380 & 1 & -0.034 & 0.257 & 4 \\
\midrule
All & 62 & 1258.22 & 12.92 & 82 & 0.269 & 0.439 & 62 & 0.007 & 0.123 & 77 & -0.012 & 0.039 & 82 \\
\bottomrule
\end{tabular}
\end{table}

\paragraph{Performance of the MILP}
Gurobi solves all 60 instances with up to nine satellites to proven optimality, but its runtime increases rapidly with the number of satellites: from an average of 0.01\,s for $n=3$ to 158\,s for $n=8$ and 844\,s for $n=9$. For $n=11$, only two of ten instances are solved within the time limit of 3600\,s, and no instance with $n\geq 13$ is solved to optimality. On the 28 instances for which the time limit is reached, the MIP gap at termination ranges from 15.5\% to 64.1\%, with an average of 41.5\%, and it is larger for the large-truck fleet (52.1\% on average) than for the small-truck fleet (29.3\%). Since the lower bounds remain weak, a phenomenon typical of formulations with big-M time-propagation constraints, the quality of the solutions for these instances can only be assessed relative to the incumbents found by Gurobi, not relative to the optimal values.

\paragraph{Performance of the alternative metaheuristics}
On the instances with up to nine satellites, both alternative metaheuristics find the optimal solution in their best run on all or almost all instances, but their behavior differs as the instance size increases. GRASP-EvPR finds the optimal solution in its best run on 60 of the 62 instances solved to optimality, but its performance deteriorates on the larger instances. On the 28 instances for which Gurobi reaches the time limit, its best solution is worse than the Gurobi incumbent on 25 instances, by up to 2.76\%, and its average deviation in the best run increases from 0.08\% for $n=11$ to 0.81\% for $n=13$ and 1.54\% for $n=15$.

SA is competitive in its best run. It finds the optimal solution on all 62 instances solved to optimality, and on the 28 instances for which Gurobi reaches the time limit, its best solution is better than the Gurobi incumbent on 6 instances, equal on 13, and worse on 9, with deviations between $-0.95\%$ and 0.71\%. Its results vary more between runs, however. Its average objective value equals the optimal value on 46 of the 62 instances solved to optimality, and its average run deviates from the Gurobi incumbent by 0.123\% over all 90 instances. The variation is not limited to the larger instances; on set M-05-03-2, with three satellites, the average run of SA deviates from the optimal value by 0.48\%.

\paragraph{Performance of GRASP-ILS-PR}
GRASP-ILS-PR finds the optimal solution in its best run on all 62 instances solved to optimality by Gurobi. On 53 of these instances, all ten runs return the optimal solution, and the average objective value exceeds the optimal value by at most 0.066\%. On the 28 instances for which Gurobi reaches the time limit, the best solution of GRASP-ILS-PR is better than the Gurobi incumbent on 7 instances, equal on 14, and worse on 7. The improvements reach 0.95\% (instance M-15-13-60-1), whereas the largest deterioration is 0.28\% (instance M-17-15-50-1). Over all 90 instances, the average deviation from the Gurobi incumbent is $-0.012\%$ for the best run and 0.039\% for the average run, and GRASP-ILS-PR finds the best-known solution on 82 instances, the same number as Gurobi, compared with 77 for SA and 62 for GRASP-EvPR. These results are obtained with a single thread and a time limit of 30\,s per run, whereas Gurobi uses up to 48 threads and up to 3600\,s.

Compared with GRASP-EvPR, GRASP-ILS-PR maintains its solution quality as the instance size increases: its average deviations in the best run for $n=11$, 13, and 15 are 0.00\%, $-0.09\%$, and $-0.02\%$. It is better than GRASP-EvPR on 28 instances in the best run and equal on the remaining 62; in the average run, it is better on 38 instances, equal on 50, and worse on 2. Since the two methods share all search components, this difference is attributable to the organization of the search, which is analyzed in Section~\ref{sec5:design}.

The comparison with SA is closer. In the best run, GRASP-ILS-PR is better than SA on 13 instances, equal on 73, and worse on 4. The four instances on which SA finds the better solution (M-15-13-60-2, M-15-13-70-2, M-15-13-80-2, and M-17-15-80-2) all belong to the small-truck fleet, a pattern that recurs on the large-scale instances (Section~\ref{sec5:large}). In the average run, GRASP-ILS-PR is better on 36 instances, equal on 44, and worse on 10, and its average objective value equals the optimal value on 53 of the 62 instances solved to optimality, compared with 46 for SA. The Wilcoxon signed-rank test confirms that GRASP-ILS-PR is better than GRASP-EvPR in both the best and the average run ($p=2.9\times10^{-6}$ and $p=2.4\times10^{-8}$). Against SA, the difference is significant for the average run ($p=1.3\times10^{-5}$) but not for the best run ($p=0.068$).

In summary, GRASP-ILS-PR matches the MILP on all instances that the MILP solves to optimality and remains within 0.28\% of the Gurobi incumbent, or improves on it, for the larger instances. It clearly outperforms GRASP-EvPR. Compared with SA, it reaches a similar solution quality in the best of ten runs but is significantly more consistent across runs, which matters in practice when the time available for re-planning allows only a single run.

\subsection{Analysis of the algorithm design}
\label{sec5:design}
GRASP-ILS-PR and GRASP-EvPR use the same construction heuristic, repair operator, local search, and path relinking procedure, and differ only in how these components are scheduled. In GRASP-EvPR, every iteration performs a complete construction, a complete local search, and two relinking paths, whose costs increase with the number of satellites (Sections~\ref{sec4:constr}--\ref{sec4:pr}). In GRASP-ILS-PR, most iterations consist of a single sampled move, whereas the local search and path relinking are invoked every $\varphi_{\mathrm{LS}}$ and $\varphi_{\mathrm{PR}}$ iterations, respectively. Table~\ref{tab:iterations} reports the number of iterations completed within the time limit of 30\,s on six representative instances that cover the range of instance and fleet sizes.

\begin{table}[htbp]
\centering
\scriptsize
\caption{Average number of iterations completed within 30\,s (ten runs per instance). For GRASP-EvPR, an iteration consists of a construction, a local search, and path relinking; for SA and GRASP-ILS-PR, an iteration consists of one sampled move.}
\label{tab:iterations}
\begin{tabular}{lrrrrr}
\toprule
Instance & $n$ & $K$ & GRASP-EvPR & SA & \textbf{GRASP-ILS-PR} \\
\midrule
M-05-03-60-1 & 3  & 1 & 41\,738 & 756\,524 & 365\,666 \\
M-08-05-60-1 & 5  & 1 & 17\,829 & 415\,163 & 295\,854 \\
M-10-07-60-2 & 7  & 5 & 7155    & 147\,831 & 103\,868 \\
M-12-11-60-1 & 11 & 2 & 2073    & 64\,603  & 100\,399 \\
M-15-13-60-1 & 13 & 3 & 1250    & 54\,408  & 86\,314  \\
M-17-15-60-2 & 15 & 7 & 3607    & 47\,957  & 54\,978  \\
\bottomrule
\end{tabular}
\end{table}

The number of iterations of GRASP-EvPR decreases sharply with the instance size, from 41\,738 for $n=3$ to 1250 for $n=13$, a reduction by a factor of 33. Over the same range, the number of iterations of GRASP-ILS-PR decreases only by a factor of 4, from 365\,666 to 86\,314. As a result, GRASP-ILS-PR completes 9 times as many iterations as GRASP-EvPR on the smallest instance and 69 times as many on instance M-15-13-60-1. The ratio does not increase monotonically with $n$, since the cost of an iteration of GRASP-EvPR also depends on the number of trucks $K$; on instance M-17-15-60-2, with seven trucks, GRASP-EvPR completes 3607 iterations, and the ratio is 15. The higher iteration rate does not come at the expense of the expensive components. Within the same 30\,s, GRASP-ILS-PR performs on average 5595 local search descents and 16\,785 path relinking calls on the six instances. On instances M-12-11-60-1 and M-15-13-60-1, both numbers exceed the total number of iterations of GRASP-EvPR; on instance M-15-13-60-1, for example, GRASP-ILS-PR performs 2877 local search descents and 8631 path relinking calls, compared with 1250 iterations of GRASP-EvPR. The iteration rate of GRASP-ILS-PR is of the same order as that of SA, and higher on the three largest instances, although GRASP-ILS-PR additionally performs periodic local search and path relinking.

The effect of these differences on solution quality corresponds to the results of Section~\ref{sec5:comparison}. On the instances with up to nine satellites, GRASP-EvPR completes enough iterations to find the optimal solution in almost all runs, and its results are practically identical to those of GRASP-ILS-PR. As the number of satellites increases and its number of iterations decreases, its average deviation in the best run increases from 0.08\% for $n=11$ to 1.54\% for $n=15$, whereas the deviation of GRASP-ILS-PR remains at or below zero. Since the two methods apply the same moves and the same relinking procedure to the same instances, these results indicate that, under a fixed time limit, the frequency with which the search components can be applied matters more than applying all of them in every iteration.

Finally, we tested whether adding evolutionary path relinking to GRASP-ILS-PR, i.e., periodically relinking all pairs of elite solutions \citep{resende2010grasp}, improves the results. On the six representative instances, with 20 paired runs of 30\,s each, the variants with and without EvPR returned identical results in 76 of 120 runs, and the difference was not significant (Wilcoxon signed-rank test, $p=0.46$); a repetition with a time limit of 300\,s gave the same conclusion ($p=0.11$). Since path relinking against the elite pool is already invoked every $\varphi_{\mathrm{PR}}$ iterations in GRASP-ILS-PR, the additional recombination provided by EvPR is redundant, and EvPR is therefore not part of the final algorithm.

\subsection{Parameter sensitivity}
\label{sec5:sensitivity}

To assess the sensitivity of GRASP-ILS-PR to its parameters, we varied each parameter individually while keeping the others at the values of Table~\ref{tab:paramvalues}. The experiment uses the six representative instances of Section~\ref{sec5:design}, with ten runs of 30\,s per instance and configuration. All configurations use the same random seeds, so that the results of each configuration can be paired with those of the default configuration, whose runs are those reported in Section~\ref{sec5:comparison}. Table~\ref{tab:sensitivity} reports the average objective value over the 60 runs of each configuration, its relative difference from the default configuration, and the $p$-value of the Wilcoxon signed-rank test on the paired runs.

\begin{table}[htbp]
\centering
\scriptsize
\caption{Sensitivity of GRASP-ILS-PR to its parameters: average objective value over six instances and ten runs per instance, relative difference $\Delta$ from the default configuration, and $p$-value of the Wilcoxon signed-rank test on the 60 paired runs. Default values are shown in bold. Positive values of $\Delta$ indicate a deterioration.}
\label{tab:sensitivity}
\begin{tabular}{llrrr}
\toprule
Parameter & Value & Avg.\ objective & $\Delta$ (\%) & $p$ \\
\midrule
Default configuration & & 19.3887 & -- & -- \\
\midrule
$\kappa$ (samples per operator) & 3 & 19.3851 & $-$0.019 & 0.497 \\
 & \textbf{5} & & & \\
 & 10 & 19.3871 & $-$0.008 & 0.967 \\
\addlinespace
$\varphi_{\mathrm{LS}}$ (local search interval) & 10 & 19.3967 & 0.041 & 0.178 \\
 & \textbf{30} & & & \\
 & 60 & 19.3848 & $-$0.021 & 0.577 \\
 & 100 & 19.3853 & $-$0.018 & 0.497 \\
\addlinespace
$\varphi_{\mathrm{PR}}$ (path relinking interval) & 5 & 19.3897 & 0.005 & 0.952 \\
 & \textbf{10} & & & \\
 & 20 & 19.3881 & $-$0.003 & 0.781 \\
 & 50 & 19.3861 & $-$0.014 & 0.714 \\
\addlinespace
$\rho$ (restart patience) & 500 & 19.3802 & $-$0.044 & 0.050 \\
 & \textbf{2000} & & & \\
 & 5000 & 19.4093 & 0.106 & $<0.001$ \\
\addlinespace
$\omega$ (elite pool size) & 10 & 19.3833 & $-$0.028 & 0.410 \\
 & \textbf{20} & & & \\
 & 40 & 19.3804 & $-$0.043 & 0.397 \\
\addlinespace
$\beta$ (pool entry threshold) & 0.003 & 19.3939 & 0.027 & 0.383 \\
 & \textbf{0.007} & & & \\
 & 0.015 & 19.3893 & 0.003 & 0.487 \\
 & 0.03 & 19.3871 & $-$0.008 & 0.936 \\
\bottomrule
\end{tabular}
\end{table}

The results are robust to the parameter values. The average objective value varies by at most 0.11\% across all configurations, and only one of the 15 tested values differs significantly from the default configuration at the 0.05 level. The sampling parameter $\kappa$, the intervals $\varphi_{\mathrm{LS}}$ and $\varphi_{\mathrm{PR}}$, the pool size $\omega$, and the entry threshold $\beta$ have no significant effect in the tested ranges ($p\geq0.17$), and all corresponding differences are below 0.05\%. The restart patience $\rho$ is the only parameter with a noticeable effect. With $\rho=5000$, the average objective value deteriorates by 0.11\% ($p<0.001$), whereas $\rho=500$ yields a slightly better average ($-0.04\%$) with a $p$-value of 0.050, at the boundary of significance.

The differences are confined to the three instances with 11 or more satellites. On the three smaller instances, all configurations yield the same average objective value, since the search converges well within the time limit. The effect of $\rho$ is consistent with the role of restarts as the main source of diversification (Section~\ref{sec4:hybrid}). Since the search between restarts never accepts a worse solution, a longer restart patience keeps the search longer in the neighborhood of a local optimum and reduces the number of restarts that can be performed within the time limit. The result for $\rho=500$ suggests that more frequent restarts could slightly benefit the largest instances, but the gain is small and not significant, and we therefore retain the default value.

These results indicate that the default values of Table~\ref{tab:paramvalues} lie in a region in which the performance of GRASP-ILS-PR is insensitive to moderate changes of the parameters, so that no instance-specific tuning is required. The analysis varies one parameter at a time on six instances; interactions between parameters were not investigated.

\subsection{Large-scale instances}
\label{sec5:large}

The large-scale instances of Table~\ref{table_large} are used to assess the performance of the methods beyond the size that the MILP can solve. Since these instances are considerably larger than those of the main set, the time limit of the metaheuristics is increased to 120\,s per run; all other settings are as in Section~\ref{sec5:setup}. Gurobi is run with a time limit of 3600\,s and 8 threads per instance, with six instances solved in parallel. Since no optimal solutions are available, all deviations are measured relative to the best-known solution of each instance, that is, the best solution found by any of the four methods. Table~\ref{tab:large} summarizes the results, and Table~\ref{tab:appA3} in \ref{app:results} reports the results for each instance.

\begin{table}[htbp]
\centering
\scriptsize
\setlength{\tabcolsep}{3.5pt}
\caption{Summary results on the large-scale instances. Gurobi (time limit 3600\,s): average MIP gap at termination, average relative deviation of the incumbent from the best-known solution ($\Delta$), and number of best-known solutions found (\#BKS). Metaheuristics (ten runs of 120\,s per instance): average relative deviation of the best ($\Delta_{\mathrm{best}}$) and average ($\Delta_{\mathrm{avg}}$) objective values from the best-known solution, in \%, and \#BKS. Each set contains five instances; the proposed method is shown with a bold header.}
\label{tab:large}
\begin{tabular}{lrrrrrrrrrrrr}
\toprule
 & \multicolumn{3}{c}{Gurobi} & \multicolumn{3}{c}{GRASP-EvPR} & \multicolumn{3}{c}{SA} & \multicolumn{3}{c}{\textbf{GRASP-ILS-PR}} \\
\cmidrule(lr){2-4}\cmidrule(lr){5-7}\cmidrule(lr){8-10}\cmidrule(lr){11-13}
Set & Gap (\%) & $\Delta$ & \#BKS & $\Delta_{\mathrm{best}}$ & $\Delta_{\mathrm{avg}}$ & \#BKS & $\Delta_{\mathrm{best}}$ & $\Delta_{\mathrm{avg}}$ & \#BKS & $\Delta_{\mathrm{best}}$ & $\Delta_{\mathrm{avg}}$ & \#BKS \\
\midrule
M-20-20-1 & 76.10 & 0.06 & 3 & 4.80 & 6.12 & 0 & 0.08 & 0.44 & 1 & 0.01 & 0.13 & 4 \\
M-20-20-2 & 48.56 & 0.15 & 1 & 1.66 & 2.14 & 0 & 0.03 & 0.22 & 4 & 0.03 & 0.19 & 4 \\
M-50-50-1 & 94.48 & 11.52 & 0 & 17.27 & 19.48 & 0 & 1.46 & 3.08 & 0 & 0.00 & 1.48 & 5 \\
M-50-50-2 & 65.14 & 5.10 & 0 & 4.83 & 5.84 & 0 & 0.15 & 0.44 & 1 & 0.03 & 0.51 & 4 \\
\midrule
All & 71.07 & 4.21 & 4 & 7.14 & 8.39 & 0 & 0.43 & 1.05 & 6 & 0.02 & 0.57 & 17 \\
\bottomrule
\end{tabular}
\end{table}

\paragraph{Performance of the MILP}
Gurobi reaches the time limit on all 20 instances, with MIP gaps between 48.1\% and 96.7\%, so that its lower bounds provide no meaningful information on solution quality. For $n=20$, its incumbents remain competitive: they deviate from the best-known solution by 0.06\% and 0.15\% on average for the two fleets, and Gurobi finds the best-known solution on 4 of the 10 instances. For $n=50$, the model becomes too large to be explored effectively within one hour. For the small-truck fleet with $K=26$, Gurobi explores at most 317 nodes of the branch-and-bound tree, and only the root node on three of the five instances. Its incumbents deviate from the best-known solutions by 5.10\% on average for the small-truck fleet and by 11.52\% for the large-truck fleet.

\paragraph{Performance of the metaheuristics}
GRASP-ILS-PR finds the best-known solution on 17 of the 20 instances, and its best solution deviates from the best-known solution by 0.02\% on average. For $n=20$, it is better than the Gurobi incumbent on 5 instances, equal on 3, and worse on 2, by at most 0.03\%, with a time limit of 120\,s instead of 3600\,s. For $n=50$, it improves on the Gurobi incumbent on all 10 instances, by 1.3\% to 11.1\%. On set M-50-50-1, it finds the best-known solution on all five instances, whereas the best solutions of SA and GRASP-EvPR deviate from it by 1.46\% and 17.27\% on average.

The performance of GRASP-EvPR deteriorates strongly with the instance size: its best solution deviates from the best-known solution by 7.14\% on average and by 17.27\% on set M-50-50-1, and GRASP-ILS-PR is better on all 20 instances in both the best and the average run (Wilcoxon signed-rank test, $p=1.9\times10^{-6}$). SA remains competitive and finds the best-known solution on 6 instances. Over all 20 instances, GRASP-ILS-PR is better than SA on 14 instances in the best run, equal on 4, and worse on 2 ($p=0.005$), and better on 13 instances in the average run and worse on 7 ($p=0.019$). The comparison depends on the fleet type, however. With the large-truck fleet, GRASP-ILS-PR is better than SA on 9 of the 10 instances in the best run, and equal on the remaining one, and better on all 10 in the average run; on set M-50-50-1, its average deviations are 0.00\% and 1.48\%, against 1.46\% and 3.08\% for SA. With the small-truck fleet, GRASP-ILS-PR is better in the best run on 5 instances and worse on 2, whereas SA is better in the average run on 7 of the 10 instances. In these instances, each truck serves only a few satellites (on set M-50-50-2, 26 trucks serve 50 satellites, so that most routes contain only one or two satellites), so that improvements mainly require moving satellites between many short, capacity-constrained routes. The occasional acceptance of worse solutions by SA appears to be more effective for such moves than the non-worsening acceptance rule of GRASP-ILS-PR, in line with the observations for the small-truck fleet on the main instance set (Section~\ref{sec5:comparison}).

The variability of GRASP-ILS-PR also increases with the instance size. Its average run deviates from the best-known solution by 0.13\% and 0.19\% for $n=20$, but by 1.48\% on set M-50-50-1, where each truck visits about eight satellites. On these instances, the search has not converged within 120\,s, so that longer runs or several independent runs are advisable in practice. Overall, the results show that the MILP is not a viable solution method for instances with 50 satellites, whereas GRASP-ILS-PR provides the best solutions on most large-scale instances within two minutes, and SA is a competitive alternative for fleets of many small trucks.

\subsection{Illustrative solution}
\label{sec5:example}
Figure~\ref{fig:example} shows the optimal solutions of instance M-09-09-60 with the large-truck fleet (panel a, $K=2$) and the small-truck fleet (panel b, $K=4$), and Table~\ref{tab:example} reports the routes, the assigned centroids, the waiting times, and the arrival times. Both solutions are proven optimal by Gurobi; GRASP-ILS-PR finds the first in all ten runs and the second in five of the ten runs.

\begin{figure}[htbp]
\centering
    \includegraphics[width=\textwidth]{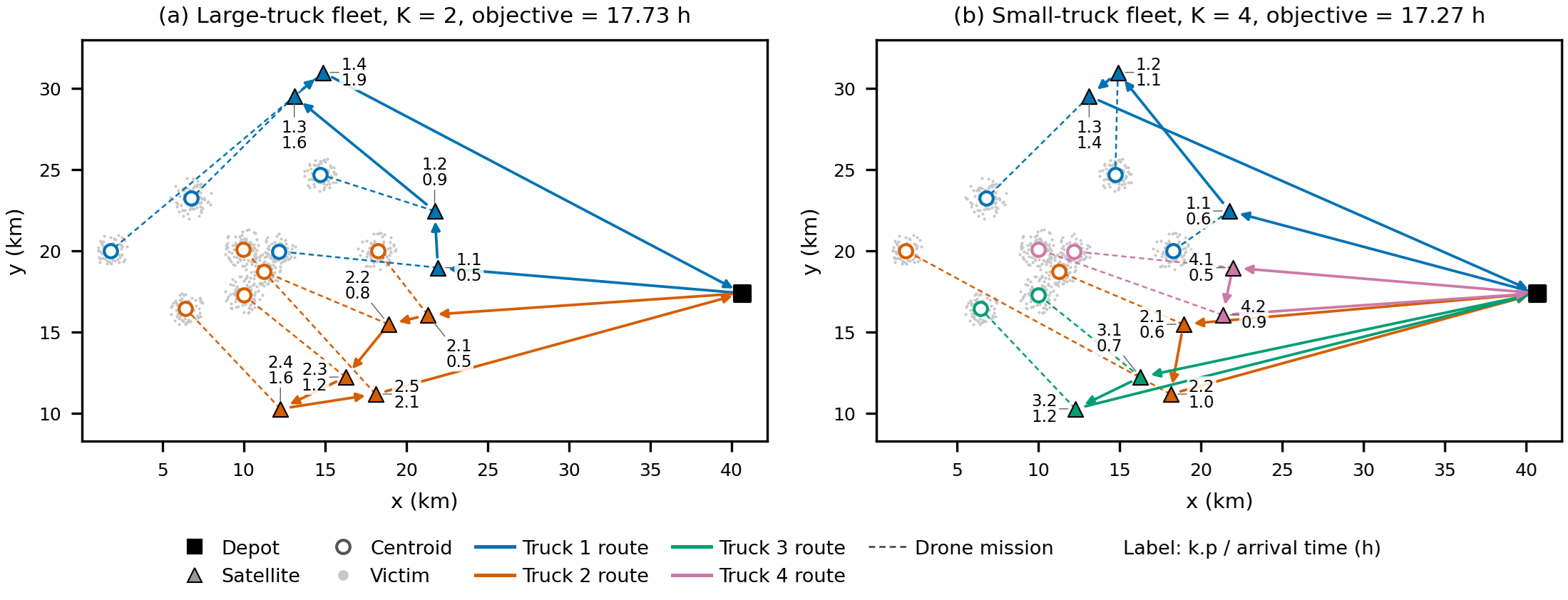}
    \caption{Optimal solutions of instance M-09-09-60 with (a) the large-truck fleet and (b) the small-truck fleet. Each satellite is labeled with the truck $k$ and its position $p$ on the route ($k.p$) and with the arrival time of the truck in hours. Truck routes are drawn as straight lines, although truck travel times are based on Manhattan distances; dashed lines connect each satellite to the centroid served by its drones.}
    \label{fig:example}
\end{figure}

\begin{table}[htbp]
\centering
\scriptsize
\setlength{\tabcolsep}{5pt}
\caption{Optimal solutions of instance M-09-09-60. For each truck, the satellites are listed in the order of visit, with the assigned centroid, the waiting time, and the arrival time; the return time to the depot and the load are given in the first row of each truck. Times are in hours.}
\label{tab:example}
\begin{tabular}{crrrrrr}
\toprule
Truck & Satellite & Centroid & Waiting & Arrival & Return & Load \\
\midrule
\multicolumn{7}{l}{\textit{(a) Large-truck fleet, $K=2$, $Q=12\,500$, objective 17.73}} \\
1 & 9 & 8 & 0.33 & 0.51 & 3.50 & 4223 \\
  & 8 & 5 & 0.25 & 0.93 &      &      \\
  & 2 & 4 & 0.30 & 1.57 &      &      \\
  & 5 & 1 & 0.57 & 1.94 &      &      \\
\addlinespace
2 & 7 & 3 & 0.17 & 0.52 & 3.17 & 5364 \\
  & 3 & 9 & 0.28 & 0.76 &      &      \\
  & 1 & 6 & 0.27 & 1.18 &      &      \\
  & 4 & 7 & 0.28 & 1.60 &      &      \\
  & 6 & 2 & 0.40 & 2.05 &      &      \\
\midrule
\multicolumn{7}{l}{\textit{(b) Small-truck fleet, $K=4$, $Q=3500$, objective 17.27}} \\
1 & 8 & 3 & 0.14 & 0.60 & 2.71 & 3302 \\
  & 5 & 5 & 0.21 & 1.13 &      &      \\
  & 2 & 4 & 0.30 & 1.42 &      &      \\
\addlinespace
2 & 3 & 9 & 0.28 & 0.59 & 2.33 & 1875 \\
  & 6 & 1 & 0.62 & 1.00 &      &      \\
\addlinespace
3 & 1 & 6 & 0.27 & 0.74 & 2.33 & 2126 \\
  & 4 & 7 & 0.28 & 1.16 &      &      \\
\addlinespace
4 & 9 & 8 & 0.33 & 0.51 & 1.84 & 2284 \\
  & 7 & 2 & 0.40 & 0.92 &      &      \\
\bottomrule
\end{tabular}
\end{table}

The solutions reflect the structure of the objective function described in Section~\ref{sec4:repr}. Since the waiting time at a satellite delays the arrivals at all subsequent satellites of the route, long drone missions are placed at the end of the routes. In both solutions, the longest waiting time of every route occurs at its last satellite: truck~1 in panel~(a) waits 0.57\,h at its last satellite, compared with at most 0.33\,h at the preceding ones, and truck~2 in panel~(b) waits 0.62\,h at its last satellite, compared with 0.28\,h at the first. Short drone missions tend to be assigned to the satellites visited first, which is the principle underlying the construction of the second-echelon assignment in Section~\ref{sec4:constr}.

The two solutions also illustrate the interaction between the echelons. With two large trucks, each truck visits four or five satellites, and the waiting times accumulate along the routes: the last satellite is reached after 2.05\,h, and the trucks are loaded to only 34\% and 43\% of their capacity. With four small trucks, each route contains two or three satellites, all satellites are reached within 1.42\,h, and the objective value is 2.6\% lower, although truck~1 uses almost its entire capacity (3302 of 3500). This instance is one of the few in which the small-truck fleet also yields the lower objective value; in 37 of the 45 pairs of instances, the large-truck fleet does (Section~\ref{Sec6}). The assignment of centroids to satellites differs between the two solutions for four of the nine satellites (satellites 5, 6, 7, and 8), since the positions of the satellites in the routes, and hence the weights of their waiting times, change with the fleet. The optimal assignment therefore cannot be determined independently of the routes, which confirms the need to optimize both echelons jointly. The effect of the fleet on the quality of the relief operation is analyzed in Section~\ref{Sec6}.

\section{Impact of the truck fleet on relief operations}
\label{Sec6}
This section compares the two truck fleets with respect to three dimensions of humanitarian logistics: efficiency, efficacy, and equity \citep{huang2012models}. The analysis uses the best solutions found by GRASP-ILS-PR for the 90 instances of the main set, which form 45 pairs of instances that differ only in the truck fleet. Let $e_i=a_i+t'_{i\sigma(i)}$ denote the time at which the demand of the centroid assigned to satellite $v_i$ is delivered, where $a_i$ is the arrival time of the truck at $v_i$. We use four measures:
\begin{align}
Z_1 &= \sum_{k\in\mathcal{K}}\sum_{i\in V\setminus\{v_0\}}a_i^k, &
Z_2 &= \frac{D}{\max_{i\in V_s}e_i}, \label{eq:Z12}\\
Z_3 &= \max_{i\in V_s}a_i-\min_{i\in V_s}a_i, &
\bar{T} &= \frac{1}{D}\sum_{i\in V_s}d_{\sigma(i)}\,e_i. \label{eq:Z3T}
\end{align}
Efficiency is measured by the objective value $Z_1$ and by the demand-weighted average delivery time $\bar{T}$. Unlike $Z_1$, $\bar{T}$ includes the drone flight times and excludes the return trips to the depot, and therefore measures the time at which supplies reach the victims. Efficacy is measured by $Z_2$, the demand delivered per hour until the last centroid is served. Equity is measured by $Z_3$, the difference between the latest and the earliest arrival of a truck at a satellite. Table~\ref{tab:fleet} reports these measures, together with the completion time $C=\max_{i\in V_s}e_i$ of the deliveries and the average truck utilization, averaged over the five drone speeds. Figure~\ref{fig:fleet} shows $Z_1$, $\bar{T}$, and $Z_3$ for the lowest and the highest drone speed.

\begin{table}[htbp]
\centering
\scriptsize
\setlength{\tabcolsep}{4pt}
\caption{Performance of the large-truck (L) and small-truck (S) fleets: objective value $Z_1$, demand-weighted average delivery time $\bar{T}$, completion time of the deliveries $C$, demand delivered per hour $Z_2$, spread of the arrival times at the satellites $Z_3$, and average truck utilization. Values are averages over the five drone speeds; times are in hours.}
\label{tab:fleet}
\begin{tabular}{rrrrrrrrrrrrr}
\toprule
 & \multicolumn{2}{c}{$Z_1$} & \multicolumn{2}{c}{$\bar{T}$} & \multicolumn{2}{c}{$C$} & \multicolumn{2}{c}{$Z_2$} & \multicolumn{2}{c}{$Z_3$} & \multicolumn{2}{c}{Util.\ (\%)} \\
\cmidrule(lr){2-3}\cmidrule(lr){4-5}\cmidrule(lr){6-7}\cmidrule(lr){8-9}\cmidrule(lr){10-11}\cmidrule(lr){12-13}
$n$ & L & S & L & S & L & S & L & S & L & S & L & S \\
\midrule
3  & 7.20  & 8.08  & 1.26 & 1.12 & 2.08 & 1.72 & 2761 & 3312  & 1.01 & 0.65 & 45 & 81 \\
4  & 7.46  & 7.68  & 1.15 & 0.83 & 2.10 & 1.21 & 3117 & 5361  & 1.29 & 0.53 & 51 & 61 \\
5  & 12.90 & 13.83 & 1.65 & 1.09 & 3.18 & 1.67 & 2909 & 5395  & 2.24 & 0.78 & 72 & 64 \\
7  & 13.97 & 16.03 & 1.33 & 0.94 & 1.96 & 1.42 & 6312 & 8602  & 1.16 & 0.64 & 49 & 69 \\
8  & 18.72 & 21.16 & 1.68 & 1.23 & 2.38 & 1.66 & 4923 & 7000  & 1.41 & 0.78 & 46 & 66 \\
9  & 18.15 & 17.53 & 1.43 & 1.08 & 2.32 & 1.60 & 4210 & 6046  & 1.60 & 0.93 & 38 & 68 \\
11 & 23.97 & 26.28 & 1.58 & 1.08 & 3.13 & 1.67 & 5023 & 9197  & 2.22 & 0.84 & 61 & 62 \\
13 & 27.58 & 30.08 & 1.46 & 1.17 & 2.35 & 1.73 & 7460 & 10036 & 1.55 & 0.90 & 46 & 71 \\
15 & 30.63 & 30.50 & 1.50 & 1.10 & 2.68 & 1.67 & 7549 & 12006 & 1.92 & 1.03 & 53 & 82 \\
\bottomrule
\end{tabular}
\end{table}

\begin{figure}[htbp]
\centering
    \includegraphics[width=\textwidth]{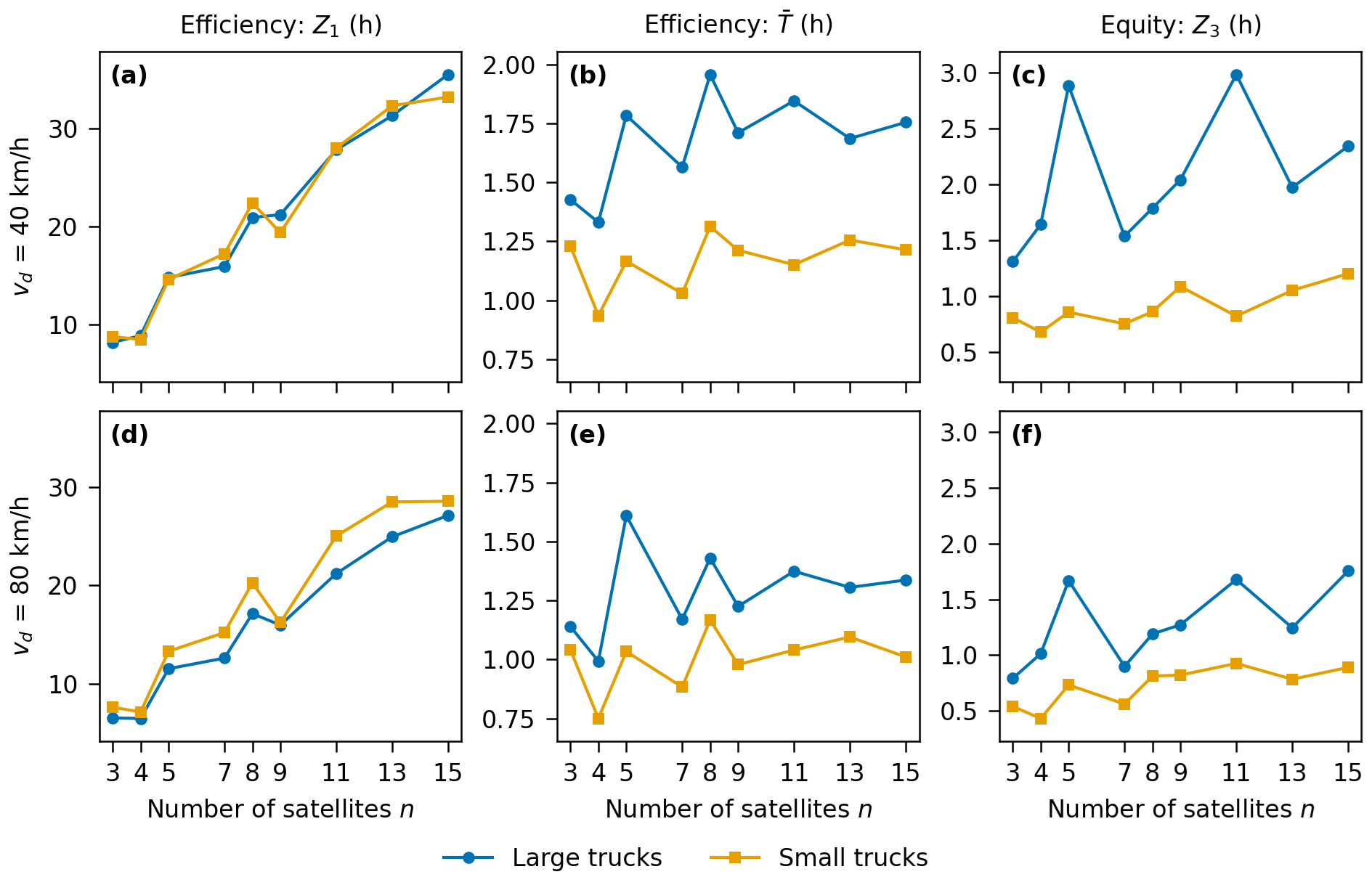}
    \caption{Objective value $Z_1$, demand-weighted average delivery time $\bar{T}$, and spread of the arrival times $Z_3$ of the large-truck and small-truck fleets for drone speeds of 40\,km/h (top) and 80\,km/h (bottom). Each point corresponds to one instance; the lines only connect instances with different numbers of satellites.}
    \label{fig:fleet}
\end{figure}

\paragraph{Efficiency}
The large-truck fleet yields the lower objective value $Z_1$ in 37 of the 45 pairs, and the small-truck fleet yields values that are 7.6\% higher on average. This comparison is affected by the structure of $Z_1$, which contains one return time per truck, so that a fleet with more trucks accumulates more return times. When efficiency is measured from the perspective of the victims, the ranking is reversed: the demand-weighted average delivery time $\bar{T}$ is lower for the small-truck fleet in all 45 pairs, by 25.5\% on average, and for every value of $n$ at both drone speeds shown in Figure~\ref{fig:fleet}(b) and~(e). The preferred fleet therefore depends on whether the return of the vehicles to the depot is part of the performance criterion.

The comparison of $Z_1$ also depends on the drone speed, as shown in Figure~\ref{fig:fleet}(a) and~(d). At a drone speed of 40\,km/h, the small-truck fleet yields the lower objective value for four of the nine values of $n$; at 80\,km/h, the large-truck fleet yields the lower value for all values of $n$. For every value of $n$, the relative difference in favor of the large-truck fleet increases with the drone speed. Increasing the drone speed from 40 to 80\,km/h reduces $Z_1$ by 22.3\% on average for the large-truck fleet but only by 12.4\% for the small-truck fleet. This asymmetry follows from Decomposition~\eqref{eq:decomp}: the waiting time at a satellite delays the arrivals at all subsequent satellites of its route, so that shorter drone missions have a larger effect on the long routes of the large trucks than on the short routes of the small trucks.

\paragraph{Efficacy}
The small-truck fleet completes the deliveries earlier in all 45 pairs, by 33.4\% on average and by between 14.1\% and 57.9\%. Its completion time $C$ lies between 1.06 and 1.99\,h across all instances and hardly depends on the number of satellites, whereas that of the large-truck fleet lies between 1.64 and 4.02\,h. Accordingly, the small-truck fleet delivers more demand per hour ($Z_2$) for every value of $n$. With several trucks leaving the depot simultaneously, each small truck visits at most three satellites, compared with up to seven for a large truck, and the waiting times at the satellites accumulate over fewer subsequent visits. The small trucks are also used more intensively: their average utilization is 69\%, with individual trucks loaded up to 99\% of their capacity, compared with 51\% for the large trucks.

\paragraph{Equity}
The spread $Z_3$ of the arrival times at the satellites is smaller for the small-truck fleet in all 45 pairs, by 48.2\% on average (Figure~\ref{fig:fleet}(c) and~(f)). It ranges from 0.43 to 1.20\,h for the small-truck fleet and from 0.79 to 2.99\,h for the large-truck fleet, so that with small trucks, the last affected community is reached at most 1.2\,h after the first. Faster drones reduce the spread by 37.6\% for the large-truck fleet but only by 19.2\% for the small-truck fleet, for the same reason as for $Z_1$.

\paragraph{Managerial implications}
The results reveal a trade-off between the two fleets. A fleet of few large trucks minimizes the cumulative arrival time, including the return trips, in most instances, and its advantage grows with the drone speed. A fleet of many small trucks delivers supplies to the victims earlier on average, completes all deliveries earlier, and distributes the arrival times more equitably across the affected communities, in every instance of the benchmark set. In the first hours after a disaster, when the time until supplies reach the victims and the fairness of the service are the primary concerns, the small-truck fleet is therefore preferable, provided that enough vehicles and drivers are available. The large-truck fleet is preferable when the availability of the vehicles for subsequent trips matters, and its disadvantage in efficacy and equity decreases as faster drones become available. Since the solutions are optimized with respect to $Z_1$, the measures $\bar{T}$, $Z_2$, and $Z_3$ are not optimized directly; optimizing them explicitly is a direction for future research (Section~\ref{Sec7:Con}).

\section{Conclusion}
\label{Sec7:Con}
This study introduced the two-echelon covering tour vehicle routing problem (2E-CTVRP) for the distribution of relief supplies after a disaster. Trucks transport supplies and drones from a depot to satellites at the periphery of the affected area, and drones deliver the supplies from the satellites to the centroids of victim clusters. We formulated the problem as a mixed integer linear program and proposed GRASP-ILS-PR, a hybrid metaheuristic that combines greedy randomized construction, a single-trajectory search, and periodic path relinking on the assignment of centroids to satellites.

The computational experiments show that Gurobi solves all instances with up to nine satellites to optimality within one hour, but only two of the 50 instances with 11 or more satellites. GRASP-ILS-PR finds all known optimal solutions within 30\,s, is at most 0.28\% worse than the Gurobi incumbent on the larger instances and better on several of them, and improves on the one-hour Gurobi incumbent by up to 11.1\% on instances with 50 satellites. It clearly outperforms a conventional GRASP with evolutionary path relinking that uses the same search components, which shows the benefit of scheduling the expensive components periodically around a single trajectory. Compared with a simulated annealing algorithm built from the same components, it reaches a similar best-run quality with significantly smaller variation between runs, whereas simulated annealing remains competitive for fleets of many small trucks. The comparison of the two truck fleets reveals a trade-off: few large trucks yield the lower sum of arrival times in most instances, whereas many small trucks deliver supplies to the victims 25.5\% earlier on average and distribute the arrival times more equitably across the affected communities in every instance.

The study has several limitations. The instances are synthetic, although motivated by the 2025 flooding in Thai Nguyen Province, Vietnam. Travel times and demands are deterministic, trucks wait at each satellite until their drones return, each satellite serves exactly one cluster, and the objective measures the arrival times of the trucks rather than the times at which supplies reach the victims. Future research could optimize victim-centered measures directly, account for uncertainty in travel times, demands, and road conditions, relax the synchronization between trucks and drones, integrate the location of the satellites into the model, and apply the approach to real data from an actual disaster.

\appendix
\section{Detailed computational results}
\label{app:results}
\setcounter{table}{0}
Tables~\ref{tab:appA1} and~\ref{tab:appA2} report the results for each instance with the large-truck and the small-truck fleet, respectively. The instances are grouped by the number of satellites $n$ and ordered by drone speed within each group.

\begin{scriptsize}
\setlength{\tabcolsep}{4pt}
\begin{longtable}{lrrrrrrrrrr}
\caption{Per-instance results for the large-truck fleet ($Q=12\,500$). Gurobi: incumbent (Obj.), best lower bound (LB), MIP gap at termination, and runtime, with a time limit of 3600\,s. Metaheuristics: best and average objective values over ten runs of 30\,s; the proposed method is shown with a bold header. Bold values denote the best-known solution (ties within a relative tolerance of $10^{-5}$).}
\label{tab:appA1}\\
\toprule
 & \multicolumn{4}{c}{Gurobi} & \multicolumn{2}{c}{GRASP-EvPR} & \multicolumn{2}{c}{SA} & \multicolumn{2}{c}{\textbf{GRASP-ILS-PR}} \\
\cmidrule(lr){2-5}\cmidrule(lr){6-7}\cmidrule(lr){8-9}\cmidrule(lr){10-11}
Instance & Obj. & LB & Gap (\%) & CPU (s) & Best & Avg. & Best & Avg. & Best & Avg. \\
\midrule
\endfirsthead
\caption[]{Per-instance results for the large-truck fleet (continued).}\\
\toprule
 & \multicolumn{4}{c}{Gurobi} & \multicolumn{2}{c}{GRASP-EvPR} & \multicolumn{2}{c}{SA} & \multicolumn{2}{c}{\textbf{GRASP-ILS-PR}} \\
\cmidrule(lr){2-5}\cmidrule(lr){6-7}\cmidrule(lr){8-9}\cmidrule(lr){10-11}
Instance & Obj. & LB & Gap (\%) & CPU (s) & Best & Avg. & Best & Avg. & Best & Avg. \\
\midrule
\endhead
\midrule
\multicolumn{11}{r}{\textit{Continued on next page}} \\
\endfoot
\bottomrule
\endlastfoot
M-05-03-40-1 & \textbf{8.174} & 8.174 & 0.00 & 0.01 & \textbf{8.174} & 8.174 & \textbf{8.174} & 8.174 & \textbf{8.174} & 8.174 \\
M-05-03-50-1 & \textbf{7.505} & 7.505 & 0.00 & 0.00 & \textbf{7.505} & 7.505 & \textbf{7.505} & 7.505 & \textbf{7.505} & 7.505 \\
M-05-03-60-1 & \textbf{7.059} & 7.059 & 0.00 & 0.01 & \textbf{7.059} & 7.059 & \textbf{7.059} & 7.059 & \textbf{7.059} & 7.059 \\
M-05-03-70-1 & \textbf{6.741} & 6.741 & 0.00 & 0.00 & \textbf{6.741} & 6.741 & \textbf{6.741} & 6.741 & \textbf{6.741} & 6.741 \\
M-05-03-80-1 & \textbf{6.502} & 6.502 & 0.00 & 0.00 & \textbf{6.502} & 6.502 & \textbf{6.502} & 6.502 & \textbf{6.502} & 6.502 \\
\addlinespace
M-05-04-40-1 & \textbf{8.866} & 8.866 & 0.00 & 0.02 & \textbf{8.866} & 8.866 & \textbf{8.866} & 8.866 & \textbf{8.866} & 8.866 \\
M-05-04-50-1 & \textbf{7.922} & 7.922 & 0.00 & 0.03 & \textbf{7.922} & 7.922 & \textbf{7.922} & 7.922 & \textbf{7.922} & 7.922 \\
M-05-04-60-1 & \textbf{7.267} & 7.267 & 0.00 & 0.02 & \textbf{7.267} & 7.267 & \textbf{7.267} & 7.267 & \textbf{7.267} & 7.267 \\
M-05-04-70-1 & \textbf{6.799} & 6.799 & 0.00 & 0.02 & \textbf{6.799} & 6.799 & \textbf{6.799} & 6.799 & \textbf{6.799} & 6.799 \\
M-05-04-80-1 & \textbf{6.448} & 6.448 & 0.00 & 0.02 & \textbf{6.448} & 6.448 & \textbf{6.448} & 6.448 & \textbf{6.448} & 6.448 \\
\addlinespace
M-08-05-40-1 & \textbf{14.799} & 14.799 & 0.00 & 0.08 & \textbf{14.799} & 14.799 & \textbf{14.799} & 14.799 & \textbf{14.799} & 14.799 \\
M-08-05-50-1 & \textbf{13.512} & 13.512 & 0.00 & 0.09 & \textbf{13.512} & 13.512 & \textbf{13.512} & 13.512 & \textbf{13.512} & 13.512 \\
M-08-05-60-1 & \textbf{12.655} & 12.655 & 0.00 & 0.08 & \textbf{12.655} & 12.655 & \textbf{12.655} & 12.655 & \textbf{12.655} & 12.655 \\
M-08-05-70-1 & \textbf{12.013} & 12.013 & 0.00 & 0.08 & \textbf{12.013} & 12.013 & \textbf{12.013} & 12.013 & \textbf{12.013} & 12.013 \\
M-08-05-80-1 & \textbf{11.520} & 11.520 & 0.00 & 0.09 & \textbf{11.520} & 11.520 & \textbf{11.520} & 11.520 & \textbf{11.520} & 11.520 \\
\addlinespace
M-10-07-40-1 & \textbf{15.908} & 15.908 & 0.00 & 4.41 & \textbf{15.908} & 15.908 & \textbf{15.908} & 15.908 & \textbf{15.908} & 15.908 \\
M-10-07-50-1 & \textbf{14.584} & 14.584 & 0.00 & 4.81 & \textbf{14.584} & 14.584 & \textbf{14.584} & 14.584 & \textbf{14.584} & 14.584 \\
M-10-07-60-1 & \textbf{13.701} & 13.701 & 0.00 & 4.40 & \textbf{13.701} & 13.701 & \textbf{13.701} & 13.701 & \textbf{13.701} & 13.701 \\
M-10-07-70-1 & \textbf{13.071} & 13.071 & 0.00 & 4.53 & \textbf{13.071} & 13.071 & \textbf{13.071} & 13.071 & \textbf{13.071} & 13.071 \\
M-10-07-80-1 & \textbf{12.598} & 12.598 & 0.00 & 5.22 & \textbf{12.598} & 12.598 & \textbf{12.598} & 12.598 & \textbf{12.598} & 12.598 \\
\addlinespace
M-10-08-40-1 & \textbf{20.906} & 20.906 & 0.00 & 25.99 & \textbf{20.906} & 20.906 & \textbf{20.906} & 20.906 & \textbf{20.906} & 20.906 \\
M-10-08-50-1 & \textbf{19.419} & 19.419 & 0.00 & 40.67 & \textbf{19.419} & 19.419 & \textbf{19.419} & 19.419 & \textbf{19.419} & 19.419 \\
M-10-08-60-1 & \textbf{18.421} & 18.421 & 0.00 & 33.55 & \textbf{18.421} & 18.421 & \textbf{18.421} & 18.421 & \textbf{18.421} & 18.422 \\
M-10-08-70-1 & \textbf{17.708} & 17.708 & 0.00 & 36.41 & \textbf{17.708} & 17.708 & \textbf{17.708} & 17.708 & \textbf{17.708} & 17.708 \\
M-10-08-80-1 & \textbf{17.135} & 17.135 & 0.00 & 44.75 & \textbf{17.135} & 17.135 & \textbf{17.135} & 17.135 & \textbf{17.135} & 17.135 \\
\addlinespace
M-09-09-40-1 & \textbf{21.194} & 21.194 & 0.00 & 143.83 & \textbf{21.194} & 21.217 & \textbf{21.194} & 21.194 & \textbf{21.194} & 21.194 \\
M-09-09-50-1 & \textbf{19.116} & 19.116 & 0.00 & 126.04 & \textbf{19.116} & 19.130 & \textbf{19.116} & 19.116 & \textbf{19.116} & 19.116 \\
M-09-09-60-1 & \textbf{17.730} & 17.730 & 0.00 & 131.92 & \textbf{17.730} & 17.732 & \textbf{17.730} & 17.730 & \textbf{17.730} & 17.730 \\
M-09-09-70-1 & \textbf{16.729} & 16.729 & 0.00 & 134.20 & \textbf{16.729} & 16.729 & \textbf{16.729} & 16.729 & \textbf{16.729} & 16.729 \\
M-09-09-80-1 & \textbf{15.966} & 15.966 & 0.00 & 270.91 & \textbf{15.966} & 15.966 & \textbf{15.966} & 15.966 & \textbf{15.966} & 15.966 \\
\addlinespace
M-12-11-40-1 & \textbf{27.795} & 15.811 & 43.12 & 3600.02 & \textbf{27.795} & 28.000 & \textbf{27.795} & 27.889 & \textbf{27.795} & 27.797 \\
M-12-11-50-1 & \textbf{25.211} & 13.210 & 47.60 & 3600.03 & 25.272 & 25.367 & 25.227 & 25.281 & \textbf{25.211} & 25.219 \\
M-12-11-60-1 & \textbf{23.462} & 14.685 & 37.41 & 3600.02 & \textbf{23.462} & 23.557 & \textbf{23.462} & 23.507 & \textbf{23.462} & 23.465 \\
M-12-11-70-1 & \textbf{22.184} & 13.700 & 38.24 & 3600.04 & 22.193 & 22.250 & \textbf{22.184} & 22.197 & \textbf{22.184} & 22.184 \\
M-12-11-80-1 & \textbf{21.213} & 13.333 & 37.15 & 3600.01 & 21.252 & 21.321 & \textbf{21.213} & 21.225 & \textbf{21.213} & 21.213 \\
\addlinespace
M-15-13-40-1 & \textbf{31.275} & 13.643 & 56.38 & 3600.02 & 31.734 & 31.870 & \textbf{31.275} & 31.321 & \textbf{31.275} & 31.276 \\
M-15-13-50-1 & \textbf{28.748} & 12.624 & 56.09 & 3600.01 & 29.067 & 29.384 & \textbf{28.748} & 28.811 & \textbf{28.748} & 28.768 \\
M-15-13-60-1 & 27.322 & 12.449 & 54.44 & 3600.01 & 27.282 & 27.513 & \textbf{27.063} & 27.128 & \textbf{27.063} & 27.075 \\
M-15-13-70-1 & \textbf{25.859} & 11.418 & 55.84 & 3600.01 & 26.112 & 26.293 & \textbf{25.859} & 25.915 & \textbf{25.859} & 25.870 \\
M-15-13-80-1 & \textbf{24.957} & 10.850 & 56.52 & 3600.01 & 25.118 & 25.293 & \textbf{24.957} & 25.013 & \textbf{24.957} & 24.957 \\
\addlinespace
M-17-15-40-1 & \textbf{35.473} & 14.397 & 59.41 & 3600.02 & 36.447 & 36.746 & 35.724 & 35.811 & 35.474 & 35.675 \\
M-17-15-50-1 & \textbf{32.155} & 11.536 & 64.12 & 3600.01 & 33.043 & 33.221 & 32.355 & 32.538 & 32.244 & 32.332 \\
M-17-15-60-1 & 30.046 & 12.187 & 59.44 & 3600.02 & 30.524 & 30.755 & 30.019 & 30.207 & \textbf{29.945} & 30.027 \\
M-17-15-70-1 & \textbf{28.364} & 12.144 & 57.19 & 3600.01 & 28.816 & 28.982 & 28.424 & 28.497 & 28.365 & 28.412 \\
M-17-15-80-1 & \textbf{27.141} & 11.054 & 59.27 & 3600.01 & 27.447 & 27.649 & 27.161 & 27.210 & \textbf{27.141} & 27.183 \\
\end{longtable}
\end{scriptsize}

\clearpage
\begin{scriptsize}
\setlength{\tabcolsep}{4pt}
\begin{longtable}{lrrrrrrrrrr}
\caption{Per-instance results for the small-truck fleet ($Q=3500$). Gurobi: incumbent (Obj.), best lower bound (LB), MIP gap at termination, and runtime, with a time limit of 3600\,s. Metaheuristics: best and average objective values over ten runs of 30\,s; the proposed method is shown with a bold header. Bold values denote the best-known solution (ties within a relative tolerance of $10^{-5}$).}
\label{tab:appA2}\\
\toprule
 & \multicolumn{4}{c}{Gurobi} & \multicolumn{2}{c}{GRASP-EvPR} & \multicolumn{2}{c}{SA} & \multicolumn{2}{c}{\textbf{GRASP-ILS-PR}} \\
\cmidrule(lr){2-5}\cmidrule(lr){6-7}\cmidrule(lr){8-9}\cmidrule(lr){10-11}
Instance & Obj. & LB & Gap (\%) & CPU (s) & Best & Avg. & Best & Avg. & Best & Avg. \\
\midrule
\endfirsthead
\caption[]{Per-instance results for the small-truck fleet (continued).}\\
\toprule
 & \multicolumn{4}{c}{Gurobi} & \multicolumn{2}{c}{GRASP-EvPR} & \multicolumn{2}{c}{SA} & \multicolumn{2}{c}{\textbf{GRASP-ILS-PR}} \\
\cmidrule(lr){2-5}\cmidrule(lr){6-7}\cmidrule(lr){8-9}\cmidrule(lr){10-11}
Instance & Obj. & LB & Gap (\%) & CPU (s) & Best & Avg. & Best & Avg. & Best & Avg. \\
\midrule
\endhead
\midrule
\multicolumn{11}{r}{\textit{Continued on next page}} \\
\endfoot
\bottomrule
\endlastfoot
M-05-03-40-2 & \textbf{8.760} & 8.760 & 0.00 & 0.02 & \textbf{8.760} & 8.760 & \textbf{8.760} & 8.812 & \textbf{8.760} & 8.760 \\
M-05-03-50-2 & \textbf{8.292} & 8.292 & 0.00 & 0.02 & \textbf{8.292} & 8.292 & \textbf{8.292} & 8.331 & \textbf{8.292} & 8.292 \\
M-05-03-60-2 & \textbf{7.980} & 7.980 & 0.00 & 0.01 & \textbf{7.980} & 7.980 & \textbf{7.980} & 8.012 & \textbf{7.980} & 7.980 \\
M-05-03-70-2 & \textbf{7.757} & 7.757 & 0.00 & 0.02 & \textbf{7.757} & 7.757 & \textbf{7.757} & 7.796 & \textbf{7.757} & 7.757 \\
M-05-03-80-2 & \textbf{7.590} & 7.590 & 0.00 & 0.02 & \textbf{7.590} & 7.590 & \textbf{7.590} & 7.622 & \textbf{7.590} & 7.590 \\
\addlinespace
M-05-04-40-2 & \textbf{8.464} & 8.464 & 0.00 & 0.24 & \textbf{8.464} & 8.464 & \textbf{8.464} & 8.464 & \textbf{8.464} & 8.464 \\
M-05-04-50-2 & \textbf{7.925} & 7.925 & 0.00 & 0.22 & \textbf{7.925} & 7.925 & \textbf{7.925} & 7.925 & \textbf{7.925} & 7.925 \\
M-05-04-60-2 & \textbf{7.566} & 7.566 & 0.00 & 0.23 & \textbf{7.566} & 7.566 & \textbf{7.566} & 7.566 & \textbf{7.566} & 7.566 \\
M-05-04-70-2 & \textbf{7.309} & 7.309 & 0.00 & 0.25 & \textbf{7.309} & 7.309 & \textbf{7.309} & 7.309 & \textbf{7.309} & 7.309 \\
M-05-04-80-2 & \textbf{7.117} & 7.117 & 0.00 & 0.20 & \textbf{7.117} & 7.117 & \textbf{7.117} & 7.117 & \textbf{7.117} & 7.117 \\
\addlinespace
M-08-05-40-2 & \textbf{14.584} & 14.584 & 0.00 & 0.41 & \textbf{14.584} & 14.584 & \textbf{14.584} & 14.584 & \textbf{14.584} & 14.584 \\
M-08-05-50-2 & \textbf{14.069} & 14.069 & 0.00 & 0.55 & \textbf{14.069} & 14.069 & \textbf{14.069} & 14.069 & \textbf{14.069} & 14.069 \\
M-08-05-60-2 & \textbf{13.726} & 13.726 & 0.00 & 0.46 & \textbf{13.726} & 13.726 & \textbf{13.726} & 13.726 & \textbf{13.726} & 13.726 \\
M-08-05-70-2 & \textbf{13.481} & 13.481 & 0.00 & 0.45 & \textbf{13.481} & 13.481 & \textbf{13.481} & 13.481 & \textbf{13.481} & 13.481 \\
M-08-05-80-2 & \textbf{13.297} & 13.297 & 0.00 & 0.51 & \textbf{13.297} & 13.297 & \textbf{13.297} & 13.297 & \textbf{13.297} & 13.297 \\
\addlinespace
M-10-07-40-2 & \textbf{17.204} & 17.204 & 0.00 & 8.34 & \textbf{17.204} & 17.204 & \textbf{17.204} & 17.204 & \textbf{17.204} & 17.204 \\
M-10-07-50-2 & \textbf{16.402} & 16.402 & 0.00 & 8.50 & \textbf{16.402} & 16.402 & \textbf{16.402} & 16.402 & \textbf{16.402} & 16.402 \\
M-10-07-60-2 & \textbf{15.868} & 15.868 & 0.00 & 8.37 & \textbf{15.868} & 15.868 & \textbf{15.868} & 15.868 & \textbf{15.868} & 15.868 \\
M-10-07-70-2 & \textbf{15.486} & 15.486 & 0.00 & 8.18 & \textbf{15.486} & 15.486 & \textbf{15.486} & 15.486 & \textbf{15.486} & 15.486 \\
M-10-07-80-2 & \textbf{15.200} & 15.200 & 0.00 & 4.28 & \textbf{15.200} & 15.200 & \textbf{15.200} & 15.200 & \textbf{15.200} & 15.200 \\
\addlinespace
M-10-08-40-2 & \textbf{22.380} & 22.380 & 0.00 & 283.47 & \textbf{22.380} & 22.380 & \textbf{22.380} & 22.382 & \textbf{22.380} & 22.380 \\
M-10-08-50-2 & \textbf{21.544} & 21.544 & 0.00 & 273.50 & \textbf{21.544} & 21.544 & \textbf{21.544} & 21.544 & \textbf{21.544} & 21.544 \\
M-10-08-60-2 & \textbf{20.987} & 20.987 & 0.00 & 297.38 & \textbf{20.987} & 20.987 & \textbf{20.987} & 20.988 & \textbf{20.987} & 20.987 \\
M-10-08-70-2 & \textbf{20.589} & 20.589 & 0.00 & 287.10 & \textbf{20.589} & 20.589 & \textbf{20.589} & 20.590 & \textbf{20.589} & 20.589 \\
M-10-08-80-2 & \textbf{20.275} & 20.275 & 0.00 & 254.07 & \textbf{20.275} & 20.275 & \textbf{20.275} & 20.278 & \textbf{20.275} & 20.275 \\
\addlinespace
M-09-09-40-2 & \textbf{19.369} & 19.369 & 0.00 & 1466.46 & \textbf{19.369} & 19.380 & \textbf{19.369} & 19.369 & \textbf{19.369} & 19.373 \\
M-09-09-50-2 & \textbf{18.111} & 18.111 & 0.00 & 1562.97 & \textbf{18.111} & 18.123 & \textbf{18.111} & 18.123 & \textbf{18.111} & 18.115 \\
M-09-09-60-2 & \textbf{17.273} & 17.273 & 0.00 & 1448.33 & \textbf{17.273} & 17.291 & \textbf{17.273} & 17.273 & \textbf{17.273} & 17.281 \\
M-09-09-70-2 & \textbf{16.674} & 16.674 & 0.00 & 1707.50 & \textbf{16.674} & 16.689 & \textbf{16.674} & 16.676 & \textbf{16.674} & 16.682 \\
M-09-09-80-2 & \textbf{16.225} & 16.225 & 0.00 & 1450.65 & \textbf{16.225} & 16.236 & \textbf{16.225} & 16.226 & \textbf{16.225} & 16.234 \\
\addlinespace
M-12-11-40-2 & \textbf{27.967} & 23.641 & 15.47 & 3600.02 & 27.972 & 28.060 & \textbf{27.967} & 28.024 & \textbf{27.967} & 27.996 \\
M-12-11-50-2 & \textbf{26.825} & 21.206 & 20.95 & 3600.02 & 26.829 & 26.921 & \textbf{26.825} & 26.963 & \textbf{26.825} & 26.840 \\
M-12-11-60-2 & \textbf{26.064} & 20.220 & 22.42 & 3600.02 & 26.084 & 26.146 & 26.130 & 26.148 & \textbf{26.064} & 26.070 \\
M-12-11-70-2 & \textbf{25.511} & 25.511 & 0.00 & 762.38 & 25.520 & 25.548 & \textbf{25.511} & 25.554 & \textbf{25.511} & 25.520 \\
M-12-11-80-2 & \textbf{25.044} & 25.044 & 0.00 & 1591.63 & 25.082 & 25.124 & \textbf{25.044} & 25.061 & \textbf{25.044} & 25.061 \\
\addlinespace
M-15-13-40-2 & \textbf{32.267} & 23.317 & 27.74 & 3600.02 & 32.531 & 32.665 & 32.296 & 32.361 & 32.292 & 32.309 \\
M-15-13-50-2 & 30.787 & 21.546 & 30.02 & 3600.02 & 31.034 & 31.141 & 30.791 & 30.872 & \textbf{30.745} & 30.832 \\
M-15-13-60-2 & 29.801 & 19.846 & 33.40 & 3600.01 & 30.086 & 30.152 & \textbf{29.731} & 29.822 & 29.782 & 29.841 \\
M-15-13-70-2 & 29.059 & 20.188 & 30.53 & 3600.17 & 29.273 & 29.376 & \textbf{29.007} & 29.088 & 29.062 & 29.126 \\
M-15-13-80-2 & \textbf{28.463} & 20.503 & 27.97 & 3600.19 & 28.674 & 28.815 & \textbf{28.463} & 28.526 & 28.519 & 28.582 \\
\addlinespace
M-17-15-40-2 & 33.191 & 23.647 & 28.76 & 3600.02 & 33.697 & 33.862 & 33.191 & 33.315 & \textbf{33.176} & 33.309 \\
M-17-15-50-2 & \textbf{31.349} & 21.271 & 32.15 & 3600.02 & 31.738 & 31.928 & 31.458 & 31.594 & \textbf{31.349} & 31.457 \\
M-17-15-60-2 & 30.160 & 20.175 & 33.11 & 3600.03 & 30.466 & 30.576 & 30.132 & 30.220 & \textbf{30.127} & 30.211 \\
M-17-15-70-2 & 29.304 & 18.028 & 38.48 & 3600.02 & 29.537 & 29.630 & 29.254 & 29.350 & \textbf{29.247} & 29.309 \\
M-17-15-80-2 & \textbf{28.528} & 17.238 & 39.58 & 3600.02 & 28.801 & 28.953 & \textbf{28.528} & 28.639 & 28.580 & 28.642 \\
\end{longtable}
\end{scriptsize}

\begin{table}[htbp]
\centering
\scriptsize
\setlength{\tabcolsep}{4pt}
\caption{Per-instance results on the large-scale instances. Gurobi: incumbent (Obj.), best lower bound (LB), and MIP gap at termination, with a time limit of 3600\,s, which is reached on all instances. Metaheuristics: best and average objective values over ten runs of 120\,s; the proposed method is shown with a bold header. Bold values denote the best-known solution.}
\label{tab:appA3}
\begin{tabular}{lrrrrrrrrr}
\toprule
 & \multicolumn{3}{c}{Gurobi} & \multicolumn{2}{c}{GRASP-EvPR} & \multicolumn{2}{c}{SA} & \multicolumn{2}{c}{\textbf{GRASP-ILS-PR}} \\
\cmidrule(lr){2-4}\cmidrule(lr){5-6}\cmidrule(lr){7-8}\cmidrule(lr){9-10}
Instance & Obj. & LB & Gap (\%) & Best & Avg. & Best & Avg. & Best & Avg. \\
\midrule
M-20-20-40-1 & \textbf{47.896} & 10.177 & 78.75 & 51.069 & 51.535 & 47.966 & 48.105 & \textbf{47.896} & 47.995 \\
M-20-20-50-1 & 43.322 & 9.532 & 78.00 & 45.318 & 45.972 & \textbf{43.265} & 43.468 & \textbf{43.265} & 43.324 \\
M-20-20-60-1 & \textbf{40.132} & 9.975 & 75.14 & 42.182 & 42.541 & 40.146 & 40.299 & 40.142 & 40.176 \\
M-20-20-70-1 & 37.877 & 9.510 & 74.89 & 39.303 & 39.838 & 37.844 & 37.977 & \textbf{37.815} & 37.848 \\
M-20-20-80-1 & \textbf{36.021} & 9.468 & 73.72 & 37.320 & 37.970 & 36.067 & 36.188 & \textbf{36.021} & 36.055 \\
\addlinespace
M-20-20-40-2 & 43.982 & 22.549 & 48.73 & 44.718 & 44.890 & \textbf{43.852} & 43.949 & \textbf{43.852} & 43.900 \\
M-20-20-50-2 & 41.901 & 21.747 & 48.10 & 42.537 & 42.695 & \textbf{41.834} & 41.924 & 41.903 & 41.939 \\
M-20-20-60-2 & 40.559 & 20.696 & 48.97 & 41.003 & 41.263 & 40.514 & 40.568 & \textbf{40.453} & 40.538 \\
M-20-20-70-2 & \textbf{39.420} & 20.199 & 48.76 & 40.112 & 40.273 & \textbf{39.420} & 39.508 & \textbf{39.420} & 39.502 \\
M-20-20-80-2 & 38.680 & 20.029 & 48.22 & 39.252 & 39.478 & \textbf{38.670} & 38.729 & \textbf{38.670} & 38.733 \\
\addlinespace
M-50-50-40-1 & 136.511 & 4.457 & 96.73 & 145.450 & 147.998 & 123.957 & 125.944 & \textbf{121.480} & 122.796 \\
M-50-50-50-1 & 118.609 & 4.595 & 96.13 & 127.397 & 130.277 & 107.768 & 110.408 & \textbf{107.029} & 109.550 \\
M-50-50-60-1 & 109.531 & 5.199 & 95.25 & 115.098 & 116.551 & 99.079 & 101.044 & \textbf{98.320} & 99.084 \\
M-50-50-70-1 & 100.587 & 3.786 & 96.24 & 105.471 & 107.144 & 92.389 & 93.477 & \textbf{90.982} & 92.254 \\
M-50-50-80-1 & 96.112 & 11.506 & 88.03 & 97.943 & 100.452 & 87.386 & 88.078 & \textbf{85.471} & 86.980 \\
\addlinespace
M-50-50-40-2 & 104.546 & 35.322 & 66.21 & 107.081 & 107.875 & \textbf{100.911} & 101.361 & 101.079 & 101.368 \\
M-50-50-50-2 & 106.250 & 33.804 & 68.18 & 99.509 & 101.235 & 95.766 & 96.012 & \textbf{95.328} & 96.119 \\
M-50-50-60-2 & 97.371 & 33.578 & 65.52 & 96.670 & 97.108 & 91.976 & 92.239 & \textbf{91.854} & 92.257 \\
M-50-50-70-2 & 92.152 & 33.651 & 63.48 & 93.173 & 94.174 & 89.444 & 89.650 & \textbf{89.375} & 89.779 \\
M-50-50-80-2 & 88.639 & 33.395 & 62.32 & 91.101 & 91.841 & 87.549 & 87.723 & \textbf{87.460} & 87.774 \\
\bottomrule
\end{tabular}
\end{table}

   \bibliographystyle{model5-names}
   \biboptions{authoryear}
{\small
  \bibliography{Refs}}




\end{document}